\documentclass[lettersize,journal]{IEEEtran}
\usepackage[utf8]{inputenc} 
\usepackage{amssymb,amsthm,amsmath}
\usepackage{booktabs}
\usepackage{bm,bbm}
\usepackage{cite}
\usepackage{color}%[dvipdfmx]
\usepackage{graphicx}%[dvipdfmx]
\usepackage{hyperref}%[dvipdfmx]
\usepackage{latexsym}
\usepackage{multirow}
\usepackage{newtxmath}
\usepackage{overpic}
\usepackage{url}
\newtheorem{theorem}{Theorem}
\newtheorem{corollary}{Corollary}

\newtheorem{proposition}{Proposition}

\definecolor{midori}{rgb}{0.0, 0.5, 0.0}

\newcommand{\zo}{{\mathrm{zo}}}
\newcommand{\tsk}{{\text{tsk}}}
\newcommand{\sur}{{\text{sur}}}
\newcommand{\KL}{{\mathrm{KL}}}
\newcommand{\SKL}{{\mathrm{SKL}}}
\newcommand{\SQ}{{\mathrm{SQ}}}
\newcommand{\LM}{{\mathrm{L}}}
\newcommand{\RLM}{{\mathrm{RL}}}
\newcommand{\LR}{{\mathrm{LR}}}
\newcommand{\LS}{{\mathrm{LS}}}
\newcommand{\MLS}{{\mathrm{MLS}}}
\newcommand{\LSQ}{{\mathrm{LSQ}}}
\newcommand{\BY}{{\mathrm{BY}}}
\newcommand{\CH}{{\mathrm{CH}}}
\DeclareMathOperator{\argmax}{{arg\,max}}
\DeclareMathOperator{\argmin}{{arg\,min}}
\DeclareMathOperator{\diag}{{diag}}
\DeclareMathOperator{\trace}{{tr}}

\newcommand{\bbeta}{{\bm{\beta}}}
\newcommand{\bg}{{\bm{g}}}
\newcommand{\bt}{{\bm{t}}}
\newcommand{\bT}{{\bm{T}}}
\newcommand{\bs}{{\bm{s}}}
\newcommand{\bp}{{\bm{p}}}
\newcommand{\bq}{{\bm{q}}}

\newcommand{\bv}{{\bm{v}}}
\newcommand{\bw}{{\bm{w}}}
\newcommand{\bx}{{\bm{x}}}
\newcommand{\bX}{{\bm{X}}}
\newcommand{\bz}{{\bm{z}}}
\newcommand{\bZ}{{\bm{Z}}}
\newcommand{\bo}{{\bm{1}}}
\newcommand{\bzero}{{\bm{0}}}

\newcommand{\calR}{{\mathcal{R}}}
\newcommand{\calG}{{\mathcal{G}}}

\newcommand{\calS}{{\mathcal{S}}}
\newcommand{\calZ}{{\mathcal{Z}}}
\DeclareMathOperator{\bbE}{{\mathbb{E}}}
\newcommand{\bbR}{{\mathbb{R}}}

\newcommand{\rmA}{{\mathrm{A}}}
\newcommand{\rmB}{{\mathrm{B}}}
\newcommand{\rmC}{{\mathrm{C}}}
\begin{document}
\title{Label Smoothing is Robustification\\against Model Misspecification}
\author{Ryoya Yamasaki, and Toshiyuki Tanaka, \IEEEmembership{Member, IEEE}%
\IEEEcompsocitemizethanks{\IEEEcompsocthanksitem
Ryoya Yamasaki and Toshiyuki Tanaka are with the Department of Informatics, 
Graduate School of Informatics, Kyoto University, Kyoto, 606-8501, Japan. 
E-mail: yamasaki@sys.i.kyoto-u.ac.jp, tt@i.kyoto-u.ac.jp.}}
\markboth{Preprint version}
{Shell \MakeLowercase{\textit{et al.}}: Label Smoothing is Robustification against Model Misspecification}
%==================================================%OK
\maketitle
\begin{abstract}
Label smoothing (LS) adopts smoothed targets in classification tasks. 
For example, in binary classification, instead of the one-hot target $(1,0)^\top$ used 
in conventional logistic regression (LR), LR with LS (LSLR) uses the smoothed target 
$(1-\frac{\alpha}{2},\frac{\alpha}{2})^\top$ with a smoothing level $\alpha\in(0,1)$, 
which causes squeezing of values of the logit.
Apart from the common regularization-based interpretation of LS 
that leads to an inconsistent probability estimator, 
we regard LSLR as modifying the loss function and 
consistent estimator for probability estimation.
In order to study the significance of each of these two modifications by LSLR, 
we introduce a modified LSLR (MLSLR) that uses the same loss function as LSLR 
and the same consistent estimator as LR, while not squeezing the logits.
For the loss function modification, we theoretically show that 
MLSLR with a larger smoothing level has lower efficiency with correctly-specified models, 
while it exhibits higher robustness against model misspecification than LR.
Also, for the modification of the probability estimator, 
an experimental comparison between LSLR and MLSLR showed that 
this modification and squeezing of the logits in LSLR have negative 
effects on the probability estimation and classification performance.
The understanding of the properties of LS provided by these comparisons 
allows us to propose MLSLR as an improvement over LSLR.
\end{abstract}
%==================================================%OK
%==================================================%OK
\section{Introduction}
\label{sec:Introduction}
%==========%OK
\IEEEPARstart{L}{abel} smoothing (LS) adopts smoothed targets in classification problems.
Conventional logistic regression (LR) uses a one-hot vector as a target (Section~\ref{sec:LR}), 
while LR with LS (LSLR) \cite{szegedy2016rethinking} uses a smoothed target vector that replaces 
the component 1 in the one-hot vector with a smaller value and 0 with a larger value (Section~\ref{sec:LS}).
Previous studies have provided, mostly through experimental considerations, 
several heuristic findings on behaviors of LSLR, for example, 
\begin{enumerate}
\item[{\hypertarget{A1}{A1}.}]
It prevents the largest logit from becoming much larger than all others (squeezes the logits) 
and encourages the model to be less confident \cite{szegedy2016rethinking}.
\item[{\hypertarget{A2}{A2}.}]
It generally improves adversarial robustness against a variety of attacks \cite{goibert2019adversarial}.
\item[{\hypertarget{A3}{A3}.}]
It can often significantly improve the generalization of a (multi-class) neural network \cite{muller2019does}.
\end{enumerate}

\noindent
Motivated by these supportive findings, 
LS has recently been actively adopted, together with a neural network model, 
in various modern applications such as speech recognition \cite{chorowski2016towards},
machine translation \cite{vaswani2017attention, gao2020towards}, 
image classification \cite{zoph2018learning, huang2019gpipe}, 
and visual tracking \cite{han2020robust}.
However, in spite of the above-mentioned findings on LS and its wide use,
the underlying mechanism of LS has not been fully explored yet.
In this paper we study it and provide further understanding on LS,
including verification of the significance of squeezing of the logits stated in \hyperlink{A1}{A1} 
and supportive arguments on \hyperlink{A2}{A2} and \hyperlink{A3}{A3},
which is the first of two contributions of this paper.

%==========%OK
\begin{figure}[t]
\centering
\includegraphics[height=1.55cm, bb=0 0 874 154]{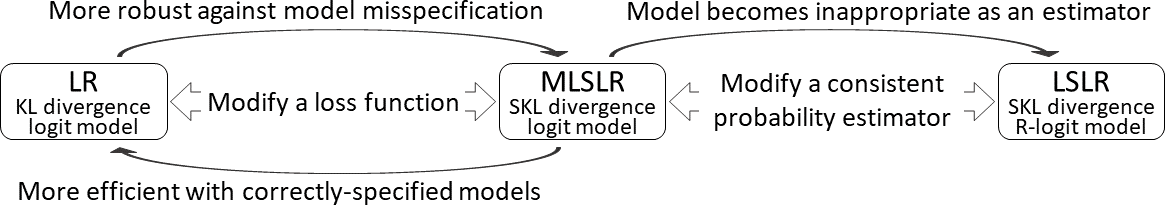}
\caption{%
Relationship between LR, LSLR, and MLSLR, and 
summary of the understanding on LS that this paper gives.}
\label{fig:relation}
\end{figure}
\IEEEpubidadjcol

%==========%OK
Although most previous studies have interpreted LSLR as entropy-regularized LR 
that squeezes the logits (Section \ref{sec:Regularization}),
we propose in this paper an alternative interpretation of LSLR 
in which LSLR performs probability estimation using a different 
loss function and consistent estimator than LR (Section \ref{sec:Loss}).
We then study the significance of each of these two modifications, 
via introducing a novel method, modified LSLR (MLSLR), 
that uses the same loss function as LSLR and 
the same consistent probability estimator as LR, 
which results in no squeezing of the logits (Section \ref{sec:MLS}).

%==========%OK
For the modification of the loss function,
we give theoretical comparisons between LR and MLSLR: 
Compared with LR, MLSLR with a larger smoothing level has lower efficiency 
with correctly-specified models (Section \ref{sec:Efficiency})
but higher robustness against model misspecification (Section \ref{sec:Robustness}).
As far as the authors' knowledge, this paper is the first to report 
the low efficiency as a disadvantage of LS, 
and the robustness is a basis for the adversarial robustness \hyperlink{A2}{A2} of LS.
Also, the trade-off between the low efficiency and the high robustness explains 
the better practical performance \hyperlink{A3}{A3}.

%==========%OK
Moreover, for the modification of the consistent probability estimator,
we prove that an estimator of LSLR can have output with 
an unnecessarily large range like $(-0.1, 1.1, 0,\ldots,0)^\top$, 
implying that it is inappropriate as a probability estimator.
We also experimentally compare LSLR and MLSLR
using a neural network model, and the result shows that 
LSLR performed worse than MLSLR (Section \ref{sec:Experiment}).
This experimental result implies that 
modifying the consistent probability estimator and squeezing the logits are not 
effective in improving the probability estimation and classification performance of LR, 
but rather have a negative impact,
as opposed to the previous understanding \hyperlink{A1}{A1}.
In other words, MLSLR based on a consistent probability estimator 
with an appropriate range would be recommended over LSLR, 
in typical usages with a large-size model such as a deep neural network model.
Collaterally, we propose to practically use MLSLR 
as the second of two contributions of this paper.

%==========%OK
Besides these findings on behaviors of LS
and the proposal of MLSLR as an improvement over LSLR, 
the final section (Section \ref{sec:Conclusion}) gives 
reference to other findings on LS \hyperlink{A4}{A4}--\hyperlink{A6}{6}
by previous studies, relevant topics, and future prospects.

%==================================================%OK
\section{Preliminaries}
\label{sec:Preliminaries}
%==================================================%OK
\subsection{Problem Formulation, Notation, and Terminology}
\label{sec:Notation}
%==========%OK
In this section, we discuss interpretations of LS.
We first formulate probability estimation and classification tasks, 
along with preparing required notations and terminologies.

%==========%OK
Suppose that one has the data $(\bx_1,y_1),\ldots,(\bx_n,y_n)\in\bbR^d\times[K]$ 
from the joint distribution of the explanatory random variable $\bX$ 
and target random variable $Y$, 
where $K$ is the number of different values that $y_1,\ldots,y_n$ take
and $[K]\coloneq\{1,\ldots,K\}$.
A classification task is defined in this paper as 
a task to obtain a good classifier $f:\bbR^d\to[K]$ 
such that the task risk $\calR_\tsk(f;\ell)\coloneq\bbE_{(\bX,Y)}[\ell(f(\bX),Y)]$, 
defined as the expectation value $\bbE_{(\bX,Y)}$ with respect to the pair $(\bX,Y)$ 
of a user-specified task loss function $\ell:[K]^2\to[0,\infty)$, is made small.
The task of minimizing misclassification rate corresponds to 
using the zero-one loss $\ell_\zo(j,k)\coloneq\mathbbm{1}(j\neq k)$ as the task loss, 
where $\mathbbm{1}(c)$ values 1 if a condition $c$ is true and 0 otherwise.

%==========%OK
Classification methods that we discuss in this paper rely on 
the framework of so-called empirical (surrogate) risk minimization (ERM).
Let $\calG=\{\bg:\bbR^d\to\calZ\}$ be a learner class,
where $\calZ$ denotes the set of values the learners in $\calG$ may output.
A classifier is constructed as $f=h\circ\bg$ with a labeling $h:\calZ\to[K]$ 
and a learner $\bg\in\calG$ that minimizes the empirical surrogate risk
$\frac{1}{n}\sum_{i=1}^n\phi(\bg(\bx_i),y_i)$ (which is an empirical counterpart of 
the surrogate risk $\calR_\sur(\bg;\phi)\coloneq\bbE_{(\bX,Y)}[\phi(\bg(\bX),Y)]$) 
for a surrogate loss function $\phi:\calZ\times[K]\to[0,\infty)$ that is continuous in its first argument.
ERM can be interpreted as estimation of 
the conditional probability distribution (CPD) function 
$\bp(\cdot)=(\Pr(Y=1|\bX=\cdot),\ldots,\Pr(Y=K|\bX=\cdot))^\top:\bbR^d\to\Delta_{K-1}$, 
where $\Delta_{K-1}$ is the probability simplex in $\bbR^K$,
and the labeling $h$ is designed according to that interpretation.

%==========%OK
Note that our notations show possibly-multivariate objects in bold
and often omit the $K$-dependence for the brevity.

%==================================================%OK
\subsection{LR: Conventional Logistic Regression}
\label{sec:LR}
%==========%OK
For the learner class $\calG=\{\bg:\bbR^d\to\bbR^K\}$,
%For the learner class $\calG=\{\bg=(g_1,\ldots,g_K)^\top:\bbR^d\to\bbR^K\}$,
LR solves
\begin{align}
\label{eq:LR}%
	\min_{\bg\in\calG}\biggl[-\frac{1}{n}\sum_{i=1}^n\sum_{k=1}^K
	t_k(y_i)\ln\Bigl(\tfrac{e^{g_k(\bx_i)}}{\sum_{l=1}^K e^{g_l(\bx_i)}}\Bigr)\biggr],
\end{align} 
where $\bt\coloneq(t_1,\ldots,t_K)^\top$ is the one-hot encoding function,
whose $k$-th component is $t_k(y)=\mathbbm{1}(k=y)$ for $k,y\in\mathcal{Y}$.
The obtained functions $\{g_k\}_{k\in[K]}$ are also called the logits.

%==========%OK
As $\bt(y)$ satisfies $\sum_{y=1}^Kt_k(y)\Pr(Y=y|\bX=\bx)=\Pr(Y=k|\bX=\bx)$,
a mean of the targets $\{\bt(y_i)\}_{\bx_i=\bx}$ can be seen 
as an empirical estimate of the CPD function $\bp(\bx)$.
The KL-divergence, defined for probability mass functions 
$\bp=(p_1,\ldots,p_K)^\top, \bq=(q_1,\ldots,q_K)^\top\in\Delta_{K-1}$ as
\begin{align}
\label{eq:KL-div}%
	D_\KL(\bp||\bq)\coloneq\sum_{k=1}^K p_k\ln\tfrac{p_k}{q_k},
\end{align}
has the consistency property
\begin{align}
\label{eq:KL-consis}%
	\underset{\bq\in\Delta_{K-1}}{\argmin}\;D_\KL(\bp||\bq)
	=\bp,\quad\text{for all }\bp\in\Delta_{K-1}.
\end{align}
This property shows that LR applies the logit model 
\begin{align}
\label{eq:LM}%
	\bq(\bx)=\bq_\LM(\bg(\bx))\coloneq
	\Bigl(\tfrac{e^{g_1(\bx)}}{\sum_{k=1}^K e^{g_k(\bx)}},\ldots,
	\tfrac{e^{g_K(\bx)}}{\sum_{k=1}^K e^{g_k(\bx)}}\Bigr)^\top
\end{align}
as a consistent estimator (see also Corollary \ref{cor:ProbPred}), 
in estimation of the true CPD function $\bp(\bx)$ through minimization of 
an empirical estimate of the mean KL-divergence $\bbE_\bX[D_\KL(\bp(\bX)||\bq(\bX))]$, 
where $\bbE_\bX$ denotes the expectation value regarding the random variable $\bX$.

%==================================================%OK
\subsection{LSLR: Logistic Regression with Label Smoothing}
\label{sec:LS}
%==========%OK
For the smoothing function 
\begin{align}
	\bs_\alpha(\bv)\coloneq(1-\alpha)\bv+\tfrac{\alpha}{K}
\end{align}
with a smoothing level $\alpha$ that is conventionally in $(0,1)$, 
LSLR applies the smoothed target 
$\bs_\alpha(\bt(y_i))=(s_\alpha(t_1(y_i)),\ldots,s_\alpha(t_K(y_i)))^\top$ instead of 
the one-hot encoded target $\bt(y_i)$ of LR.
Namely, LSLR considers the learning process
\begin{align}
\label{eq:LSLR}%
	\min_{g\in\calG}\biggl[-\frac{1}{n}\sum_{i=1}^n\sum_{k=1}^K
	s_\alpha(t_k(y_i))\ln\Bigl(\tfrac{e^{g_k(\bx_i)}}{\sum_{l=1}^K e^{g_l(\bx_i)}}\Bigr)\biggr].
\end{align}

%==================================================%OK
\subsection{Regularization View: LSLR is Entropy Regularized LR}
\label{sec:Regularization}
%==========%OK
Since the smoothed target satisfies $\sum_{y=1}^Ks_\alpha(t_k(y))\Pr(Y=y|\bX=\bx)=s_\alpha(\Pr(Y=k|\bX=\bx))$,
a mean of the smoothed targets $\{\bs_\alpha(\bt(y_i))\}_{\bx_i=\bx}$ 
can be seen as an empirical estimate of the smoothed CPD function 
$\bs_\alpha(\bp(\bx))=(s_\alpha(\Pr(Y=1|\bX=\bx)),\ldots,s_\alpha(\Pr(Y=K|\bX=\bx)))^\top$.
On the basis of this consideration, 
along with
an implicit supposition that LSLR adopts the logit model $\bq(\bx)=\bq_\LM(\bg(\bx))$ 
for estimation of the true CPD function $\bp(\bx)$, 
and the equation
\begin{align}
\label{eq:regularview}%
	\begin{split}
	D_\KL(\bs_\alpha(\bp)||\bq)
	=&\,(1-\alpha)D_\KL(\bp||\bq)
	+\alpha D_\KL(\bo/K||\bq)\\
	&+(\text{$\bq$-independent term})
	\end{split}
\end{align}
with the all-1 $K$-dimensional vector $\bo\coloneq(1,\ldots,1)^\top$,
most previous studies regard LSLR as LR
with the entropy regularization term $D_\KL(\bo/K||\bq)$
which penalizes the deviation of $\bq$ from the uniform CPD function $\bx\mapsto\bo/K$;
See \cite{szegedy2016rethinking, muller2019does, goibert2019adversarial, meister2020generalized}.
As a result, the learned logit model $\bq_\LM(\bg(\bx))$ is expected not 
to take an extreme probability estimate like $(1,0,\ldots)^\top$,
or equivalently, the logits by LSLR will be squeezed;
See Theorem \ref{thm:RLM-SKL}, \hyperlink{B6}{B6}.
Many previous studies claim that this squeezing helps avoid over-fitting of the model,
as in the finding \hyperlink{A1}{A1}.

%==================================================%OK
\subsection{Our Loss View: LSLR modifies Loss Function and Consistent Probability Estimator from LR}
\label{sec:Loss}
%==========%OK
We here introduce an alternative view of LS that 
LSLR adopts a loss function and consistent estimator
different from those of LR for probability estimation (the loss view).
First, we define the smoothed KL (SKL)-divergence
\begin{align}
\label{eq:SKL-div}%
	D_{\SKL,\alpha}(\bp||\bq)\coloneq\sum_{k=1}^Ks_\alpha(p_k)\ln\tfrac{s_\alpha(p_k)}{s_\alpha(q_k)}
\end{align} 
for $\bp,\bq\in\bs_\alpha^{-1}(\Delta_{K-1})\coloneq\{\bs_\alpha^{-1}(\bp)\mid\bp\in\Delta_{K-1}\}$.
The SKL-divergence satisfies the consistency property
\begin{align}
\label{eq:SKL-consis}%
	\underset{\bq\in\calS}{\argmin}\;D_{\SKL,\alpha}(\bp||\bq)
	=\bp,\quad\text{for all }\bp\in\calS
\end{align}
for $\calS=\Delta_{K-1}$ or $\bs_\alpha^{-1}(\Delta_{K-1})$ and 
$\alpha$ in a certain range (see Theorem \ref{thm:RLM-SKL}, \hyperlink{B3}{B3}).
Then, we view that LSLR \eqref{eq:LSLR} adopts 
not the logit model $\bq_\LM(\bg(\bx))$ but
what we call the roughened logit (R-logit) model,
which is defined with an unconventional link function as
\begin{align}
\label{eq:RLM}%
	\begin{split}
	&\bq_{\RLM,\alpha}(\bg(\bx))
	\coloneq\bs_\alpha^{-1}(\bq_\LM(\bg(\bx)))\\
	&\hphantom{\bq_{\RLM}}
	=\Bigl(s_\alpha^{-1}\Bigl(\tfrac{e^{g_1(\bx)}}{\sum_{k=1}^K e^{g_k(\bx)}}\Bigr),\ldots,
	s_\alpha^{-1}\Bigl(\tfrac{e^{g_K(\bx)}}{\sum_{k=1}^K e^{g_k(\bx)}}\Bigr)\Bigr)^\top,
	\end{split}
\end{align}
as an estimator (as $\bq(\bx)$) of the true CPD function $\bp(\bx)$
through minimizing an empirical estimate of the mean SKL-divergence $\bbE_\bX[D_{\SKL,\alpha}(\bp(\bX)||\bq(\bX))]$.
As one can see from the fact that
$\bs_\alpha$ appearing in the loss \eqref{eq:SKL-div} and 
$\bs_\alpha^{-1}$ in the model \eqref{eq:RLM} cancel out each other,
this estimator, not the logit model, is consistent to the true CPD for LSLR.%
\footnote{%
One may see that LSLR estimates the smoothed CPD function $\bs_\alpha(\bp(\bx))$ with the logit model $\bq_\LM(\bg(\bx))$.
However, since this view changing properties of the data makes it difficult to compare with the original LR, 
our paper will not discuss this view further.}

%==========%OK
Note that the logit model is not the only probability estimator;
For example, another well-known probability estimator is the probit model in probit regression.
Our loss view uniformly determines 
which parts we treat as a probability estimator and 
which parts we treat as a loss function for a probability estimation method, 
according to the consistency of the probability estimator to the true CPD function,
for fair comparisons to be performed in this paper.

%==========%OK
We here list basic properties of the SKL-divergence $D_{\SKL,\alpha}$ 
and R-logit model $\bq_{\RLM,\alpha}$, which are the KL-divergence 
$D_{\KL}$ and logit model $\bq_\LM$ when $\alpha=0$:
\begin{theorem}
\label{thm:RLM-SKL}
For any $K\ge2$,
\begin{enumerate}
\item[{\hypertarget{B1}{B1}.}]
$q_{\RLM,\alpha,k}(\bg)$ for $\bg\in\bbR^K$ and $k=1,\ldots,K$ 
can range $[-\frac{\alpha}{K(1-\alpha)},\frac{K-\alpha}{K(1-\alpha)}]$
or $[\frac{K-\alpha}{K(1-\alpha)},-\frac{\alpha}{K(1-\alpha)}]$ 
for $\alpha\in[0,1)$ or $\alpha\in(1,\frac{K}{K-1}]$, 
and these intervals cover $[0,1]$.
\item[{\hypertarget{B2}{B2}.}]
$\sum_{k=1}^Kq_{\RLM,\alpha,k}(\bg)=1$ for any $\bg\in\bbR^K$ and $\alpha\neq1$. 
\item[{\hypertarget{B3}{B3}.}]
$D_{\SKL,\alpha}(\bp||\bq)\ge0$, and $D_{\SKL,\alpha}(\bp||\bq)=0$ if and only if $\bp=\bq$, 
for any $\bp,\bq\in\bs_\alpha^{-1}(\Delta_{K-1})$ and $\alpha\in[0,1)\cup(1,\frac{K}{K-1}]$.
Also, $D_{\SKL,1}(\bp||\bq)=0$ for any $\bp,\bq\in\bs_\alpha^{-1}(\Delta_{K-1})$.
\item[{\hypertarget{B4}{B4}.}]
$D_{\SKL,\alpha}(\bp||\bq)$ is convex in $(\bp,\bq)$ for any $\alpha\in[0,\frac{K}{K-1}]$:
$D_{\SKL,\alpha}(r\bp_1+(1-r)\bp_2||r\bq_1+(1-r)\bq_2)\le rD_{\SKL,\alpha}(\bp_1||\bq_1)+(1-r)D_{\SKL,\alpha}(\bp_2||\bq_2)$,
for any $\bp_1,\bp_2,\bq_1,\bq_2\in\bs_\alpha^{-1}(\Delta_{K-1})$ and $r\in[0,1]$.
\item[{\hypertarget{B5}{B5}.}]
$\argmin_{\bg\in\bbR^K, g_K=0} D_{\SKL,\alpha}((1,0,\ldots)^\top||\bq_\LM(\bg))
=(\infty,0,\ldots)^\top$ for $\alpha\in[0,1)\cup(1,\frac{K}{K-1}]$.
\item[{\hypertarget{B6}{B6}.}]
$\argmin_{\bg\in\bbR^K, g_K=0} D_{\SKL,\alpha}((1,0,\ldots)^\top||\bq_{\RLM,\alpha}(\bg))
=(\infty,0,\ldots)^\top$, $(\ln(\frac{K}{\alpha}+1-K),0,\ldots)^\top$, or $(-\infty,0,\ldots)^\top$
for $\alpha=0$, any $\alpha\in(0,1)\cup(1,\frac{K}{K-1})$, or $\alpha=\frac{K}{K-1}$.
\item[{\hypertarget{B7}{B7}.}]
$D_{\SKL,\alpha}(\bp||\bq)=D_{\SKL,K-(K-1)\alpha}(\frac{\bo-\bp}{K-1}||\frac{\bo-\bq}{K-1})$,
$D_{\SKL,\alpha}(\bp||\bq_{\RLM,\alpha}(\bg))=D_{\SKL,K-(K-1)\alpha}(\frac{\bo-\bp}{K-1}||\bq_{\RLM,K-(K-1)\alpha}(\bg))$,
$K-(K-1)\alpha\in[0,1)$,
and $\frac{\bo-\bp}{K-1}\in\bs_{K-(K-1)\alpha}^{-1}(\Delta_{K-1})$,
for any $\bp,\bq\in\bs_\alpha^{-1}(\Delta_{K-1})$, $\bg\in\bbR^K$, and $\alpha\in(1,\frac{K}{K-1}]$.
\end{enumerate}
For $K=2$, 
\begin{enumerate}
\item[{\hypertarget{B8}{B8}.}]
$D_{\SKL,\alpha}(\bp||\bq_{\LM}(\bg))=D_{\SKL,2-\alpha}(\bo-\bp||\bq_{\LM}(-\bg))= D_{\SKL,2-\alpha}(\bp||\bq_{\LM}(\bg))$,
$D_{\SKL,\alpha}(\bp||\bq_{\RLM,\alpha}(\bg))=D_{\SKL,2-\alpha}(\bo-\bp||\bq_{\RLM,2-\alpha}(\bg))=D_{\SKL,2-\alpha}(\bp||\bq_{\RLM,2-\alpha}(-\bg))$,
$2-\alpha\in[0,1)$,
and $\bo-\bp\in\bs_{2-\alpha}^{-1}(\Delta_1)$,
for any $\bp\in\bs_\alpha^{-1}(\Delta_1)$, $\bg\in\bbR^2$, and $\alpha\in(1,2]$.
\end{enumerate}
\end{theorem}
The constraint $g_K=0$ in \hyperlink{B5}{B5} and \hyperlink{B6}{B6}
is for removing the degree of freedom of translation of the minimizers;
Consider that, for example, if $\bg=\bar{\bg}$ minimizes $D_{\SKL,\alpha}(\bp||\bq_\LM(\bg))$,
then $\bg=\bar{\bg}+v\bo$ also minimizes that SKL-divergence for any $v\in\bbR$.
The consistency property \eqref{eq:SKL-consis} holds for 
$\alpha\in[0,1)\cup(1,\frac{K}{K-1}]$ as stated in \hyperlink{B3}{B3},
and one can perform LSLR and MLSLR (formulated in the next section) even with $\alpha\in(1,\frac{K}{K-1}]$.
However, we have found no advantage of choosing $\alpha\in(1,\frac{K}{K-1}]$ for $K>2$,
and there is no advantage at all especially when $K=2$ because
LSLRs or MLSLRs with $\alpha\in(1,2]$ and $2-\alpha\in[0,1)$ 
behave equivalently as suggested by \hyperlink{B8}{B8},
so we restrict our subsequent discussion to $\alpha\in[0,1)$
to simplify the statement of the paper;
See appendices for supplemental discussions 
on the case with $\alpha\in(1,\frac{K}{K-1}]$.

%==================================================%OK
\subsection{MLSLR: Modified LSLR}
\label{sec:MLS}
%==========%OK
As we described in the previous section,
LS of LSLR modifies a loss function and a consistent estimator for probability estimation, 
from the KL-divergence and logit model of LR 
to the SKL-divergence and R-logit model, respectively.
A direct comparison between LR and LSLR does not distinguish 
effects of these two modifications on behaviors of LS.
We therefore propose a novel method
which we call the modified LSLR (MLSLR),
\begin{align}
\label{eq:MLSLR}%
	\min_{\bg\in\calG}\biggl[-\frac{1}{n}\sum_{i=1}^n\sum_{k=1}^K
	s_\alpha(t_k(y_i))\ln s_\alpha\Bigl(\tfrac{e^{g_k(\bx_i)}}{\sum_{l=1}^Ke^{g_l(\bx_i)}}\Bigr)\biggr]
\end{align}
that adopts as the loss function the same SKL-divergence $D_{\SKL,\alpha}(\bp||\bq)$ as LSLR, 
along with the same logit model $\bq(\bx)=\bq_\LM(\bg(\bx))$ as LR for estimating the CPD function $\bp(\bx)$.
MLSLR allows us to study significance of the modification of the loss function
via its comparison with LR, as well as significance of
the modification of the estimator via its comparison with LSLR.
It should be noted that the latter comparison will also
clarify the importance of squeezing the logits in LSLR,
as MLSLR does not squeeze the logits
(see Theorem \ref{thm:RLM-SKL}, \hyperlink{B5}{B5}).

%==================================================%OK
\subsection{LSQLR: MLSLR with $\alpha\to1$}
\label{sec:SQ}
%==========%OK
We find it useful to consider the limit $\alpha\to1$ of MLSLR
in understanding properties of MLSLR and LSLR.
\begin{theorem}
\label{thm:limit}
For any $K\ge2$ and $\bp,\bq\in\bs_\alpha^{-1}(\Delta_{K-1})$,
$\tfrac{2\alpha}{(1-\alpha)^2K}D_{\SKL,\alpha}(\bp||\bq)\approx\|\bp-\bq\|^2$ as $\alpha\to1$, 
where $\|\cdot\|$ is the Euclidean norm in $\bbR^K$.
\end{theorem}
This theorem indicates especially that MLSLR with a limiting smoothing level $\alpha\to1$ 
approaches least squares logistic regression (LSQLR)
\begin{align}
\label{eq:LSQLR}%
	\min_{\bg\in\calG}\biggl[\frac{1}{n}\sum_{i=1}^n\sum_{k=1}^K
	\Bigl(t_k(y_i)-\tfrac{e^{g_k(\bx_i)}}{\sum_{l=1}^Ke^{g_l(\bx_i)}}\Bigr)^2\biggr].
\end{align}
LSQLR adopts the logit model $\bq_\LM(\bg(\bx))$ as a consistent estimator (as $\bq(\bx)$) of the true CPD function $\bp(\bx)$ via minimizing 
an empirical version of the mean squared distance $\bbE_\bX[D_\SQ(\bp(\bX)||\bq(\bX))]$
(plus $\bq$-independent quantity $\{1-\bbE_\bX[\|\bp(\bX)\|^2]\}$),
where $D_\SQ(\bp||\bq)\coloneq\|\bp-\bq\|^2$.
Then, one can regard MLSLR with $\alpha\in(0,1)$ as interpolating
between LR ($\alpha=0$) and LSQLR ($\alpha\to1$).

%==========%OK
In contrast, 
even if considering the consequence of Theorem~\ref{thm:limit} into account, 
LSLR with $\alpha\to1$
\begin{align}
	\min_{\bg\in\calG}\biggl[\frac{1}{n}\sum_{i=1}^n\sum_{k=1}^K
	\Bigl\{t_k(y_i)-s_\alpha^{-1}\Bigl(\tfrac{e^{g_k(\bx_i)}}{\sum_{l=1}^Ke^{g_l(\bx_i)}}\Bigr)\Bigr\}^2\biggr]
\end{align}
remains the $\alpha$-dependence,
and this method is not practical since $s_\alpha^{-1}$ diverges almost everywhere.
Due to this trouble, we do not study LSLR using $\alpha$ very close to 1.

%==================================================%OK
\subsection{Summary on Our Comparing Methods}
\label{sec:Summary}
%==========%OK
LR applies the KL-divergence loss and logit model 
(see Section~\ref{sec:LR} and \eqref{eq:LR}),
LSLR applies the SKL-divergence loss and R-logit model 
(see Sections~\ref{sec:LS} and \ref{sec:Loss}, and \eqref{eq:LSLR}),
and MLSLR including LSQLR applies the SKL-divergence loss and logit model 
(see Sections~\ref{sec:MLS} and \ref{sec:SQ}, and \eqref{eq:MLSLR} and \eqref{eq:LSQLR}).
This paper formulated these methods according to the framework of ERM,
and these methods do not have an explicit term for regularization in our formulations
(so we do not use the terminology `regularization' except when discussing statements of existing studies).
We will compare these methods that use different loss functions and consistent probability estimators, 
under settings regarding the underlying data distribution and data characteristics, 
and the model size or representation ability of the probability estimators,
according to classical analysis for ERM methods.

%==========%OK
One may be concerned about the remaining one of the four combinations 
of the KL- or SKL-divergence loss and the logit or R-logit model.
The following problem corresponds to the combination 
of the KL-divergence loss and the R-logit model:
\begin{align}
\label{eq:Other}%
	\min_{\bg\in\calG}\biggl[-\frac{1}{n}\sum_{i=1}^n\sum_{k=1}^K
	t_k(y_i)\ln s_\alpha^{-1}\Bigl(\tfrac{e^{g_k(\bx_i)}}{\sum_{l=1}^Ke^{g_l(\bx_i)}}\Bigr)\biggr].
\end{align}
However, an element of the R-logit model can take a negative value, 
so the KL-divergence for this model will be ill-defined and the optimization will fail.
This method is therefore not promising and will not be considered in this study.

%==================================================%OK
\section{Statistical Analysis: LR versus MLSLR}
\label{sec:Theory}
%==================================================%OK
\subsection{Setting, Notation, and Basic Properties}
\label{sec:Setting}
%==========%OK
In the theoretical analysis presented in this section,
for the sake of interpretability of the comparison results,
we consider the simple case of binary classification ($K=2$)
for the task of minimizing misclassification rate ($\ell=\ell_\zo$).
See Appendix~\ref{sec:ProCla}
for analysis under more general settings including 
multi-class cases and cost-sensitive tasks.

%==========%OK
We write the distribution $\Pr(\bX=\bx)$ of $\bX$ as $p_0(\bx)$,
relabel $Y=1,2$ to $+1, -1$, 
and abbreviate the conditional probability $\Pr(Y=+1|\bX=\bx)$ as $p_1(\bx)$.
Then, for the real-valued learner class $\calG=\{g:\bbR^d\to\bbR\}$,
the surrogate loss functions $\phi$ for LR, LSLR, MLSLR, and LSQLR 
are respectively given by $\phi(v, y)=\varphi(y v)$ with
\begin{align}
\label{eq:SurrogateLoss}
	\begin{split}
	\varphi_\LR(v)
	&=-\ln\bigl(\tfrac{1}{1+e^{-v}}\bigr),\\
	\varphi_{\LS,\alpha}(v)
	&=-\bigl(1-\tfrac{\alpha}{2}\bigr)\ln\bigl(\tfrac{1}{1+e^{-v}}\bigr)-\tfrac{\alpha}{2}\ln\bigl(\tfrac{1}{1+e^v}\bigr),\\
	\varphi_{\MLS,\alpha}(v)
	&=-\bigl(1-\tfrac{\alpha}{2}\bigr)\ln\bigl(\tfrac{1-\alpha}{1+e^{-v}}+\tfrac{\alpha}{2}\bigr)
	-\tfrac{\alpha}{2}\ln\bigl(\tfrac{1-\alpha}{1+e^{v}}+\tfrac{\alpha}{2}\bigr),\\
	\varphi_\LSQ(v)
	&=\tfrac{1}{2(1+e^v)^2}\,\bigl(\propto
	\bigl(1-\tfrac{1}{1+e^{-v}}\bigr)^2+\bigl(0-\tfrac{1}{1+e^v}\bigr)^2\bigr),
	\end{split}
\end{align}
where $\varphi$ itself is also called a surrogate loss
(subscript LR, LS, MLS, or LSQ of an object indicates that it is for LR, LSLR, MLSLR, or LSQLR).%
\footnote{%
The learner model $g$ is written as a real-valued function in Section \ref{sec:Theory},
which is a simplification in the binary case, from the formulations \eqref{eq:LR}, 
\eqref{eq:LSLR}, \eqref{eq:MLSLR}, and \eqref{eq:LSQLR} of the four methods that 
adopted an $\bbR^2$-valued function, by letting $(g(\bx), 0)^\top$ be the model.}
The loss function $\varphi_\LSQ$ is also known as Savage loss \cite{masnadi2008design}.
Also, a labeling function $h$ is fixed to the sign function, 
$h_{\ell_\zo}(v)\coloneq-1$ (if $v\le0$), $\coloneq+1$ (if $v>0$),
considering the task and surrogate losses.

%==========%OK
First, we summarize results showing that LR, MLSLR, and LSQLR can 
consistently perform probability estimation via the logit model,
while LSLR can via the R-logit model:
\begin{corollary}
\label{cor:ProbPred}
Assume $\alpha\in[0,1)$,
and let $\bar{g}\in\argmin_{g:\bbR^d\to\bbR}\calR_\sur(g; \phi)$.
Then, regardless of the distribution of $(\bX,Y)$,
$\frac{1}{1+e^{-\bar{g}(\bx)}}=p_1(\bx)$ a.s.~for LR, MLSLR, and LSQLR,
and $s_\alpha^{-1}\bigl(\frac{1}{1+e^{-\bar{g}(\bx)}}\bigr)=p_1(\bx)$,
a.s.~for LSLR.
\end{corollary}

%==========%OK
Also, it can be found that surrogate loss $\phi$ is properly designed 
for the task loss $\ell_\zo$ under the labeling function $h_{\ell_\zo}$.
\begin{corollary}
\label{cor:ClasCali}
Assume $\alpha\in[0,1)$,
and let $\bar{g}=\argmin_{g:\bbR^d\to\bbR}\calR_\sur(g; \phi)$.
Then, regardless of the distribution of $(\bX,Y)$,
$\calR_\tsk(h_{\ell_\zo}\circ\bar{g};\ell_\zo)=\inf_{f:\bbR^d\to[K]}\calR_\tsk(f;\ell_\zo)$
for LR, LSLR, MLSLR, and LSQLR.
\end{corollary}
This result can be proved from the fact that the loss 
$\varphi$ is classification calibrated; 
Refer to \cite{bartlett2006convexity}.
Also, this result can be generalized to the multi-class cases ($K>2$) 
and cost-sensitive tasks ($\ell\neq\ell_\zo$) \cite{pires2013cost}.

%==========%OK
These properties form an important basis for analysis of the performance in 
probability estimation and classification tasks 
with an empirical surrogate risk minimizer, 
which are the subjects below.
These results indicate that the methods have no difference in the limit of the performances, 
and suggest that we should discuss their estimation performance (error) 
and the adequacy of the models in more specific settings.

%==================================================%OK
\subsection{Lower Efficiency with Correctly-Specified Models}
\label{sec:Efficiency}
%==================================================%OK
\subsubsection{Theories on Asymptotic Behaviors}
\label{sec:AB}
%==========%OK
In order to make a detailed comparison, 
we focus in this section on a specific case where 
the data are distributed in association with a certain linear model and 
methods adopt a linear learner class $\calG$ (the correctly-specified model).
Namely, assuming that the true conditional positive probability is 
$p_1(\bx)=\frac{1}{1+e^{-\tilde{\bbeta}^\top\bx}}$ for any $\bx\in\bbR^d$,
we study the estimation result of the true parameter $\tilde{\bbeta}$ 
by LR, MLSLR, and LSQLR that use the same logit model with the learner class 
$\calG=\{\bg(\cdot)=\bbeta^\top\cdot\mid\bbeta\in\bbR^d\}$ as a probability estimator.
Since LSLR adopts a consistent probability estimator 
(the R-logit model) different from the others, 
it will not give a consistent parameter estimate of 
$\tilde{\bbeta}$ under the above-mentioned setting.
Thus, it is difficult to perform fair comparisons of LSLR with the other methods,
and so we here consider only LR and MLSLR including LSQLR.
The only difference between these compared methods lies in their loss functions.

%==========%OK
With the correctly-specified model, one can show consistency, asymptotic normality, and 
asymptotic mean squared error (AMSE) of an empirical parameter estimate defined by
\begin{align}
\label{eq:estimate}
	\hat{\bbeta}_n\coloneq\underset{\bbeta\in\bbR^d}{\argmin}\;
	\frac{1}{n}\sum_{i=1}^n\phi(\bbeta^\top\bx_i, y_i),
\end{align}
on the basis of well-established theories for 
generalized linear models (see \cite[Section 3]{fahrmeir1985consistency} for LR) or 
M-estimators (see \cite{bianco1996robust} or \cite[Chapter 6]{huber2004robust} for MLSLR and LSQLR).
\begin{theorem}
\label{thm:consistency}
For LR, MLSLR, or LSQLR, assume $\alpha\in[0,1)$,
\begin{enumerate}
\item[{\hypertarget{C1}{C1}.}]
$\phi$ is a surrogate loss function defined by \eqref{eq:SurrogateLoss}.
\item[{\hypertarget{C2}{C2}.}]
$\Pr(\bbeta^\top\bX=0)=0$ for any $\bbeta\in\bbR^d$ such that $\bbeta\neq\bzero$.
\item[{\hypertarget{C3}{C3}.}]
$p_1(\bx)=\frac{1}{1+e^{-\tilde{\bbeta}^\top\bx}}$ for any $\bx\in\bbR^d$ and some $\tilde{\bbeta}\in\bbR^d$.
\end{enumerate}
Then, $\hat{\bbeta}_n$ defined by \eqref{eq:estimate} converges almost surely to $\tilde{\bbeta}$.
\end{theorem}
\begin{theorem}
\label{thm:normality}
Assume \hyperlink{C1}{C1}--\hyperlink{C3}{C3} in Theorem \ref{thm:consistency}, $\alpha\in[0,1)$, and
\begin{enumerate}
\item[{\hypertarget{C4}{C4}.}]
$\bbE_\bX[\|\bX\|^2]<\infty$ for LR, 
or $\bbE_\bX[\|\bX\|^3]<\infty$ for MLSLR and LSQLR.
\end{enumerate}
Then, for $\hat{\bbeta}_n$ defined by \eqref{eq:estimate}, 
$\sqrt{n}(\hat{\bbeta}_n-\tilde{\bbeta})$ converges in distribution to 
a $d$-dimensional normal distribution with mean $\bzero$ and 
covariance matrix $\rmC=\rmB^{-1}\rmA\rmB^{-1}$ with
\begin{align}
\label{eq:rmAB}
	\begin{split}
	&\rmA=\bbE_{(\bX,Y)}\bigl[\nabla\phi(\tilde{\bbeta}^\top\bX,Y)\nabla\phi(\tilde{\bbeta}^\top\bX,Y)^\top \bigr],\\
	&\rmB=\bbE_{(\bX,Y)}\bigl[\nabla^2\phi(\tilde{\bbeta}^\top\bX,Y)\bigr],
	\end{split}
\end{align}
where $\nabla$ and $\nabla^2$ are respectively the nabla 
and Hessian operator with respect to the model parameter,
and where $\rmA$ and $\rmB$ for LR, MLSLR, and LSQLR are
\begin{align}
\label{eq:rmABs}
	\rmA_\LR&=\rmB_\LR=\bbE_\bX\bigl[p_1(\bX)\{1-p_1(\bX)\}\bX\bX^\top\bigr],\nonumber\\
	\rmA_{\MLS,\alpha}&=\bbE_\bX\Bigl[\tfrac{\{p_1(\bX)\}^3\{1-p_1(\bX)\}^3\bX\bX^\top}{\{p_1(\bX)-\alpha(p_1(\bX)-\frac{1}{2})\}^2\{1-p_1(\bX)+\alpha(p_1(\bX)-\frac{1}{2})\}^2}\Bigr],\nonumber\\
	\rmB_{\MLS,\alpha}&=\bbE_\bX\Bigl[\tfrac{\{p_1(\bX)\}^2\{1-p_1(\bX)\}^2\bX\bX^\top}{\{p_1(\bX)-\alpha(p_1(\bX)-\frac{1}{2})\}\{1-p_1(\bX)+\alpha(p_1(\bX)-\frac{1}{2})\}}\Bigr],\nonumber\\
	\rmA_\LSQ&=\bbE_\bX\bigl[\{p_1(\bX)\}^3\{1-p_1(\bX)\}^3\bX\bX^\top\bigr],\nonumber\\
	\rmB_\LSQ&=\bbE_\bX\bigl[\{p_1(\bX)\}^2\{1-p_1(\bX)\}^2\bX\bX^\top\bigr].
\end{align}
\end{theorem}

%==========%OK
In the large-sample limit $n\to\infty$, the AMSE of $\hat{\bbeta}_n$ becomes
\begin{align}
	n\bbE\bigl[\|\hat{\bbeta}_n-\tilde{\bbeta}\|^2\bigr]\to\trace(\rmC).
\end{align}
Therefore, the ratio $\trace(\rmC_\LR)/\trace(\rmC)$ (which is called asymptotic relative efficiency, ARE)
serves as an indicator of the estimation performance of a method corresponding to the matrix $\rmC$;
The smaller the ARE is, the lower the asymptotic efficiency of the method is
(a larger-size sample is needed to achieve the same level of estimation performance as LR).
Note that it can be theoretically found that LR gives asymptotically the most efficient estimate
by considering the Cram\'er-Rao bound since LR is a maximum likelihood method.

%==========%OK
\begin{table}[t]
\centering
\renewcommand{\tabcolsep}{5pt}
\caption{%
ARE in estimation of $\tilde{\bbeta}$ under \eqref{eq:nominal}.
The lower it values, the less efficient it is.
Note that ARE is 1 if $\tilde{\beta}_2=0$.}
\label{tab:ARE}
\begin{tabular}{c|ccccc}
\toprule
\multirow{2}{*}{$\tilde{\bbeta}^\top$} & \multicolumn{4}{c}{MLSLR} & LSQLR \\
 & $\alpha=0.2$ & $\alpha=0.4$ & $\alpha=0.6$ & $\alpha=0.8$ & $\alpha\to1$ \\
\midrule
%$(0,.5)$ & .9982 & .9954 & .9929 & .9913 & .9907 \\
$(0,1)$ & .9815 & .9627 & .9496 & .9421 & .9396 \\
$(0,2)$ & .9043 & .8531 & .8239 & .8085 & .8036 \\
$(0,4)$ & .7815 & .7112 & .6749 & .6566 & .6510 \\
%$(1,.5)$ & .9898 & .9771 & .9675 & .9618 & .9599 \\
$(1,1)$ & .9604 & .9286 & .9085 & .8972 & .8937 \\
$(1,2)$ & .8844 & .8282 & .7969 & .7805 & .7754 \\
$(1,4)$ & .7760 & .7051 & .6688 & .6504 & .6447 \\
%$(2,.5)$ & .9580 & .9224 & .8997 & .8872 & .8831 \\
$(2,1)$ & .8915 & .8279 & .7916 & .7725 & .7665 \\
$(2,2)$ & .8281 & .7605 & .7248 & .7065 & .7009 \\
$(2,4)$ & .7605 & .6885 & .6518 & .6334 & .6277 \\
\bottomrule
\end{tabular}
\end{table}

%==========%OK
Table \ref{tab:ARE} shows AREs of MLSLR and LSQLR in the example where 
the data with a 2-dimensional covariate follow the distribution $\tilde{F}$, 
in which the probability density at $(\bX,Y)=(\bx,+1)$ is given as a product of
\begin{align}
\label{eq:nominal}
	p_0(\bx)=\delta_{X_1}(1)\cdot\tfrac{1}{\sqrt{2\pi}}\exp\bigl(-\tfrac{1}{2}x_2^2\bigr),\quad%\mathbbm{1}(x_1=1)\calN(x_2),\quad
	p_1(\bx)=\tfrac{1}{1+e^{-\tilde{\bbeta}^\top\bx}}
\end{align}
with $\tilde{\beta}_1=0,1,2$, $\tilde{\beta}_2=1,2,4$,
where $\delta_\bZ(\bz)$ is a point mass distribution at $\bZ=\bz$.
%, and $\calN(v)\coloneq\frac{1}{\sqrt{2\pi}}\exp(-\frac{1}{2}v^2)$.
%
It indicates that the asymptotic efficiency of MLSLR tends to decrease, 
as the smoothing level $\alpha$ increases to 1.

%==================================================%OK
\subsubsection{Simulation Experiment}
\label{sec:ESimulation}
\begin{figure}[t]
\centering
\begin{tabular}{cc}
\includegraphics[width=4cm, bb=0 0 780 348]{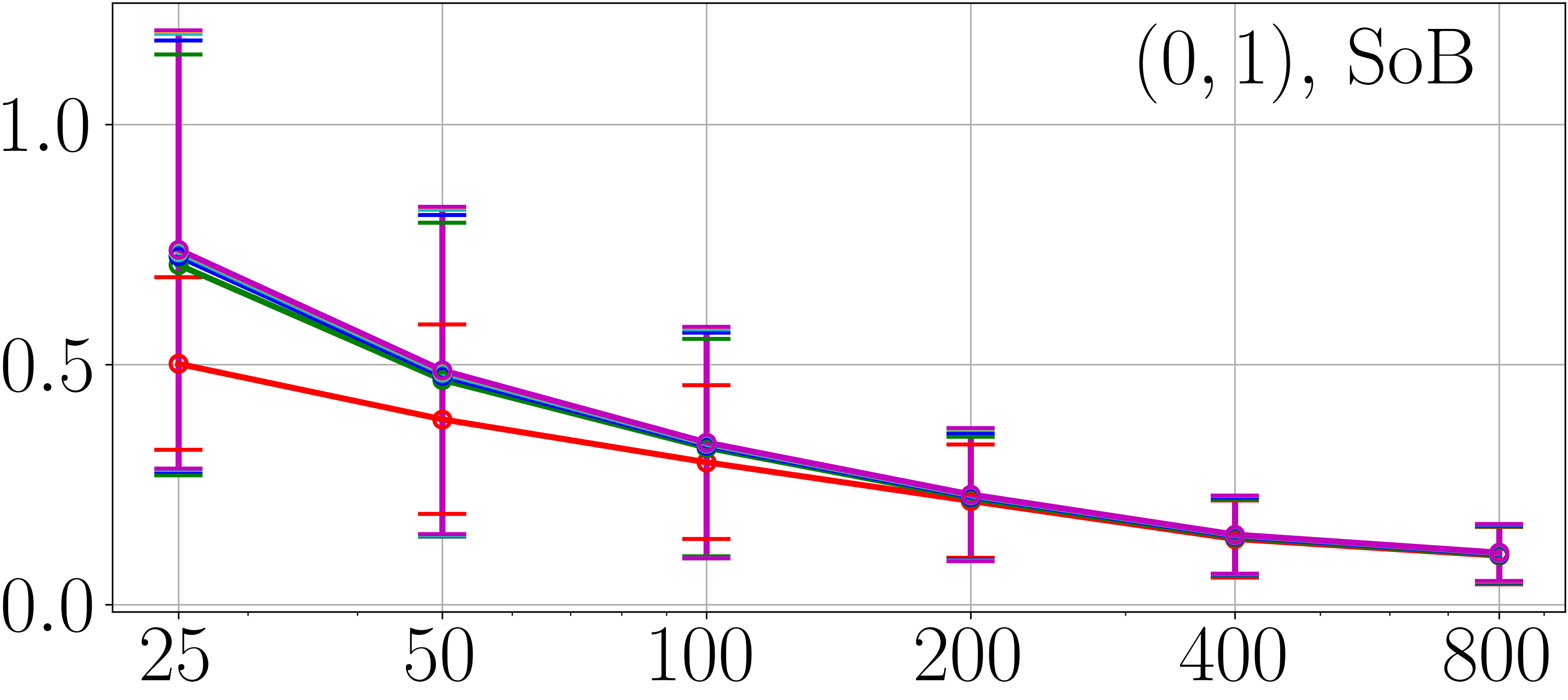}&
\includegraphics[width=4cm, bb=0 0 780 348]{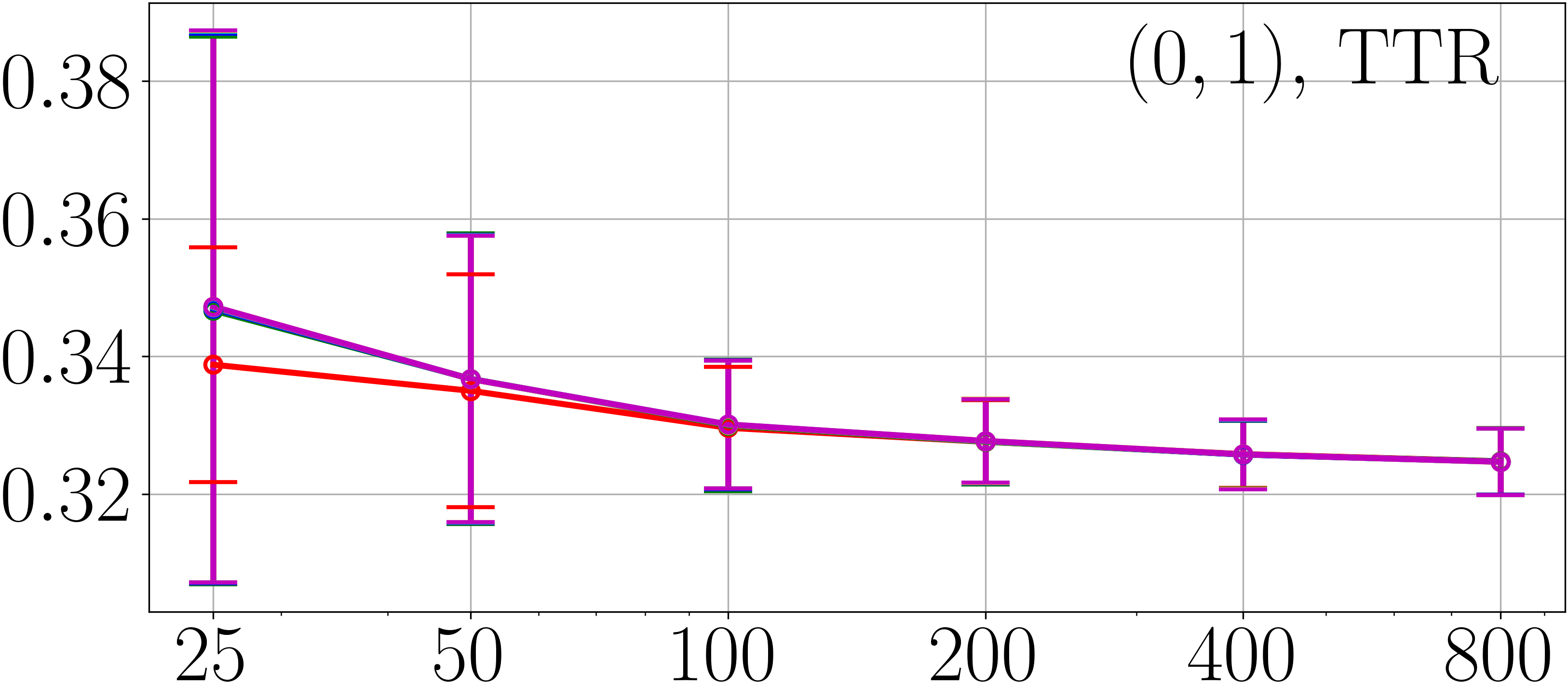}\\
\includegraphics[width=4cm, bb=0 0 780 348]{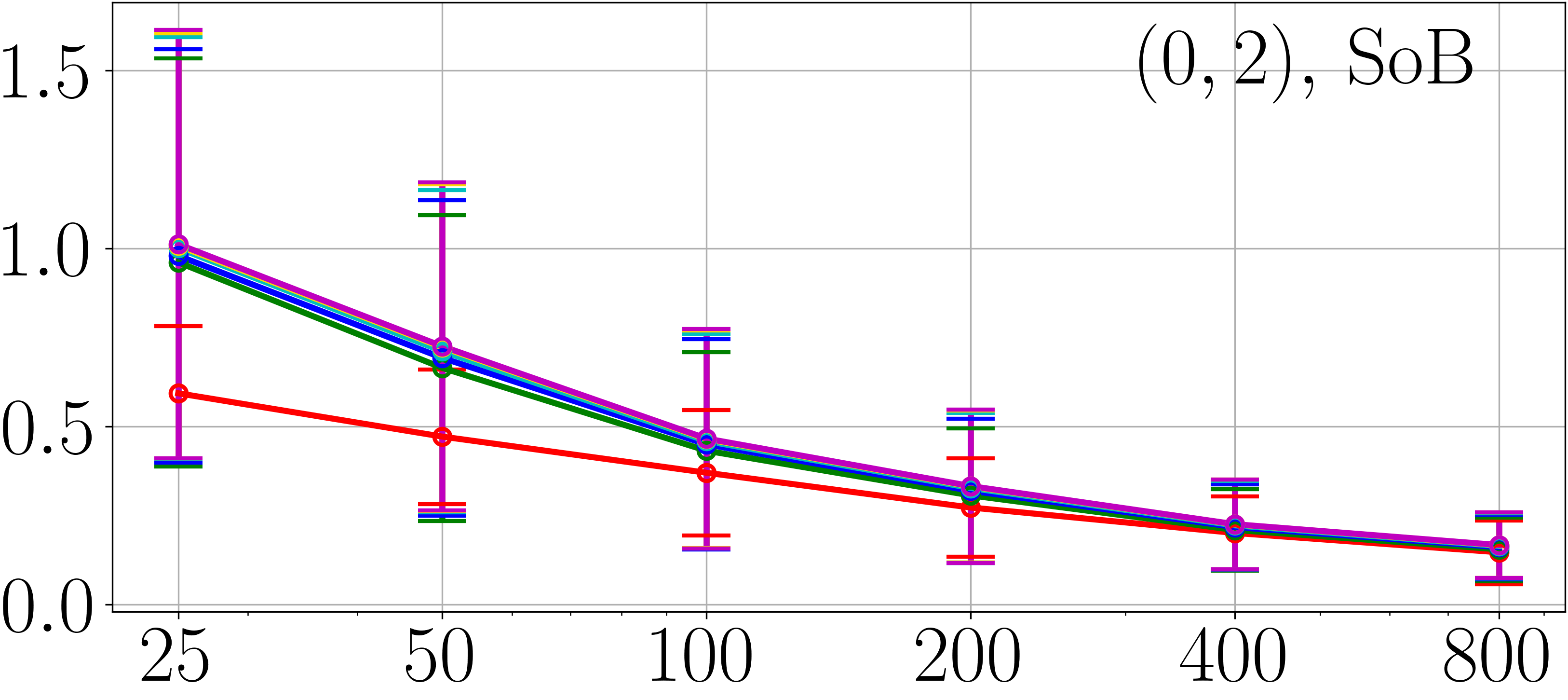}&
\includegraphics[width=4cm, bb=0 0 780 348]{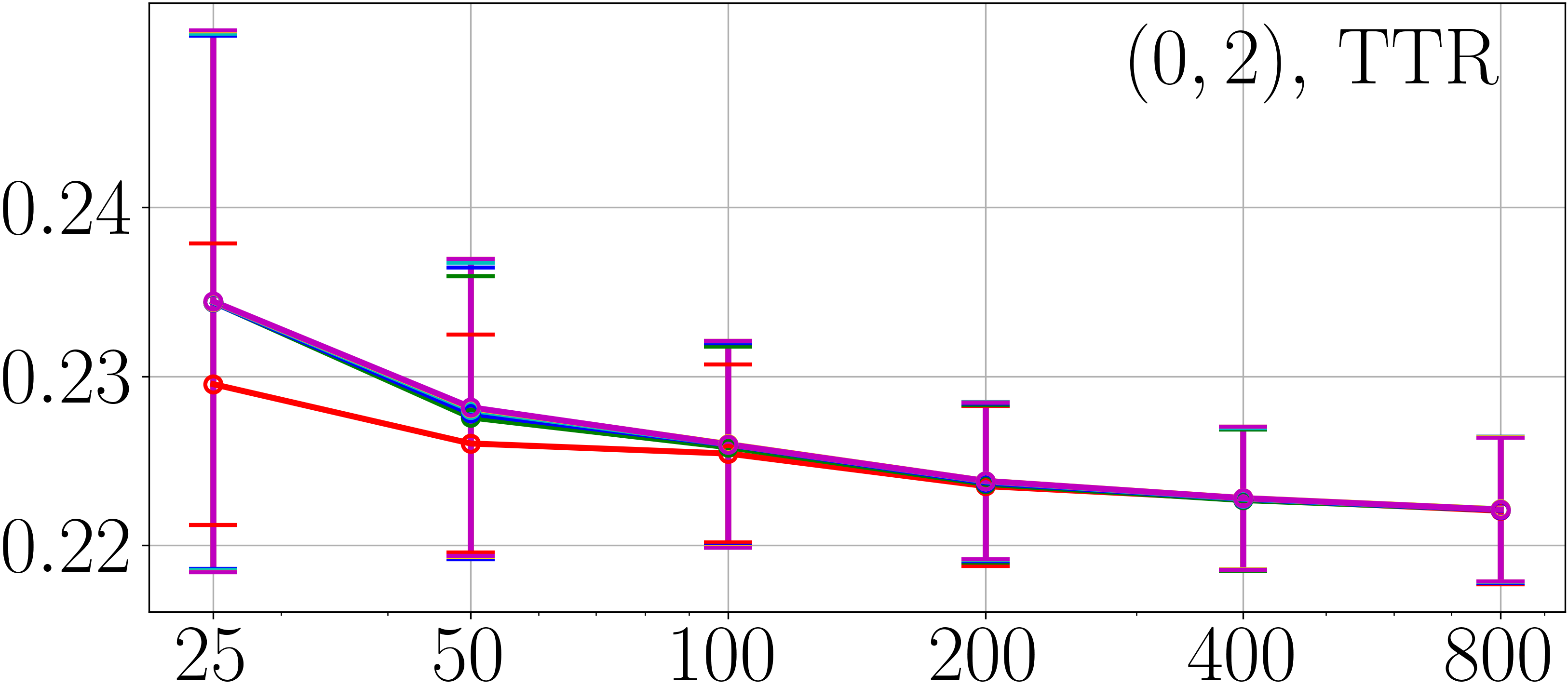}\\
\includegraphics[width=4cm, bb=0 0 780 348]{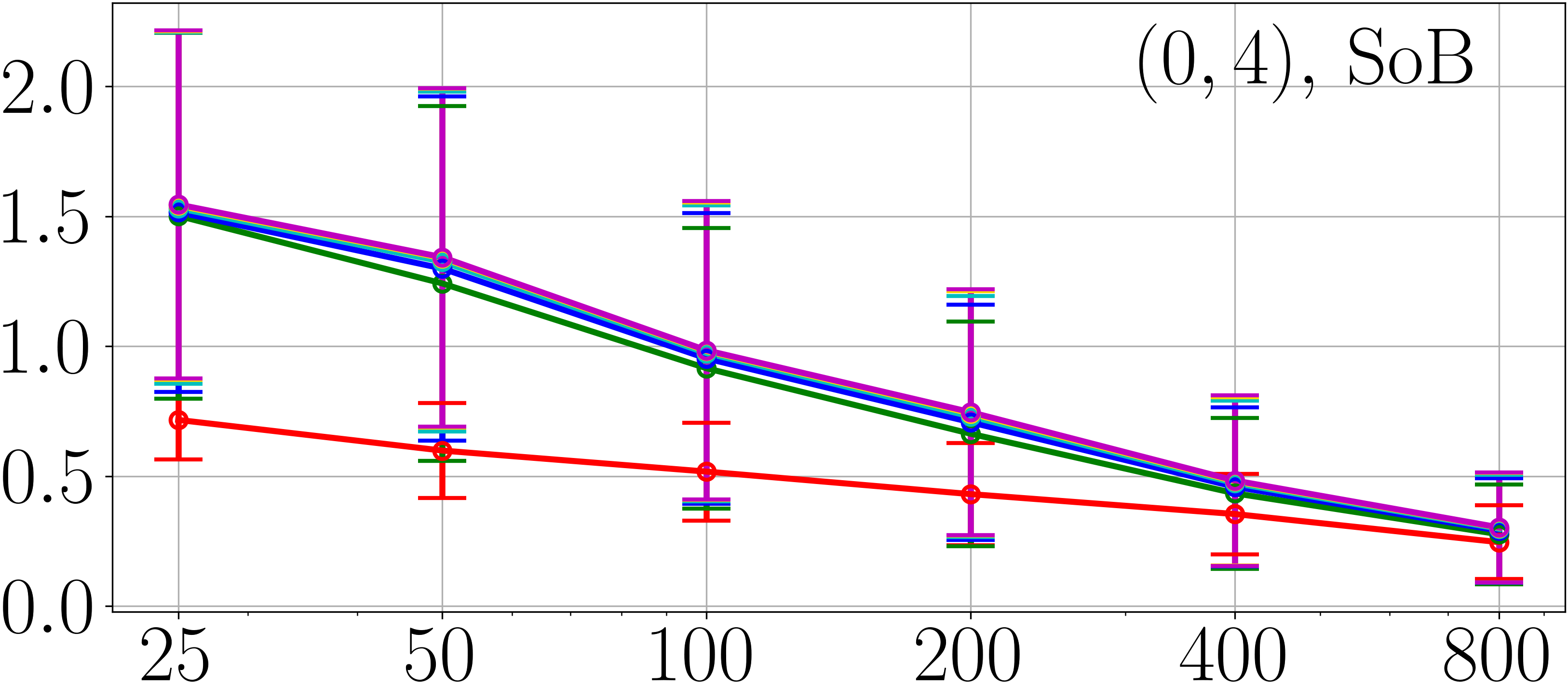}&
\includegraphics[width=4cm, bb=0 0 780 348]{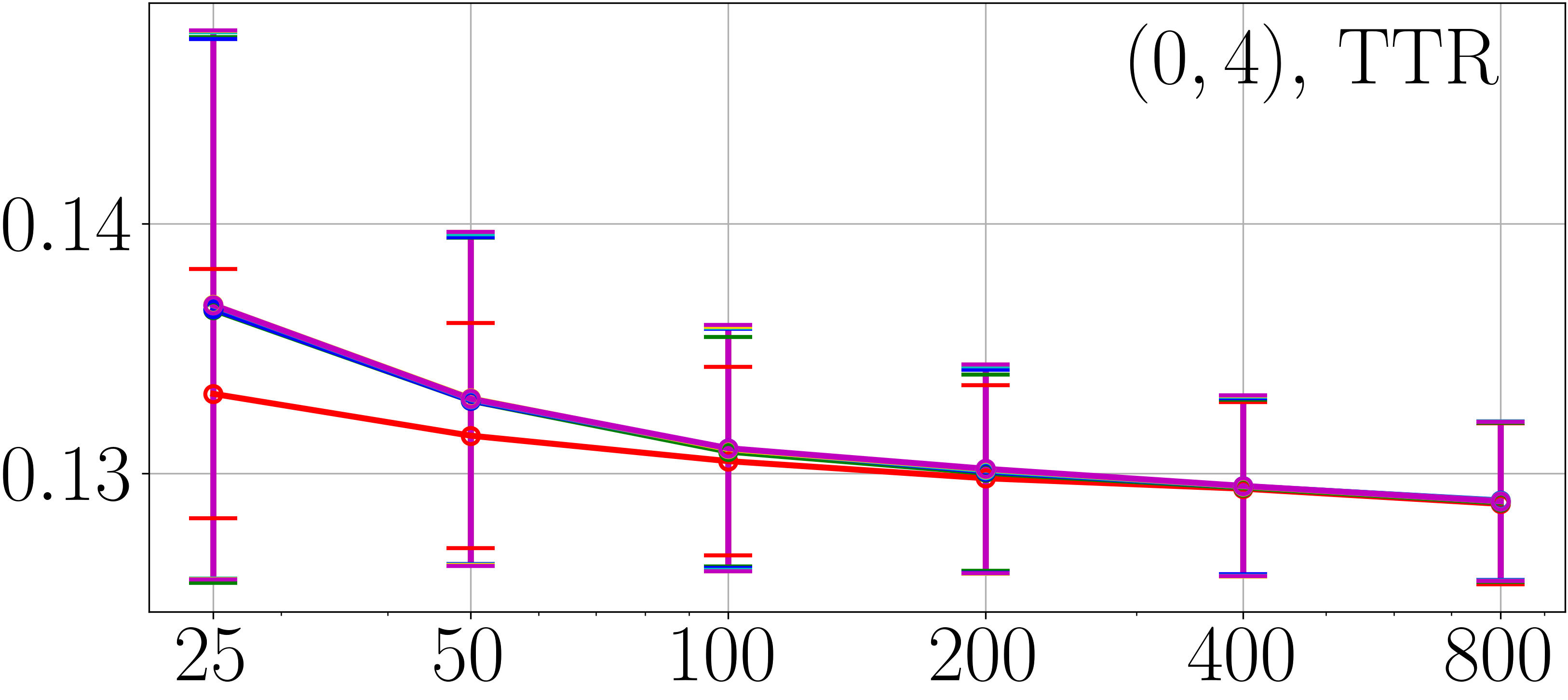}\\
\includegraphics[width=4cm, bb=0 0 780 348]{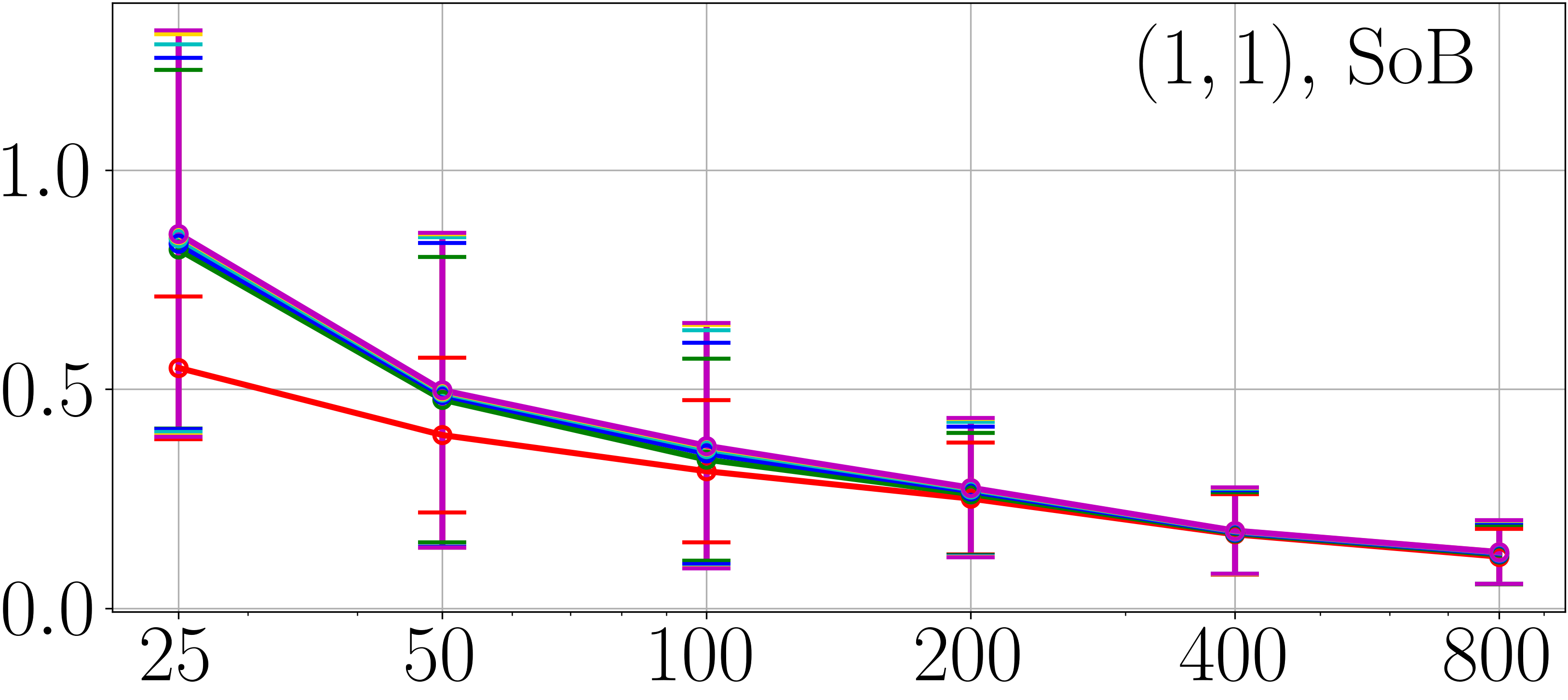}&
\includegraphics[width=4cm, bb=0 0 780 348]{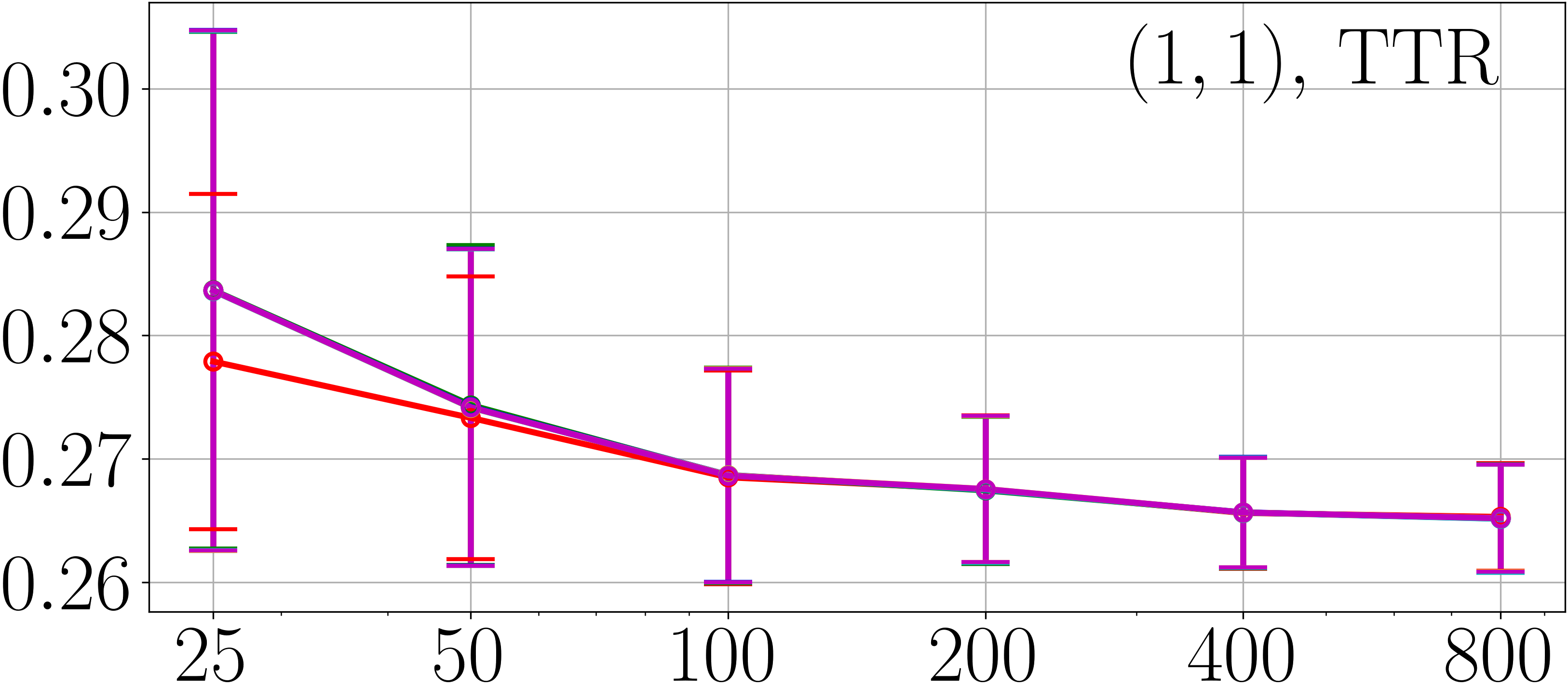}\\
\includegraphics[width=4cm, bb=0 0 780 348]{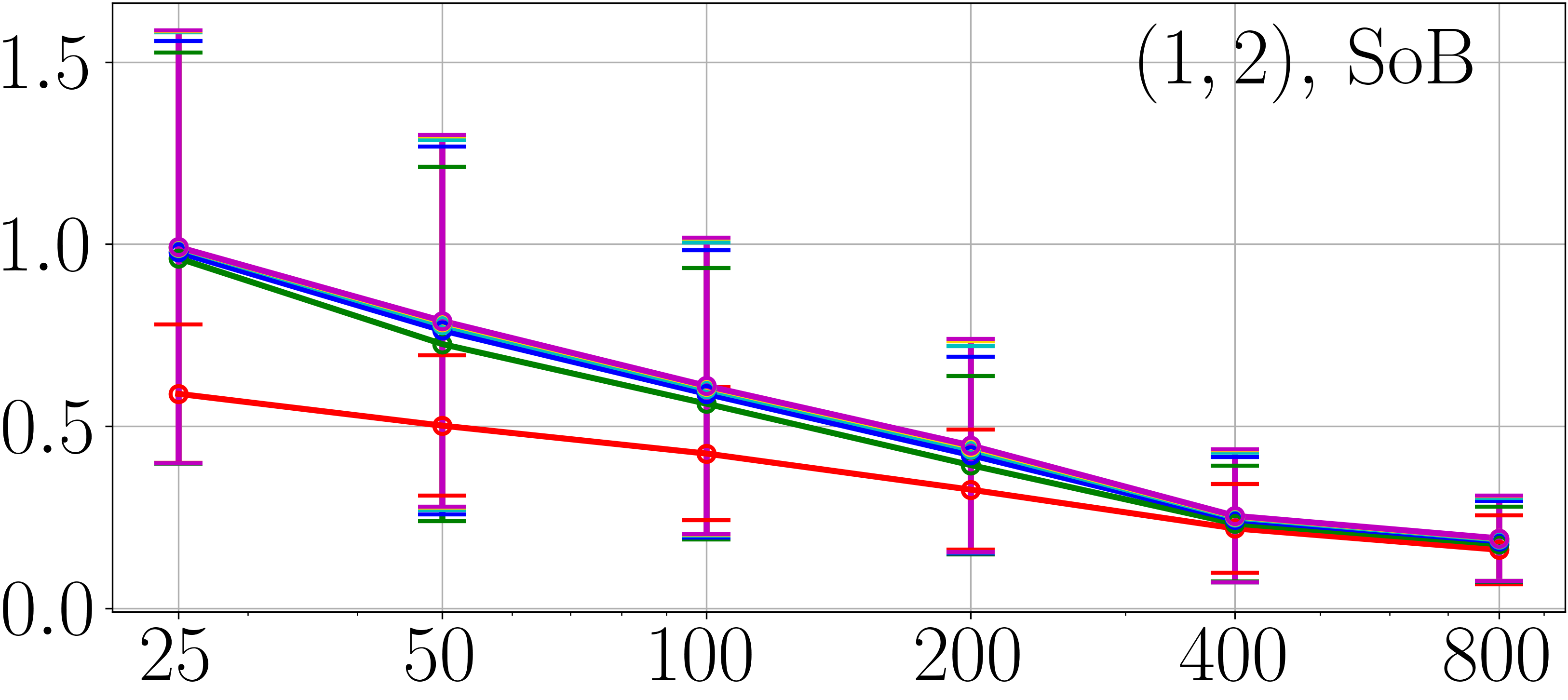}&
\includegraphics[width=4cm, bb=0 0 780 348]{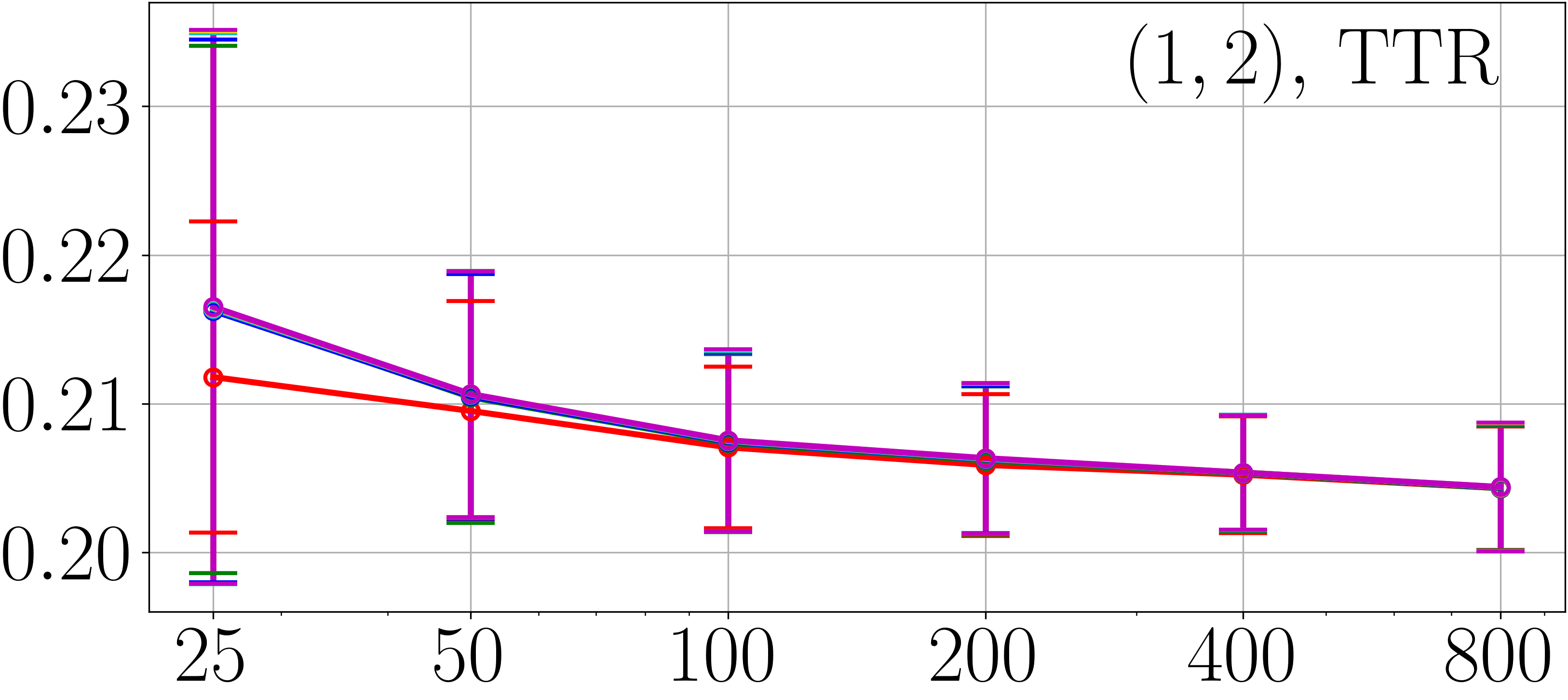}\\
\includegraphics[width=4cm, bb=0 0 780 348]{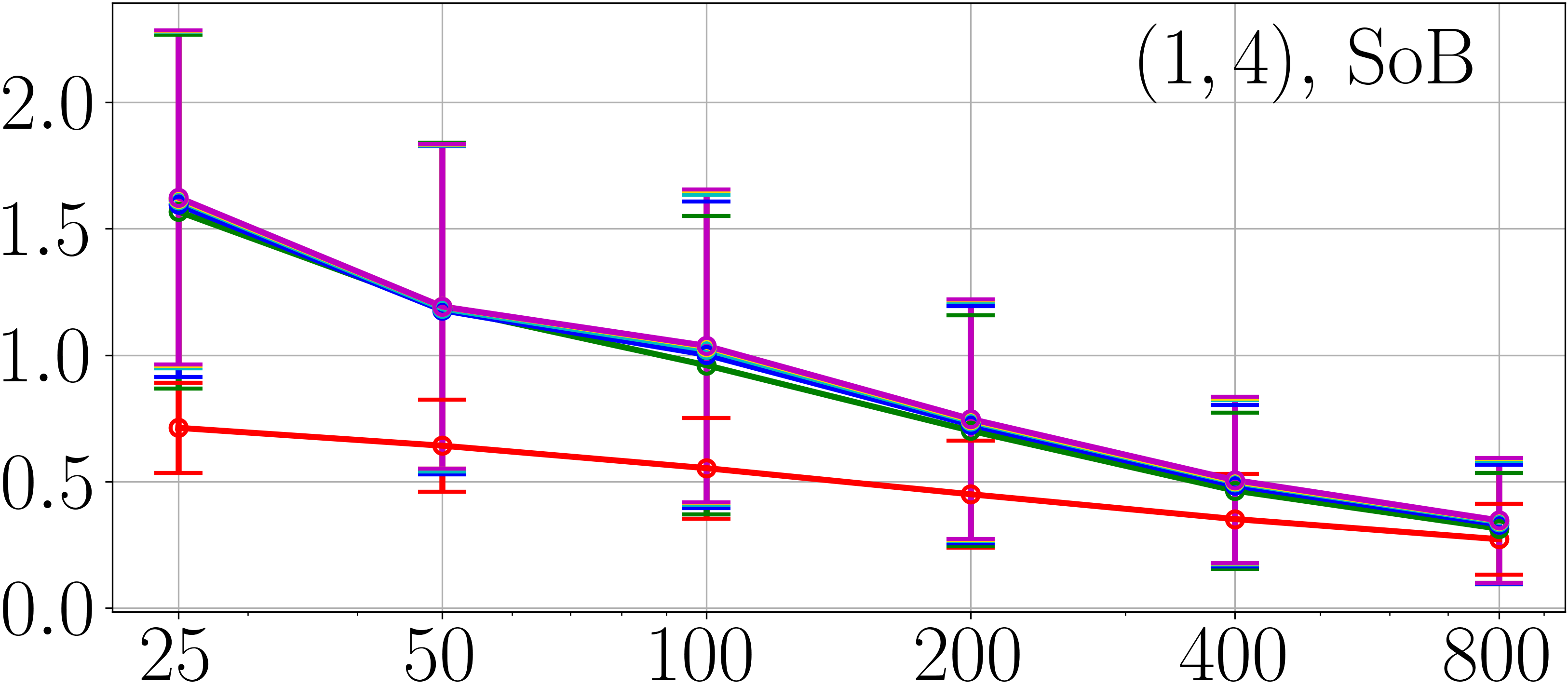}&
\includegraphics[width=4cm, bb=0 0 780 348]{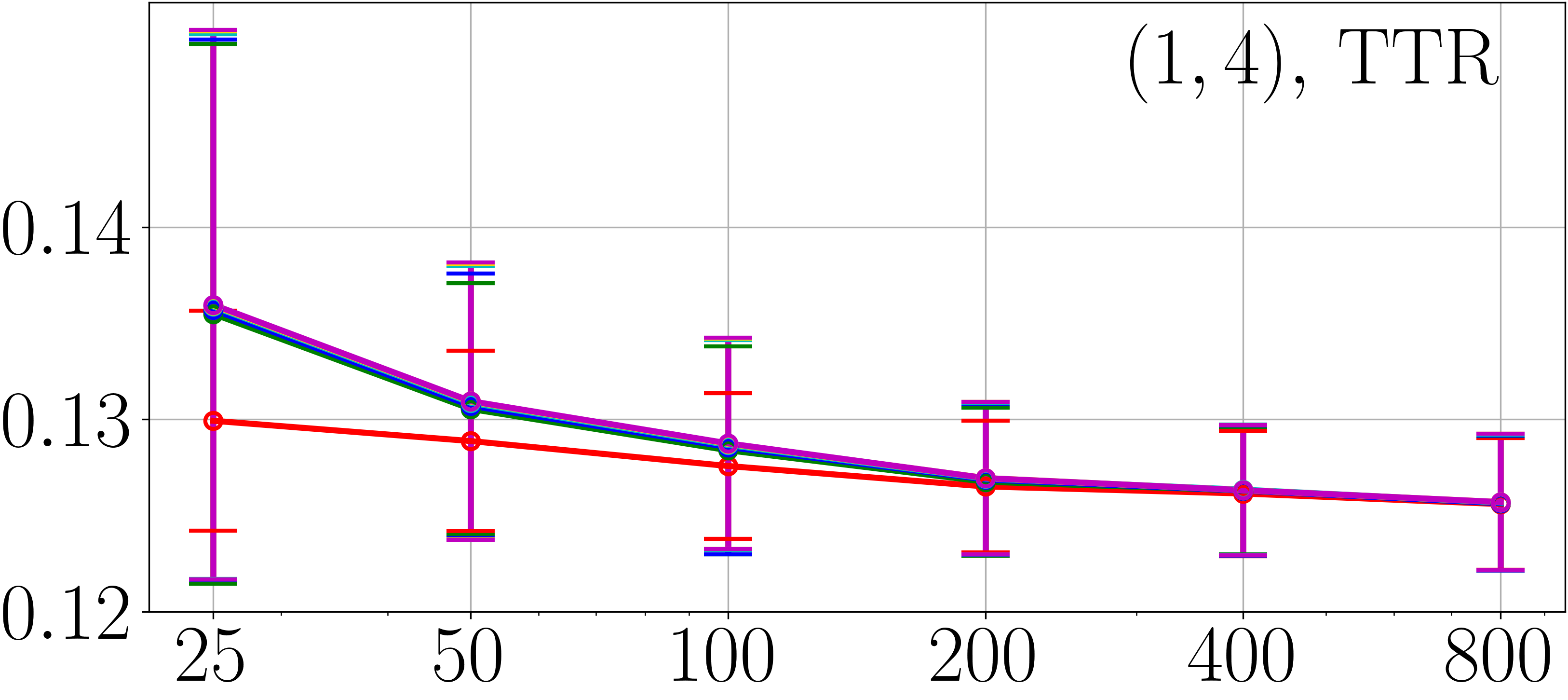}\\
\includegraphics[width=4cm, bb=0 0 780 348]{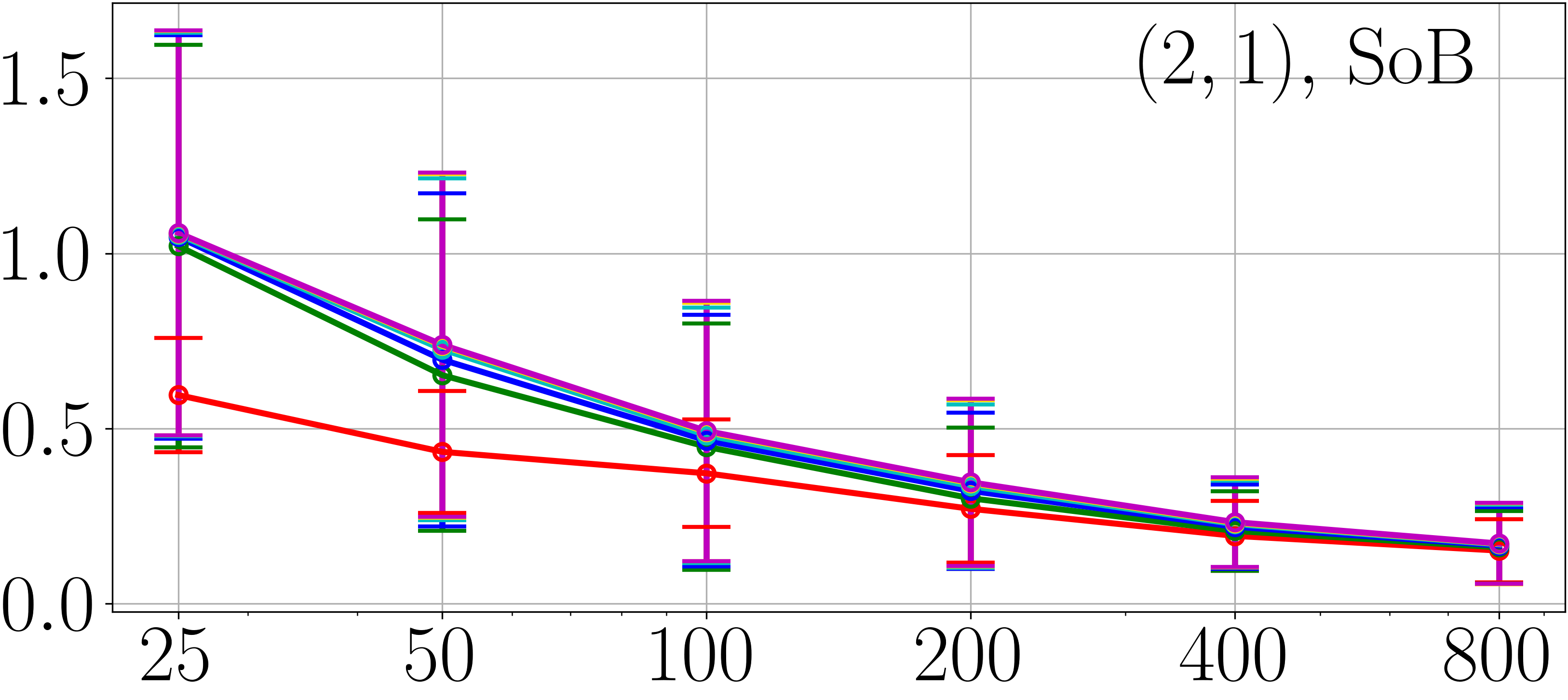}&
\includegraphics[width=4cm, bb=0 0 780 348]{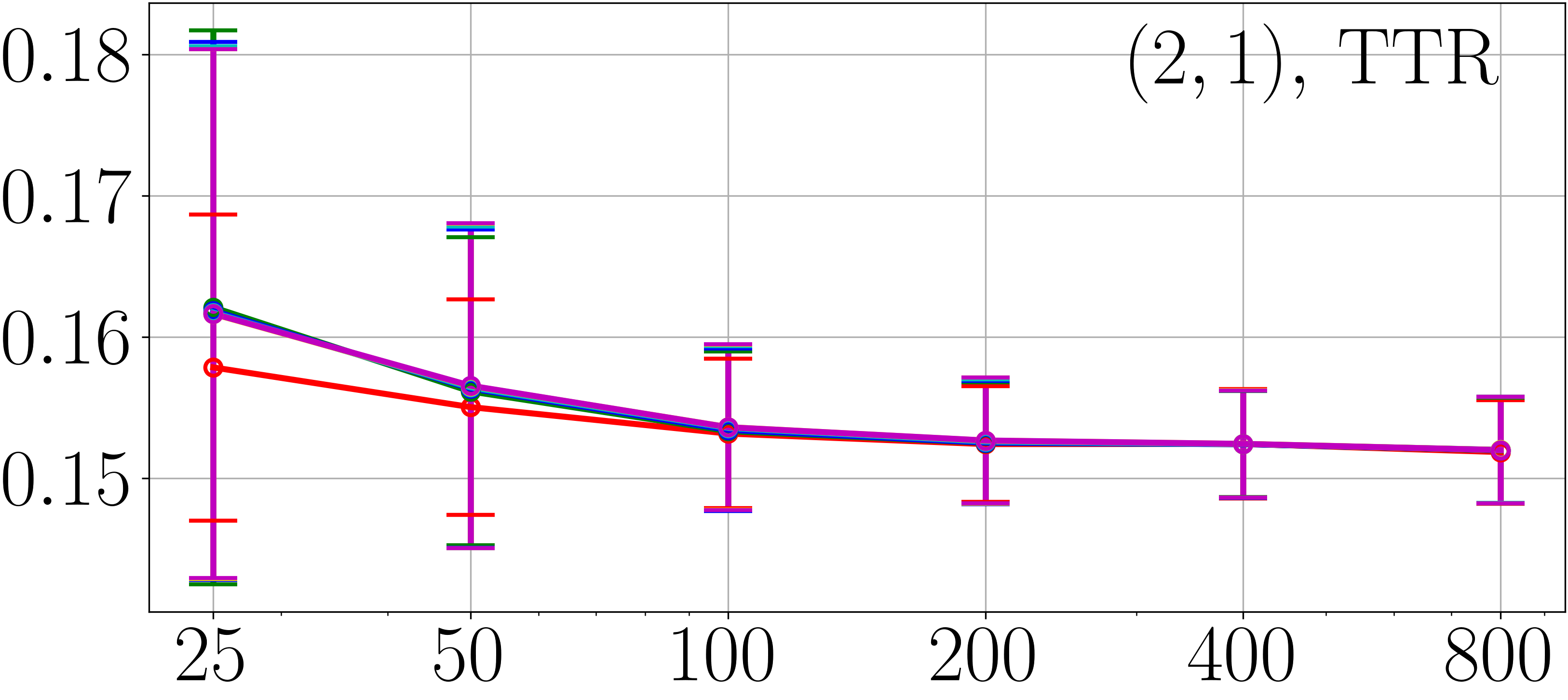}\\
\includegraphics[width=4cm, bb=0 0 780 348]{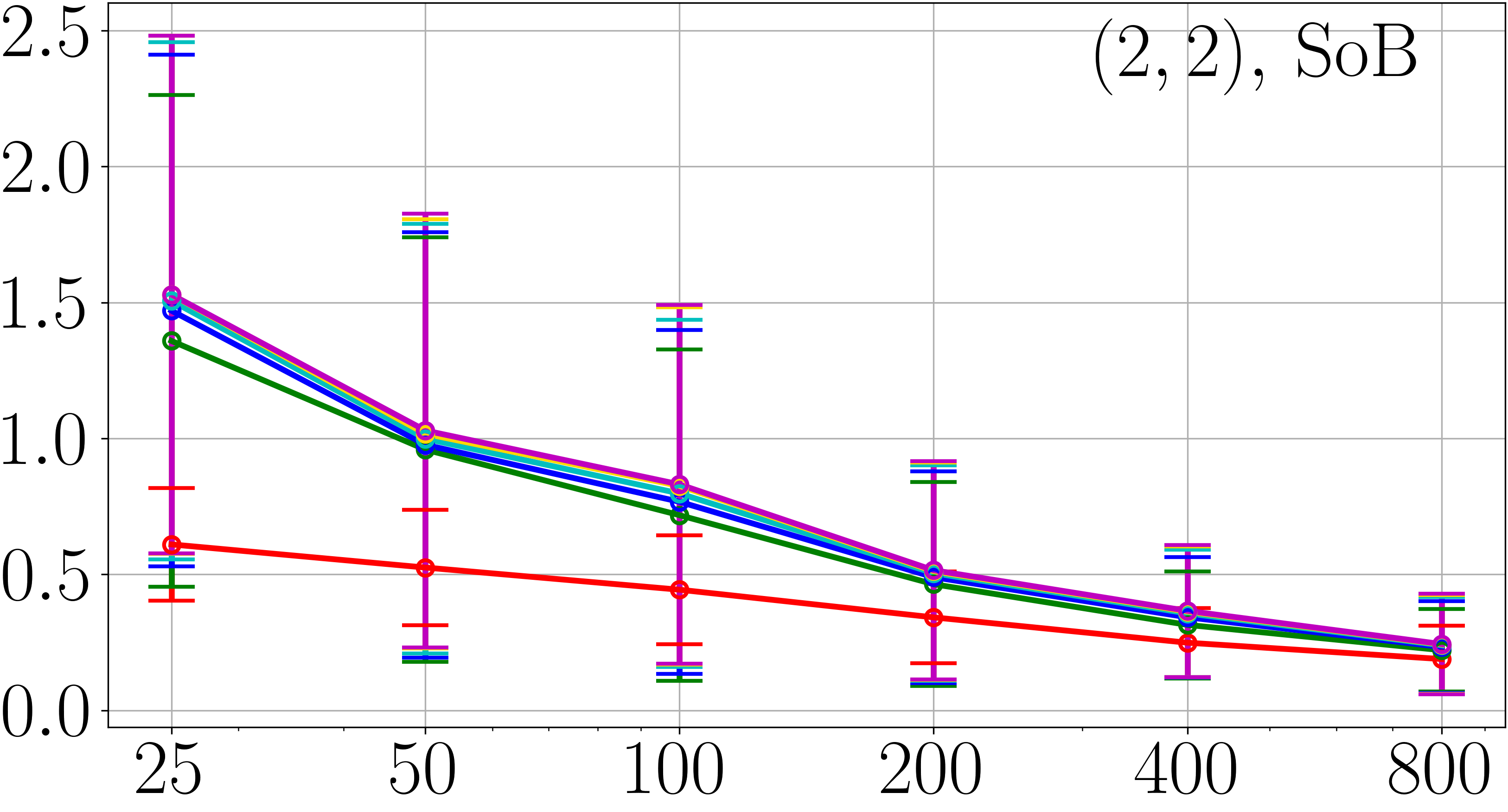}&
\includegraphics[width=4cm, bb=0 0 780 348]{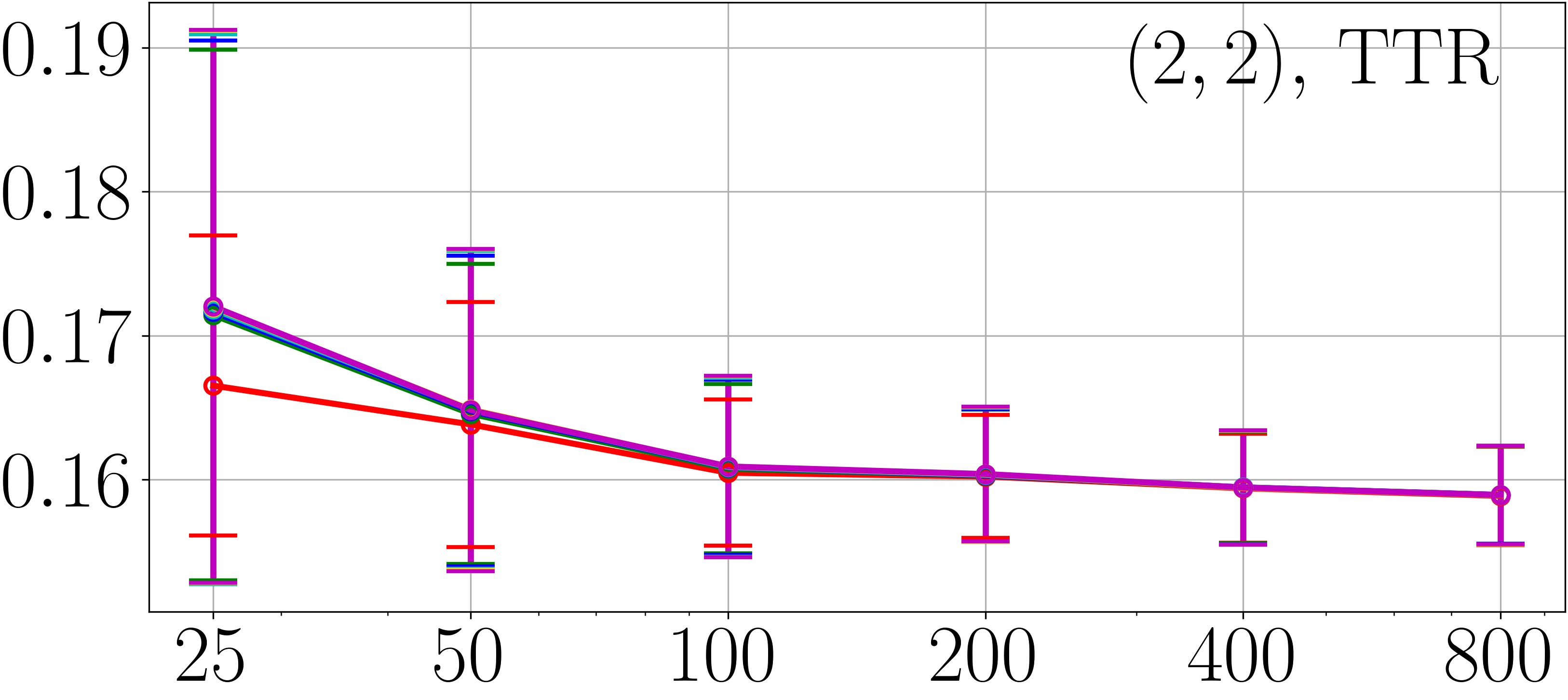}\\
\includegraphics[width=4cm, bb=0 0 780 348]{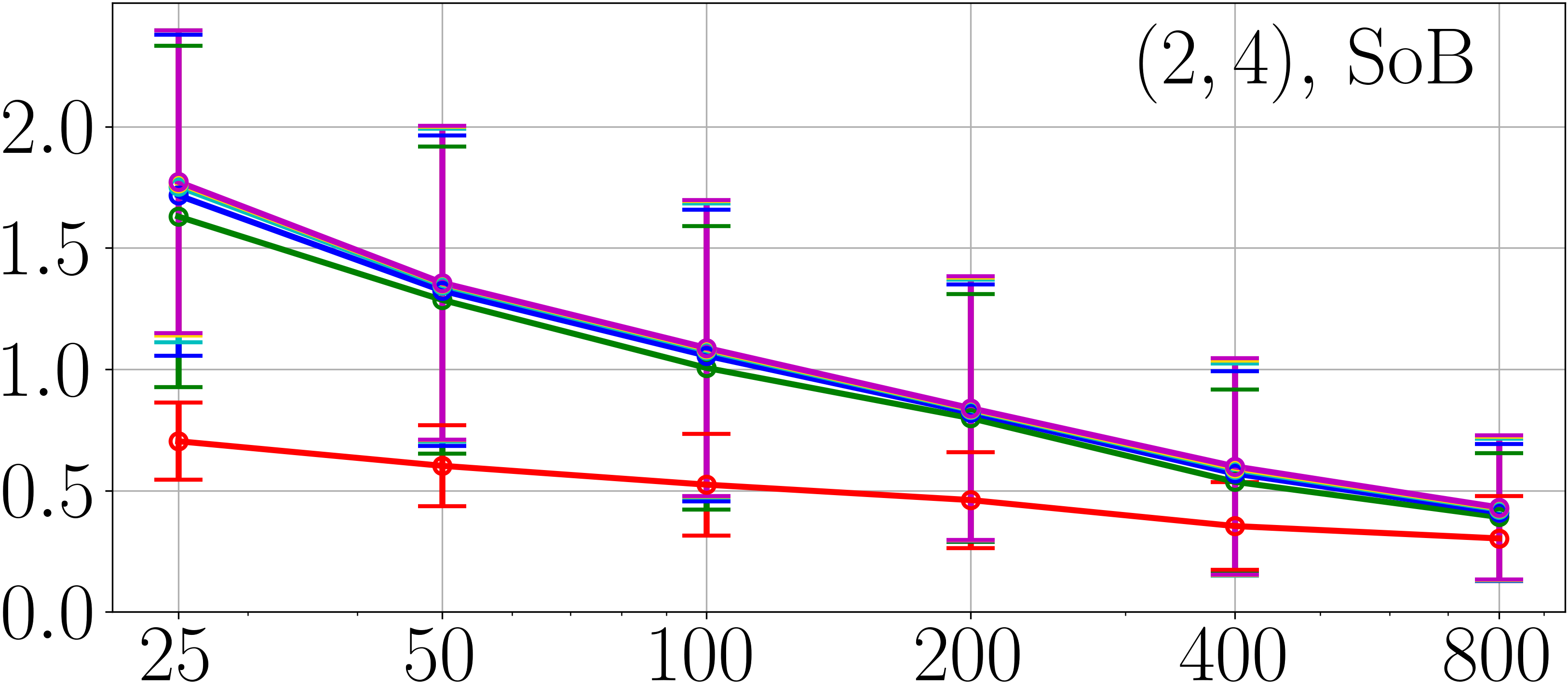}&
\includegraphics[width=4cm, bb=0 0 780 348]{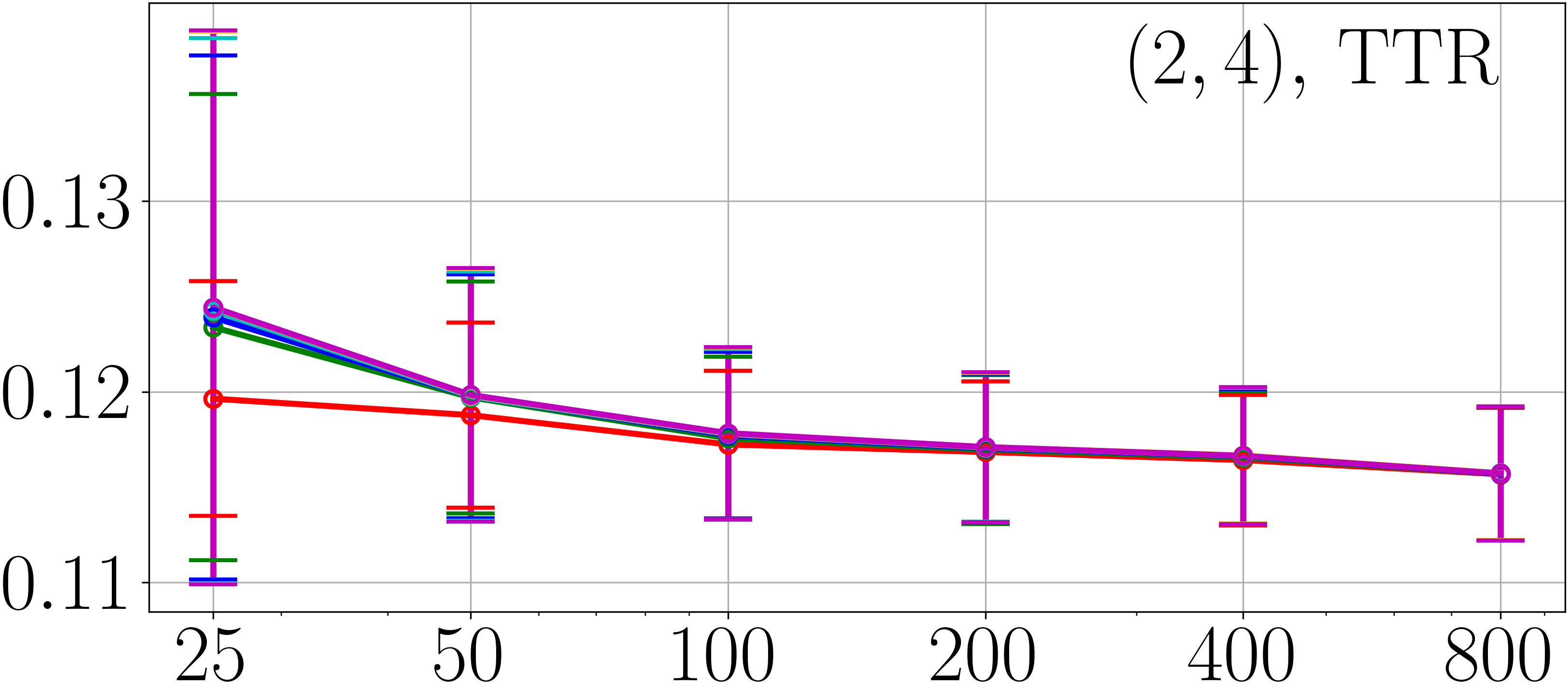}
\end{tabular}
\caption{%
Errorbar plots of SoB and TTR for LR (red), MLSLR with $\alpha=0.2,0.4,0.6,0.8$ (green, blue, cyan, yellow), and LSQLR (magenta) 
versus $n$, where circles and errorbars represent the mean and STD of the errors (for Section \ref{sec:ESimulation}).}
\label{fig:ES}
\end{figure}

%==========%OK
We are also interested in estimation performance of MLSLR and LSQLR 
in a finite-sample situation as well as the large-sample limit. 
The numerical experiment in this section was performed to see it.%
\footnote{%
Our program codes for experiments in Sections \ref{sec:ESimulation}, 
\ref{sec:RSimulation}, and \ref{sec:Experiment} and Appendix~\ref{sec:OutAlpha} 
are in \url{https://github.com/yamasakiryoya/LSIS}.}
%single-blind

%==========%OK
We consider the task with the zero-one task loss $\ell_\zo$.
With the correctly-specified model \eqref{eq:nominal},
we calculated $\hat{\bbeta}_n$ with a training sample of size $n=25, 50, \ldots, 800$
and evaluated the size of bias (SoB) $\|\hat{\bbeta}_n-\tilde{\bbeta}\|$ and 
test task risk (TTR) $\frac{1}{m}\sum_{i=1}^m\ell_\zo(h_{\ell_\zo}(\hat{\bbeta}_n^\top\bx_i), y_i)$
(i.e., misclassification rate) with a test sample of size $m=10^4$,
for LR, MLSLR with $\alpha=0.2, 0.4, 0.6, 0.8$, and LSQLR.
Note that it makes no sense to compare the test surrogate risk (TSR) 
$\frac{1}{m}\sum_{i=1}^m\phi(\hat{\bbeta}_n^\top\bx_i, y_i)$ 
for different surrogate losses.
We trained the model parameter using (batch-version) Adam with 
learning rate that is multiplied by $10^{-1/2}$ every 50 epochs from 0.01 for 150 epochs,
from the initial point $\tilde{\bbeta}$ for LR
or together with multi-start strategy with 30 initial points scattered from $\tilde{\bbeta}$ for MLSLR and LSQLR.
Figure \ref{fig:ES} shows some of the results 
on the mean and standard deviation (STD) of the errors over 100 randomized trials.

%==========%OK
Consequently, similarly to the ARE results in Table \ref{tab:ARE}, 
it was also experimentally confirmed that
parameter estimation performance with the correctly-specified model 
decreased along with increase of the smoothing level $\alpha$
even when the sample size is small (see mean of SoB in Figure \ref{fig:ES}).
Along with this behavior, TTR for MLSLR and LSQLR also deteriorated.

%==================================================%OK
\subsection{Robustness against Model Misspecification}
\label{sec:Robustness}
%==================================================%OK
\subsubsection{Similarity to Existing Robust LRs}
\label{sec:Similarity}
%==========%OK
Use of a loss function that leads to low efficiency with correctly-specified models 
may be motivated by expectation for robustness against model misspecification, 
as also suggested by \cite{zhang2019theoretically}:
One may want to make the estimation result stable even 
when the data distribution deviates from the specified model.
Various studies have confirmed that LR lacks robustness against
model misspecification \cite{pregibon1981logistic},
and then proposed robust LRs via the idea of $\rho$-transformation 
(see below) of the loss function $\varphi_\LR$
\cite{pregibon1982resistant, bianco1996robust, croux2003implementing};
See also the monograph \cite[Chapter 7.2]{maronna2019robust}.

%==========%OK
\cite{pregibon1982resistant} proposed to use a $\rho$-transformed loss
\begin{align}
\label{eq:rho-tra}
	\varphi(v)=\rho(\varphi_\LR(v)),
\end{align}
where $\rho:[0,\infty)\to\bbR$ is a function with sublinear growth 
(e.g., $\rho_{{\rm P},c}(v)=v$ (if $v\le c$), $=2(cv)^{1/2}-c$ (if $v>c$) with some constant $c>0$,
inspired by a Huber-type loss).
However, his robust estimator based on $\rho_{{\rm P},c}$ 
is not consistent even with correctly-specified models, 
and it has been pointed out that it is not sufficiently robust against outliers \cite{copas1988binary}.
Thus, seeking a robust LR with consistency (having a guarantee like Corollary \ref{cor:ProbPred}),
\cite{bianco1996robust}%
\footnote{%
They originally defined their estimator by the loss function
\begin{align}
\label{eq:BY-origin}
	\varphi(v)=\zeta_{1,c}(\varphi_\LR(v))+\eta\bigl(\tfrac{1}{1+e^{-v}}\bigr)+\eta\bigl(\tfrac{1}{1+e^{v}}\bigr),
\end{align}
where $\zeta_{1,c}(v)=v-v^2/(2c)$ (if $v\le c$), $=c/2$ (if $v>c$) with some constant $c>0$,
and where $\eta(v)=\int_0^v \zeta_{1,c}'(-\ln t)\,dt$.
The form \eqref{eq:BY-review}--\eqref{eq:eta-AF} in the body is that our study derived 
so that it follows the unified formulation through the $\rho$-transformation \eqref{eq:rho-tra}.}
found the following $\rho$-transformation:
\begin{align}
\label{eq:BY-review}
	\rho_{\BY,c}(v)=\zeta_{1,c}(v)+\zeta_{2,c}(e^{-v})+\zeta_{2,c}(1-e^{-v})
\end{align}
that is non-decreasing in $v\ge0$ and bounded,
where
\begin{align}
\label{eq:eta-AF}
	\begin{split}
	&\zeta_{2,c}(v)=
	\bigl[v-e^{-c}+\tfrac{1}{c}\bigl\{e^{-c}(c+1)\\
	&\hphantom{\zeta_{2,c}(v)=\bigl[v-e^{-c}+\tfrac{1}{c}\,}
	+v(\ln v-1)\bigr\}\bigr]\mathbbm{1}(v\ge e^{-c}).
	\end{split}
\end{align}
Also, \cite{croux2003implementing} used,
besides a data-weighting scheme similar to \cite{carroll1993robustness}, 
$\bar{\zeta}_{1,c}(v)=ve^{-\sqrt{c}}$ (if $v\le c$),
$=-2e^{-\sqrt{v}}(1+\sqrt{v})+e^{-\sqrt{c}}(2(1+\sqrt{c})+c)$ (if $v>c$)
instead of $\zeta_{1,c}$ in the formulation \eqref{eq:BY-origin},
yielding an increasing and bounded $\rho$-transformation $\rho_{\CH,c}$.%

%==========%OK
\begin{figure}[t]
\centering
\begin{tabular}{ccc}\hskip-5pt
\includegraphics[width=2.85cm, bb=0 0 780 420]{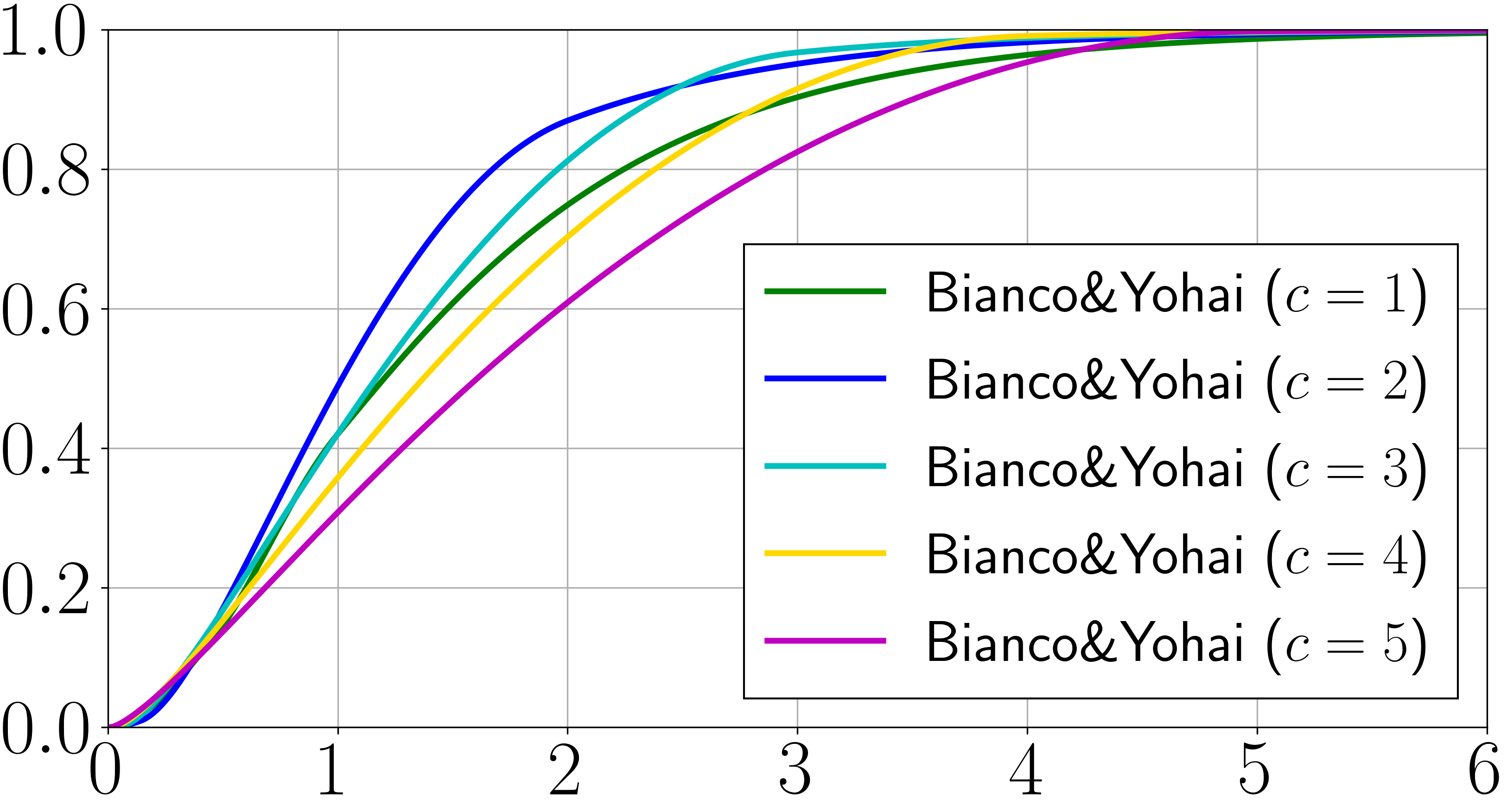}&\hskip-10pt
\includegraphics[width=2.85cm, bb=0 0 780 420]{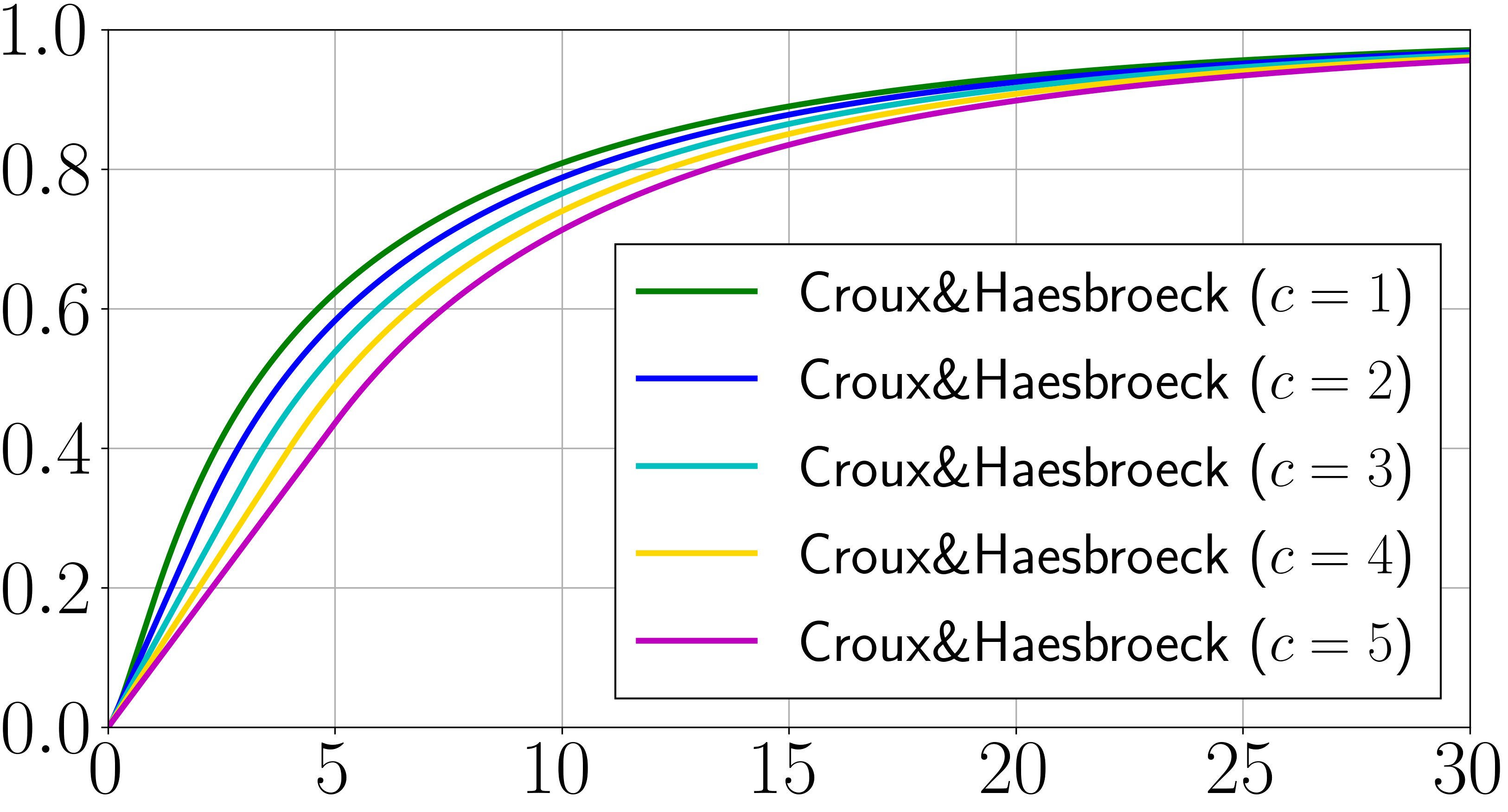}&\hskip-10pt
\includegraphics[width=2.85cm, bb=0 0 780 420]{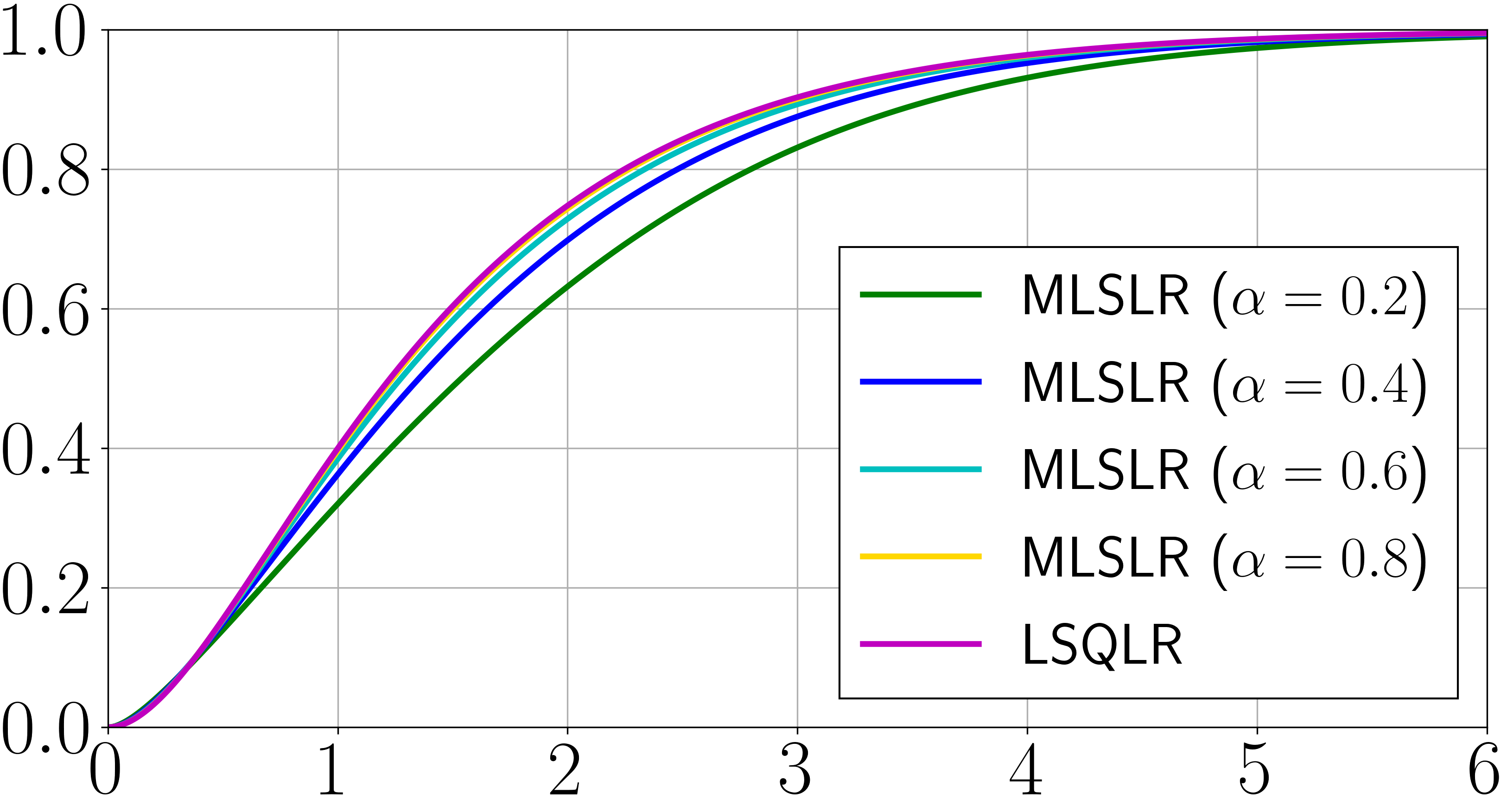}
\end{tabular}
\caption{%
Rescaled $\rho$-transformation:
$\{\rho(v)-\rho(0)\}/\{\lim_{v'\to\infty}\rho(v')-\rho(0)\}$ versus $v\ge0$,
for $\rho=\rho_{\BY,c}, \rho_{\CH,c}, \rho_{\MLS,\alpha}$, and $\rho_\LSQ$.}
\label{fig:rho}
\end{figure}

%==========%OK
Our observation here is that MLSLR and LSQLR can be regarded as
instances of $\rho$-transformation-based robust LR, with
\begin{align}
	\begin{split}
	\rho_{\MLS,\alpha}(v)
	=&\,-\bigl(1-\tfrac{\alpha}{2}\bigr)\ln\bigl((1-\alpha)e^{-v}+\tfrac{\alpha}{2}\bigr)\\
	&\,-\tfrac{\alpha}{2}\ln\bigl((1-\alpha)(1-e^{-v})+\tfrac{\alpha}{2}\bigr),\\
	\rho_\LSQ(v)
	=&\,\tfrac{1}{2}(1-e^{-v})^2.
	\end{split}
\end{align}
Both of these transformations are increasing (in $v\ge0$) and bounded.
This property and Figure \ref{fig:rho} show close similarity of 
MLSLR, LSQLR, and the existing robust LRs by \cite{bianco1996robust, croux2003implementing},
which will suggest the robustness of MLSLR and LSQLR.

%==================================================%OK
\subsubsection{Simulation Experiment}
\label{sec:RSimulation}
%==========%OK
In robust statistics, 
it is common to study the performance of an estimator under 
(point-mass) data contaminations, which cause model misspecification.
The discussion in this section follows that approach.

%==========%OK
We introduce the $\bbR^d$-valued functional $\bT$, defined on 
the space of probability distribution functions $F$ of random variables $(\bX,Y)\in\bbR^d\times[K]$,
that represents the procedure of the surrogate risk minimization:
\begin{align}
	\bT(F)\coloneq\underset{\bbeta\in\bbR^d}{\argmin}\;
	\bbE_{(\bX,Y)\sim F}\bigl[\phi(\bbeta^\top\bX,Y)\bigr],
\end{align}
where the expectation is taken with respect to $(\bX,Y)\sim F$.
Now, we are interested in behaviors of
the estimate $\bbeta_{\epsilon, (\bx_c, y_c)}=\bT(F_{\epsilon, (\bx_c, y_c)})$ 
for the contaminated distribution
\begin{align}
\label{eq:ContDis}
	F_{\epsilon, (\bx_c, y_c)}=(1-\epsilon)\tilde{F}+\epsilon\delta_{(\bX,Y)}(\bx_c,y_c),
\end{align}
where $\tilde{F}$ is a nominal distribution of $(\bX,Y)$ such that 
$p_1(\bx)=\frac{1}{1+e^{-\tilde{\bbeta}^\top\bx}}$ implying $\tilde{\bbeta}=\bT(\tilde{F})$,
and $\epsilon\ll1$ is a small ratio.

\begin{figure}[t]
\centering
\begin{tabular}{cc}
\includegraphics[width=4cm, bb=0 0 780 348]{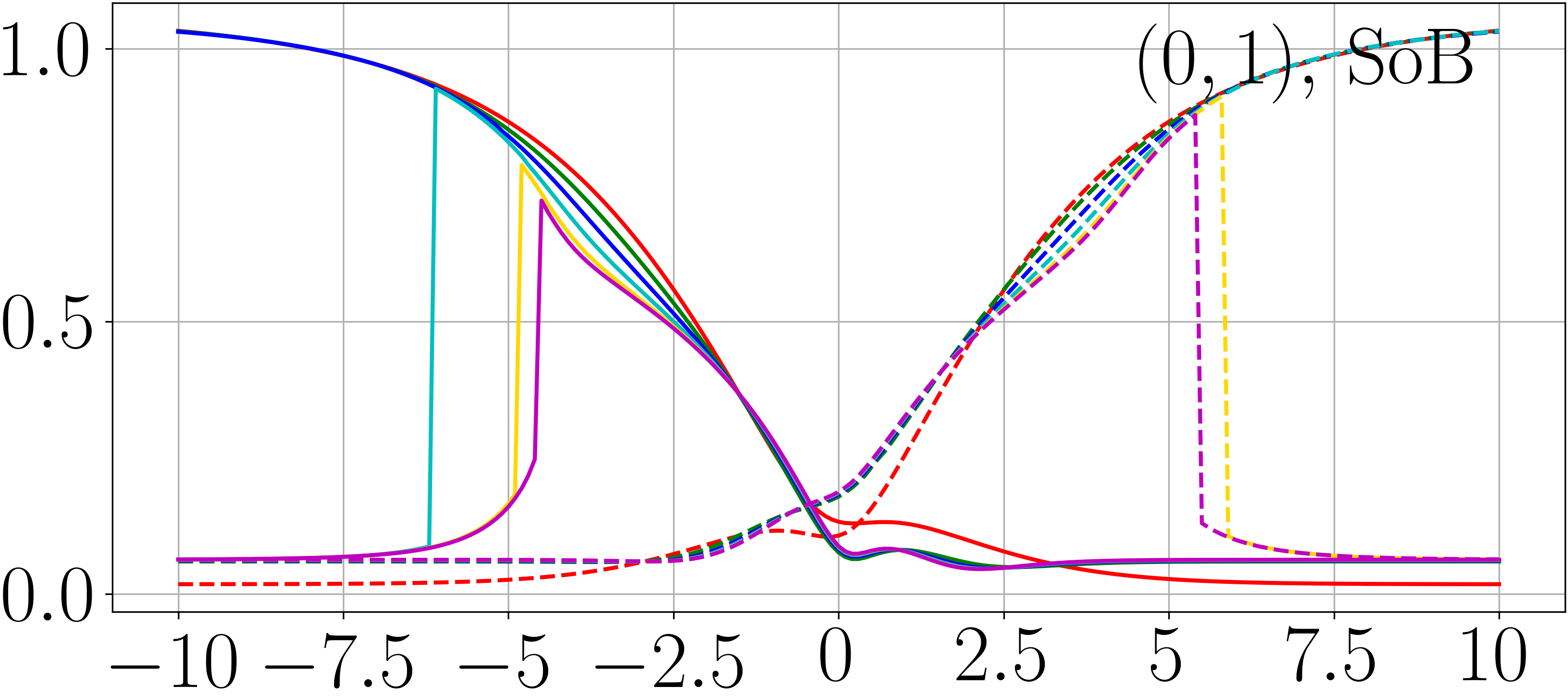}&
\includegraphics[width=4cm, bb=0 0 780 348]{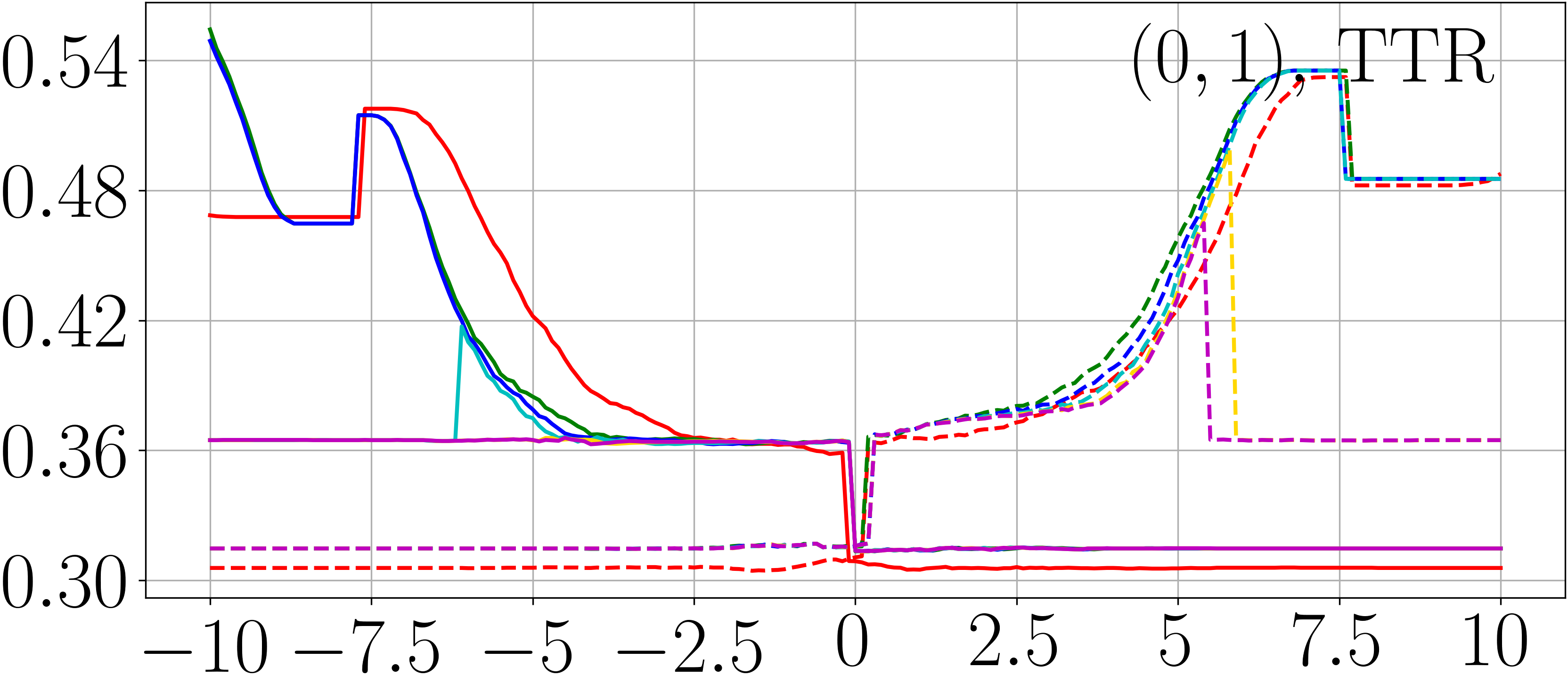}\\
\includegraphics[width=4cm, bb=0 0 780 348]{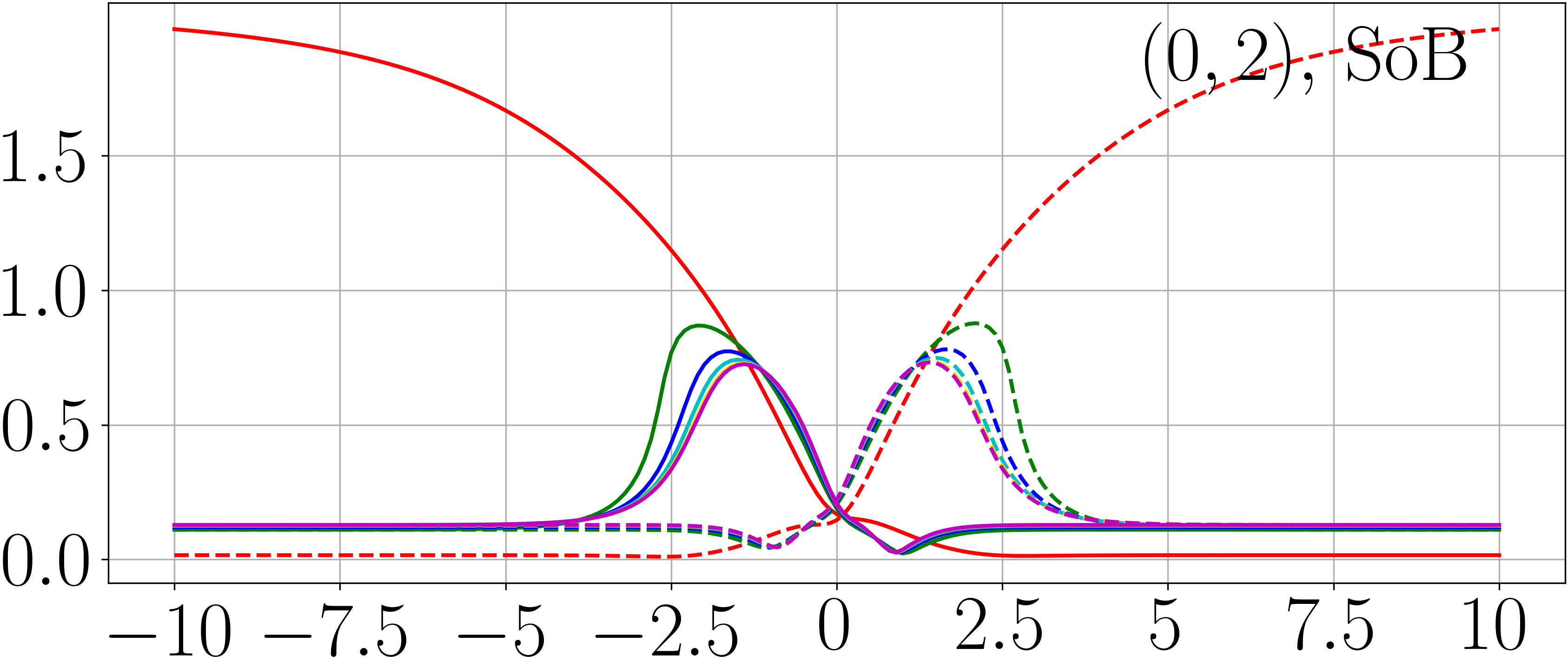}&
\includegraphics[width=4cm, bb=0 0 780 348]{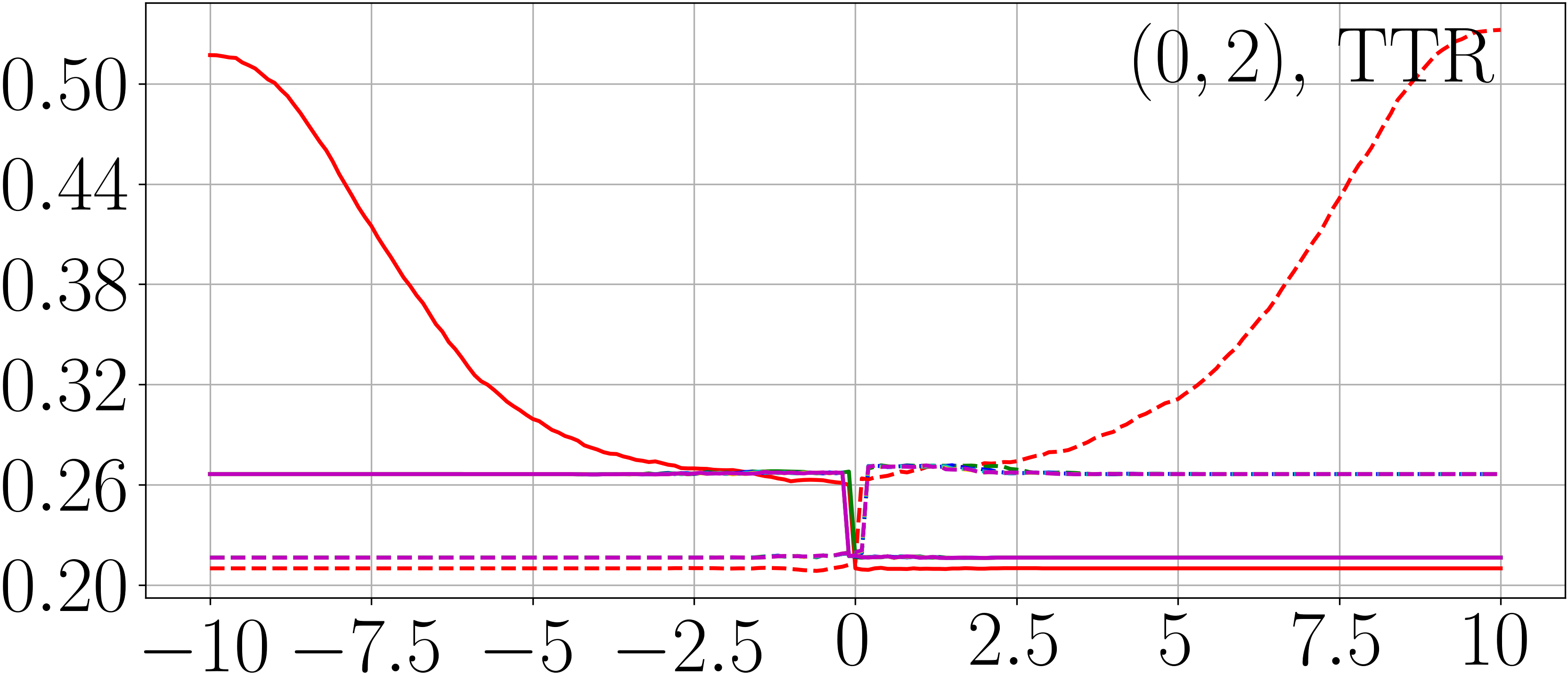}\\
\includegraphics[width=4cm, bb=0 0 780 348]{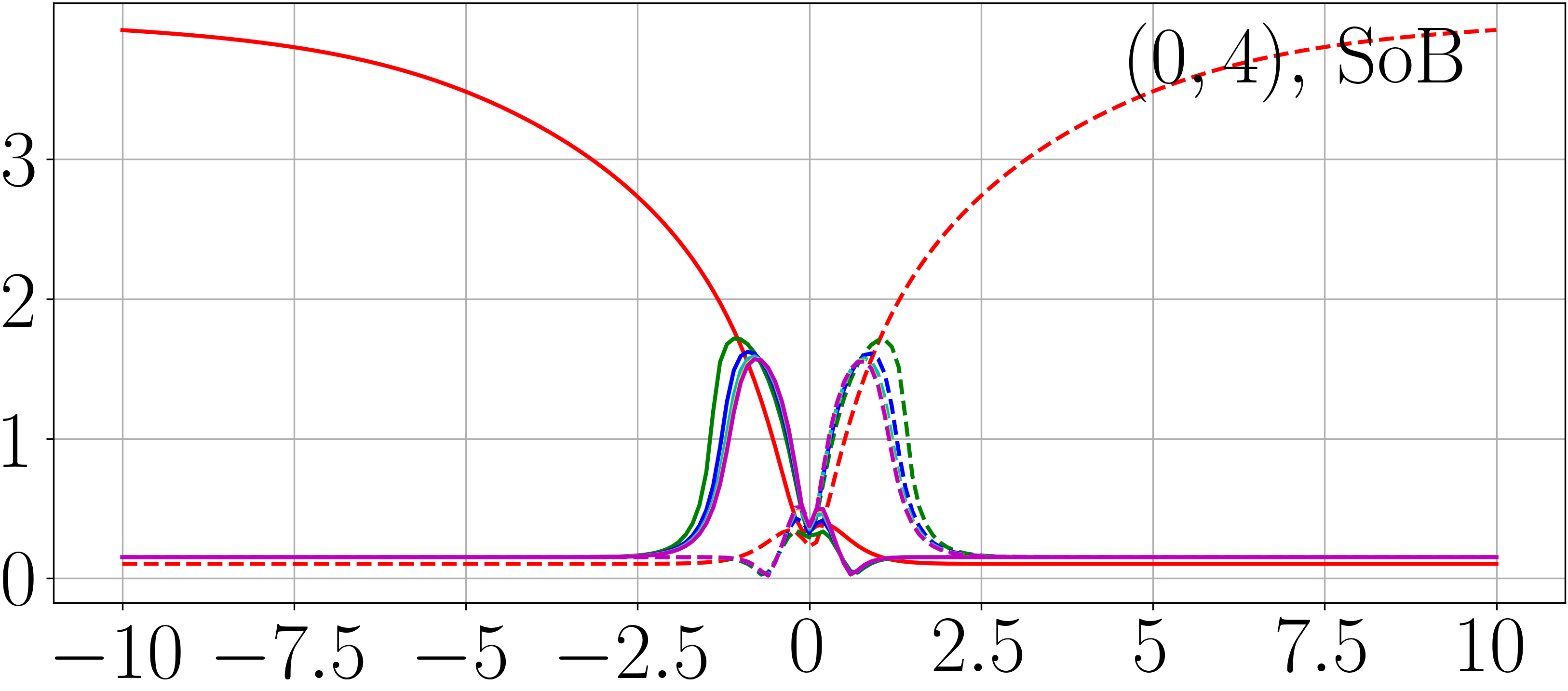}&
\includegraphics[width=4cm, bb=0 0 780 348]{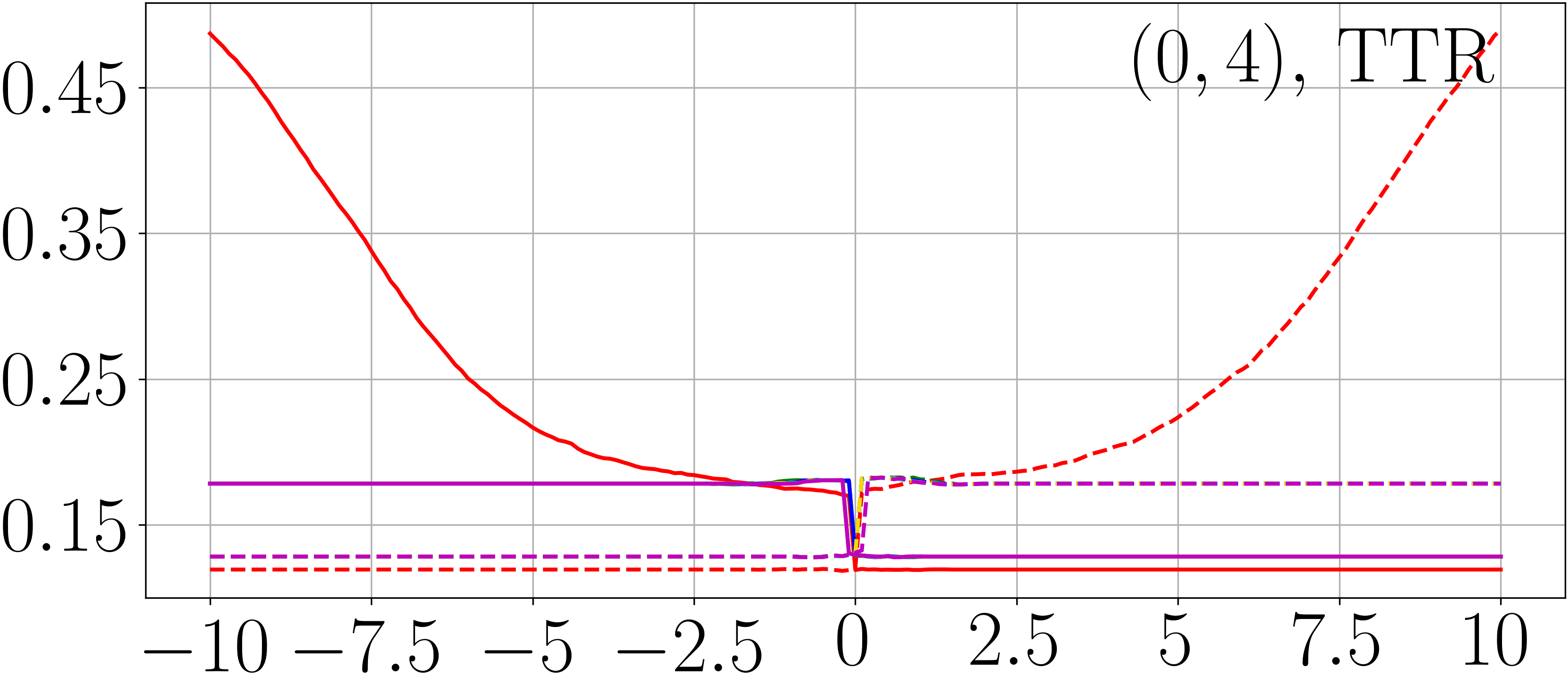}\\
\includegraphics[width=4cm, bb=0 0 780 348]{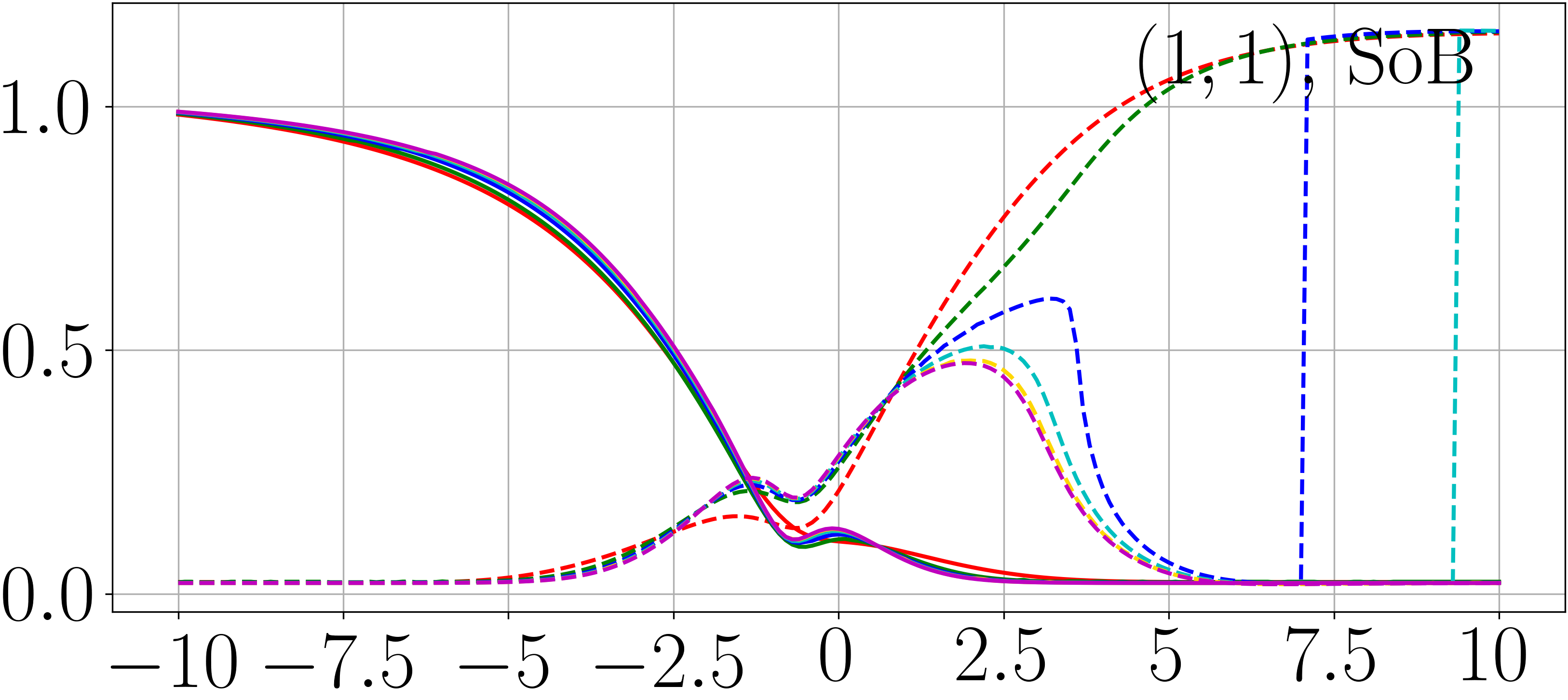}&
\includegraphics[width=4cm, bb=0 0 780 348]{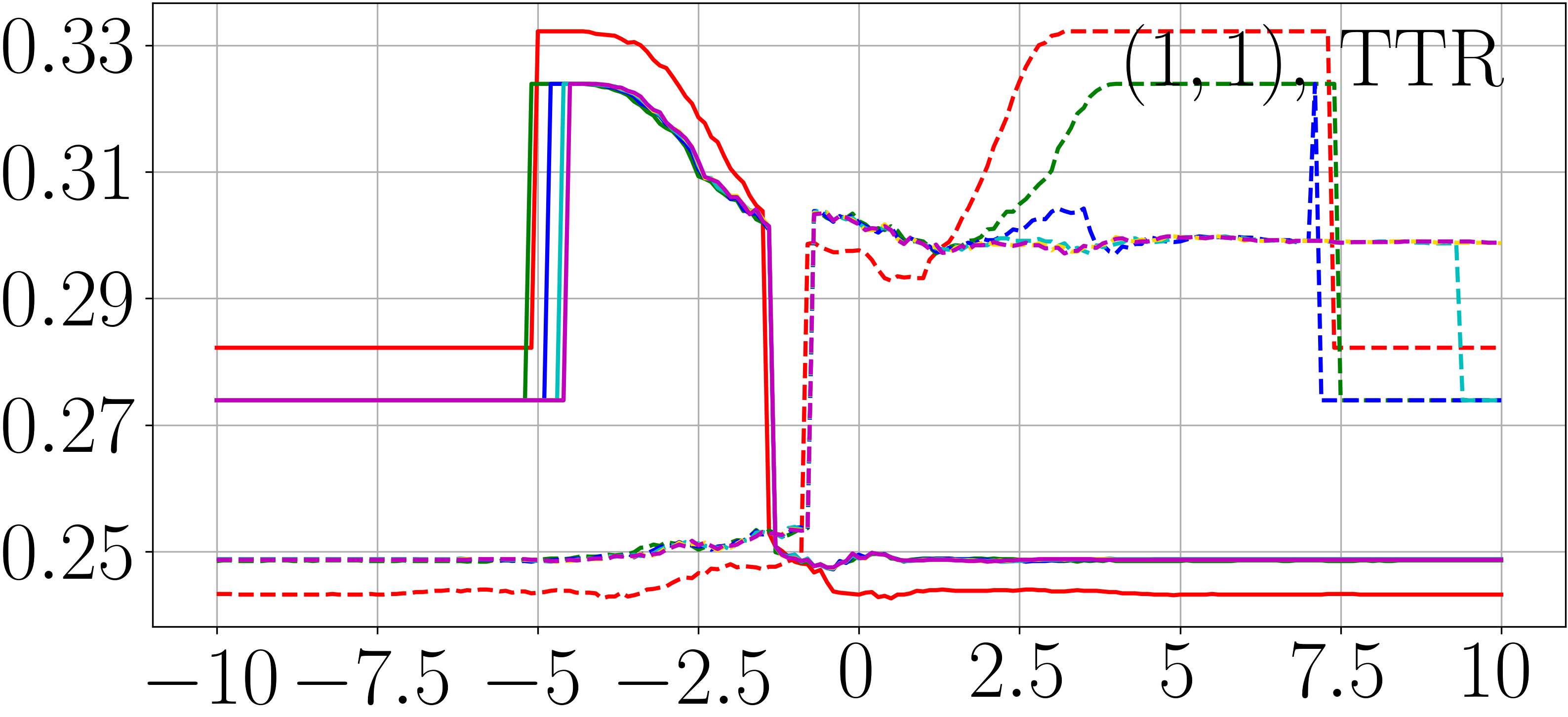}\\
\includegraphics[width=4cm, bb=0 0 780 348]{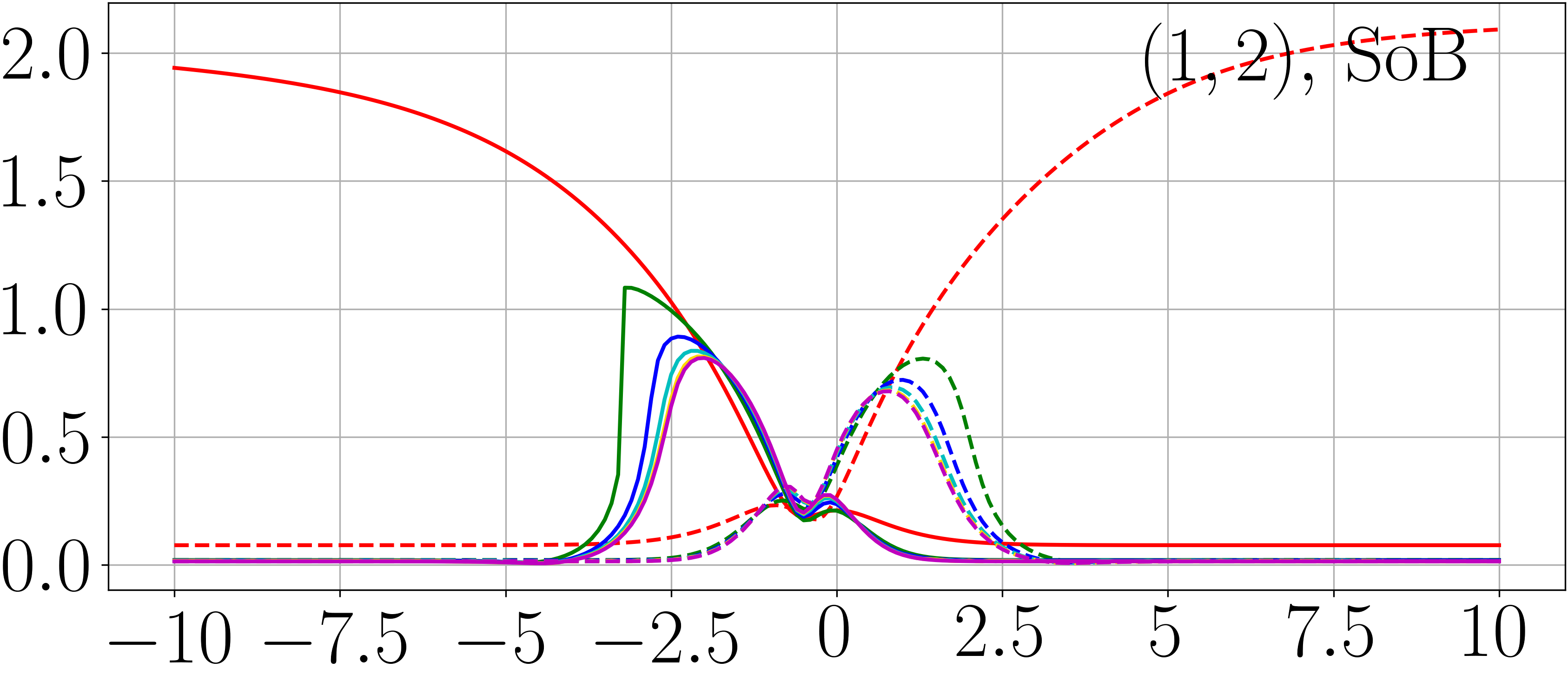}&
\includegraphics[width=4cm, bb=0 0 780 348]{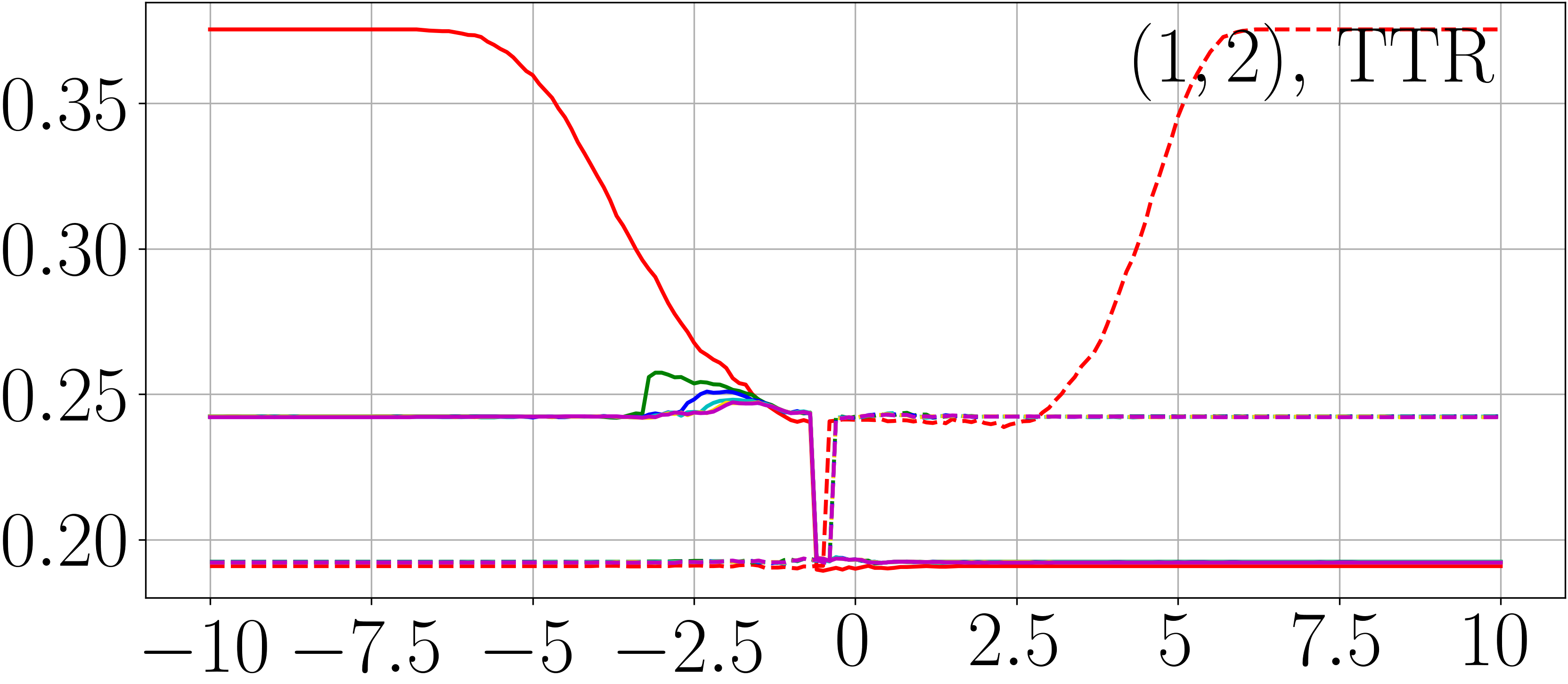}\\
\includegraphics[width=4cm, bb=0 0 780 348]{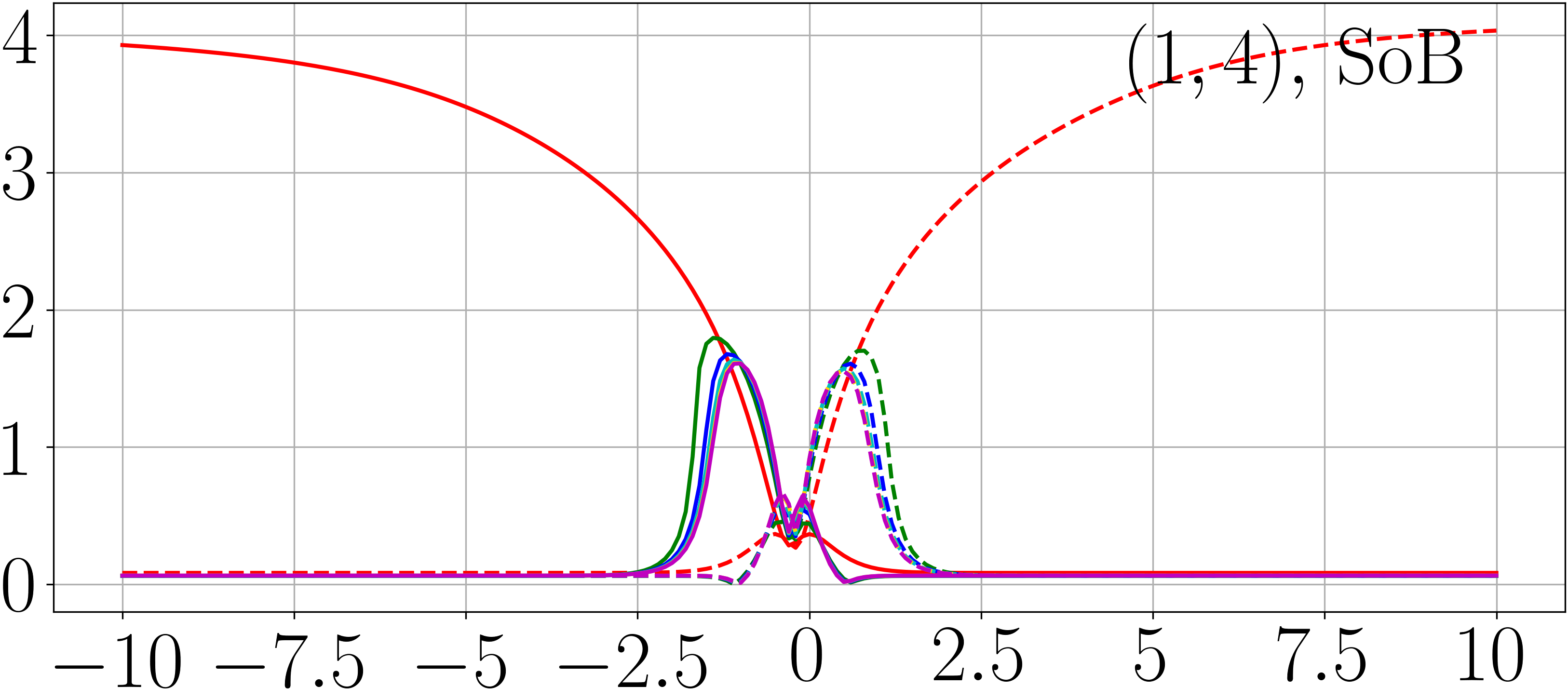}&
\includegraphics[width=4cm, bb=0 0 780 348]{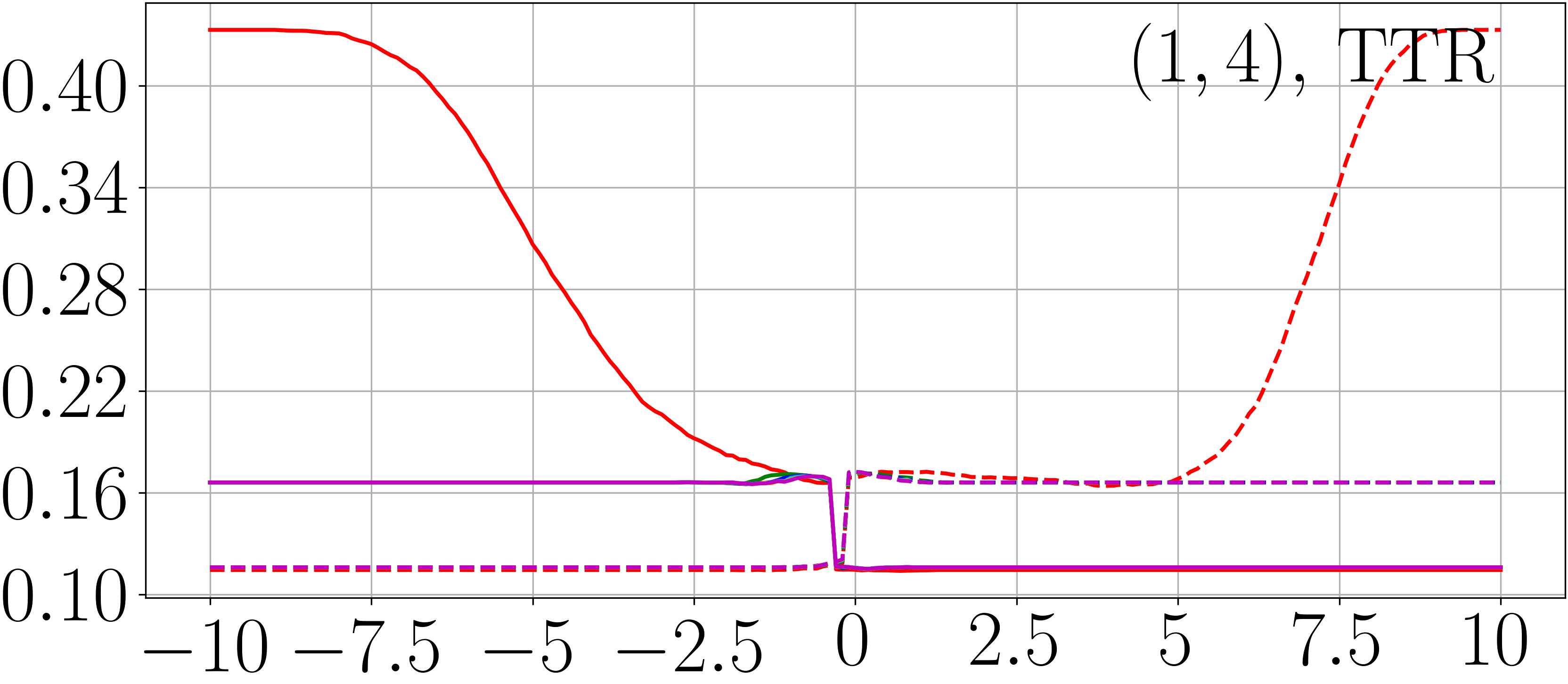}\\
\includegraphics[width=4cm, bb=0 0 780 348]{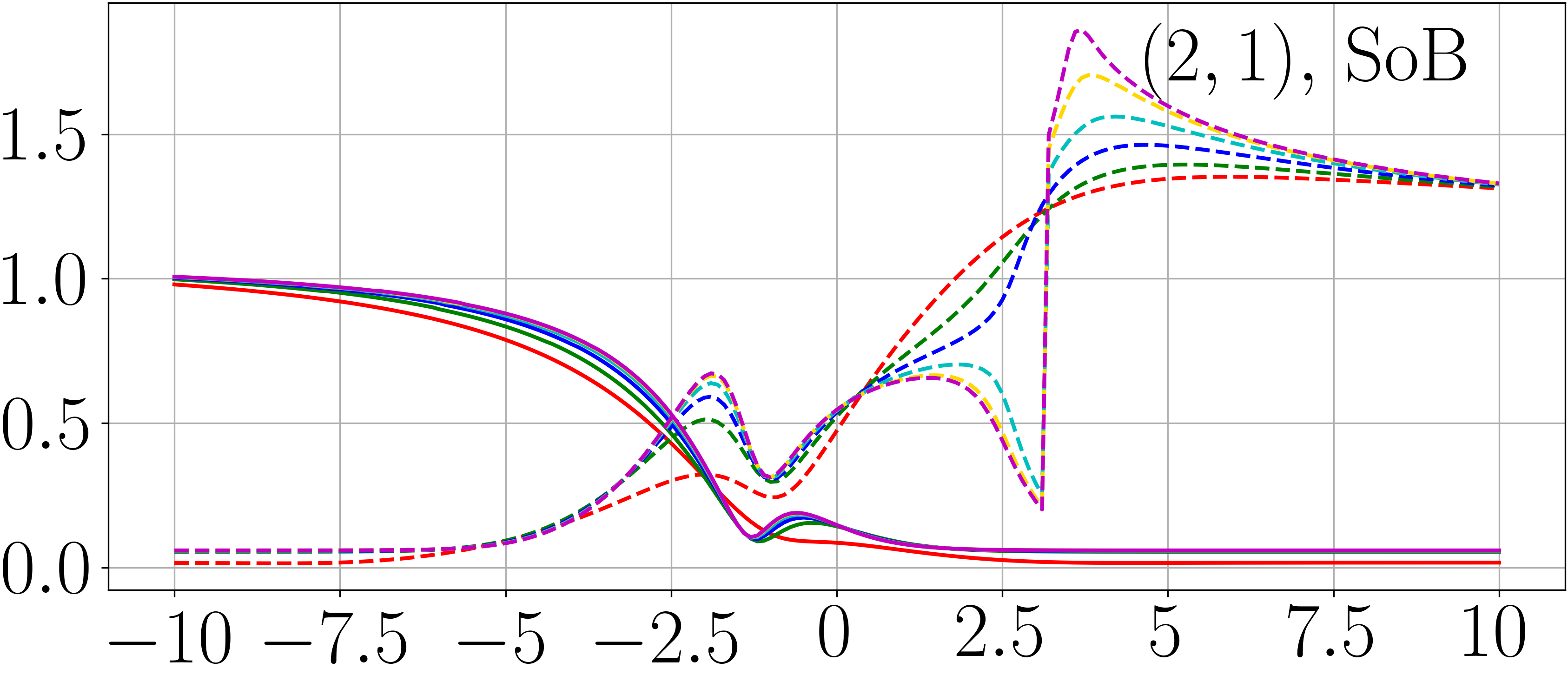}&
\includegraphics[width=4cm, bb=0 0 780 348]{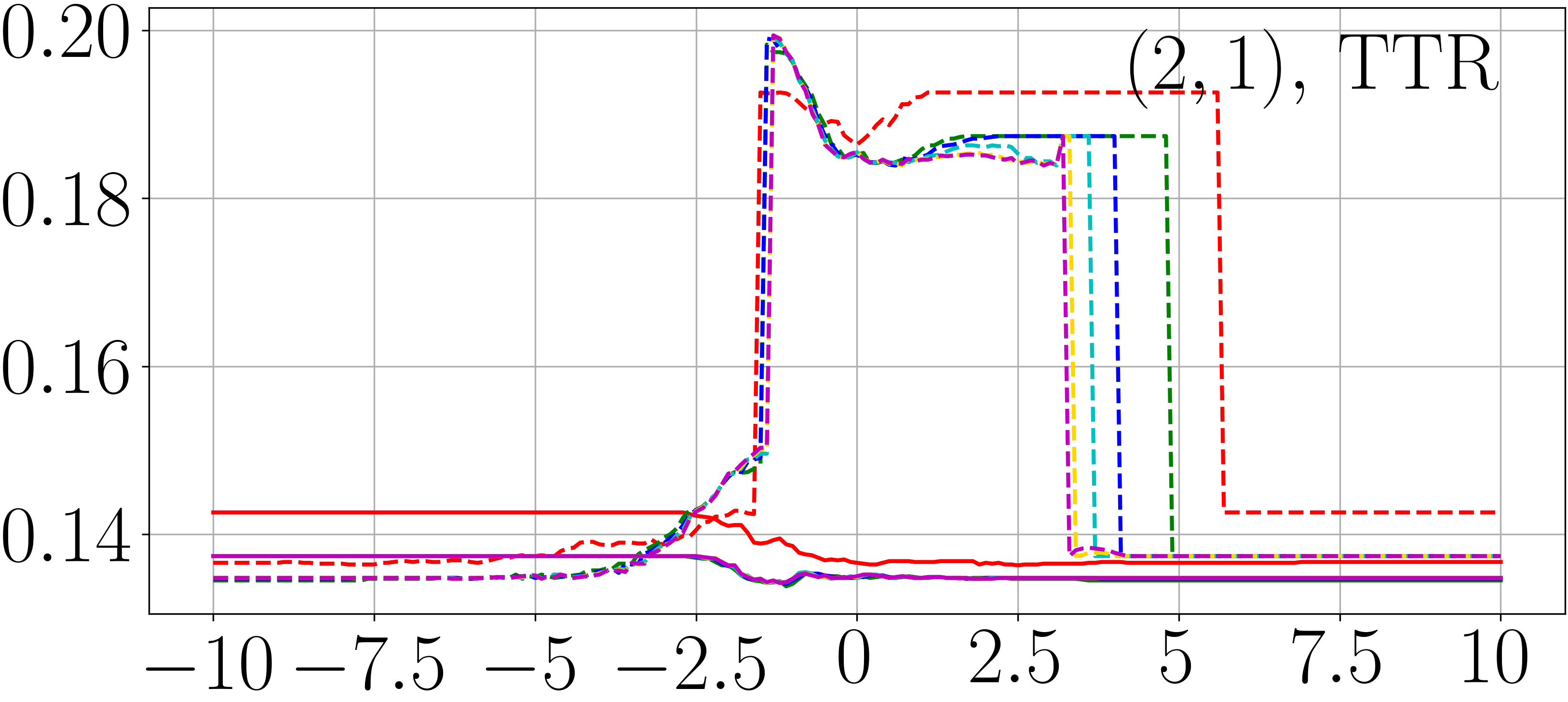}\\
\includegraphics[width=4cm, bb=0 0 780 348]{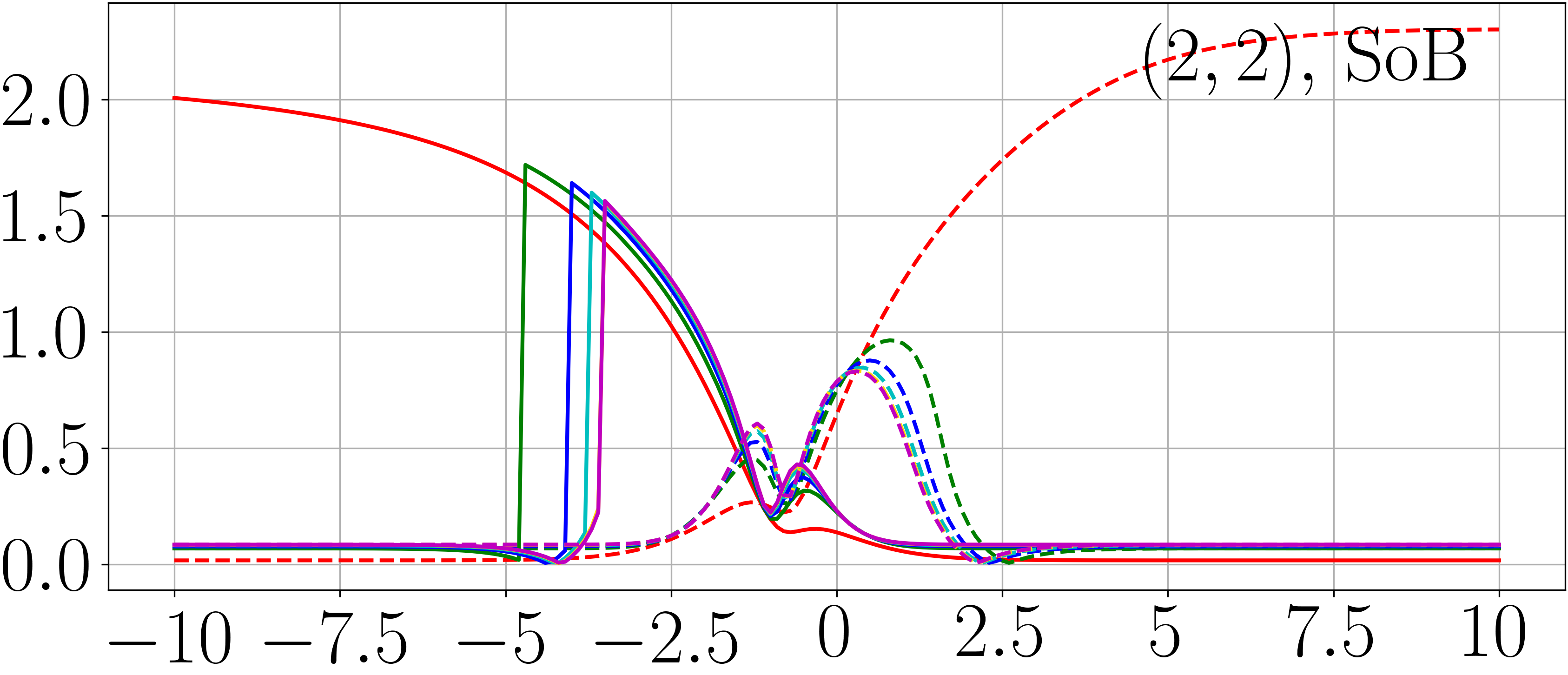}&
\includegraphics[width=4cm, bb=0 0 780 348]{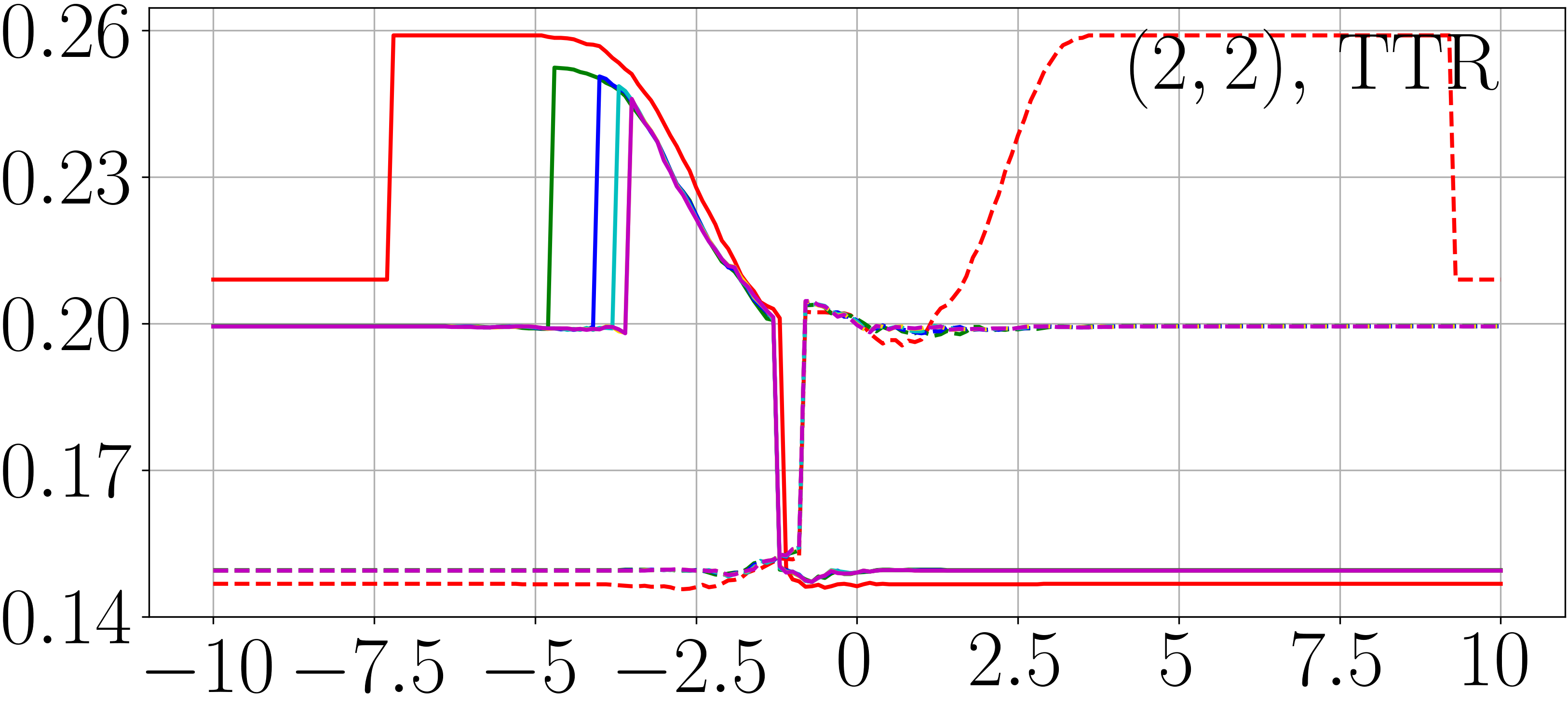}\\
\includegraphics[width=4cm, bb=0 0 780 348]{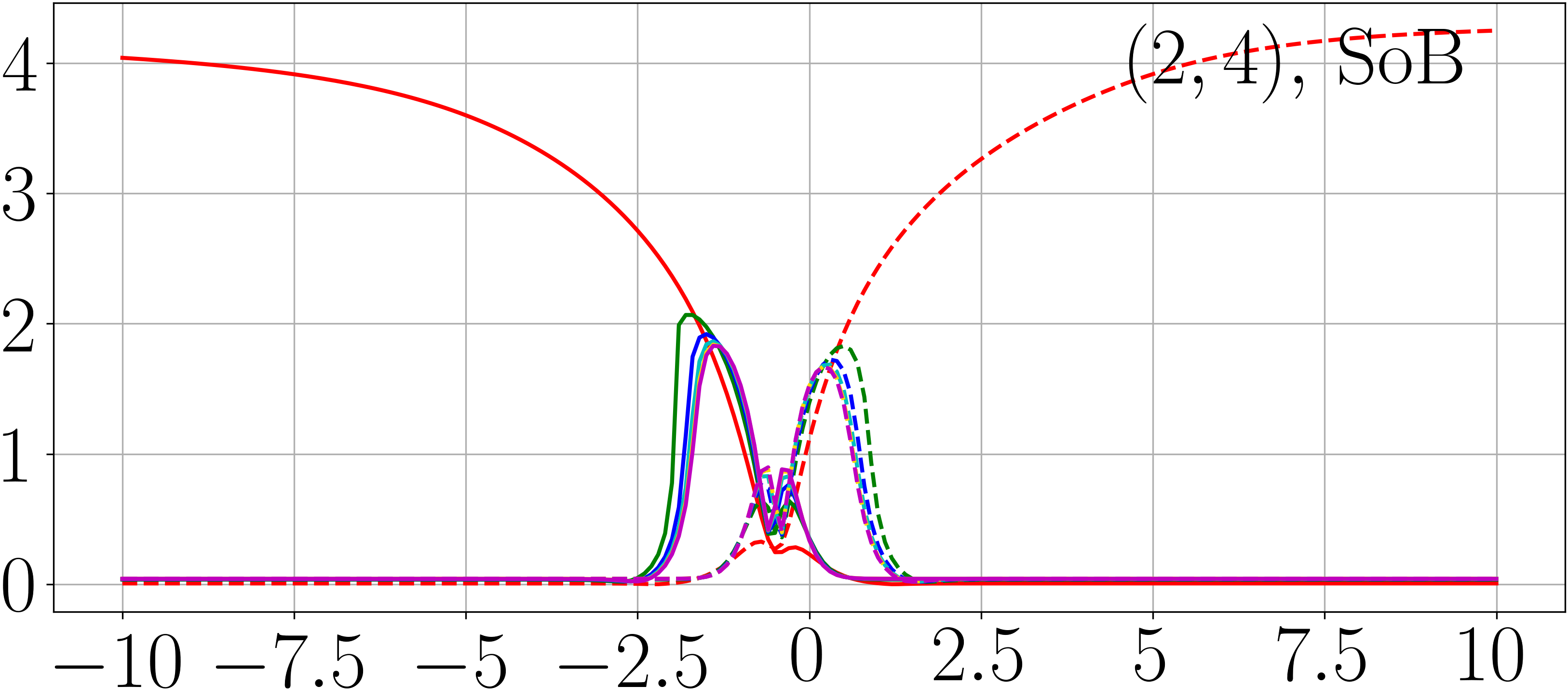}&
\includegraphics[width=4cm, bb=0 0 780 348]{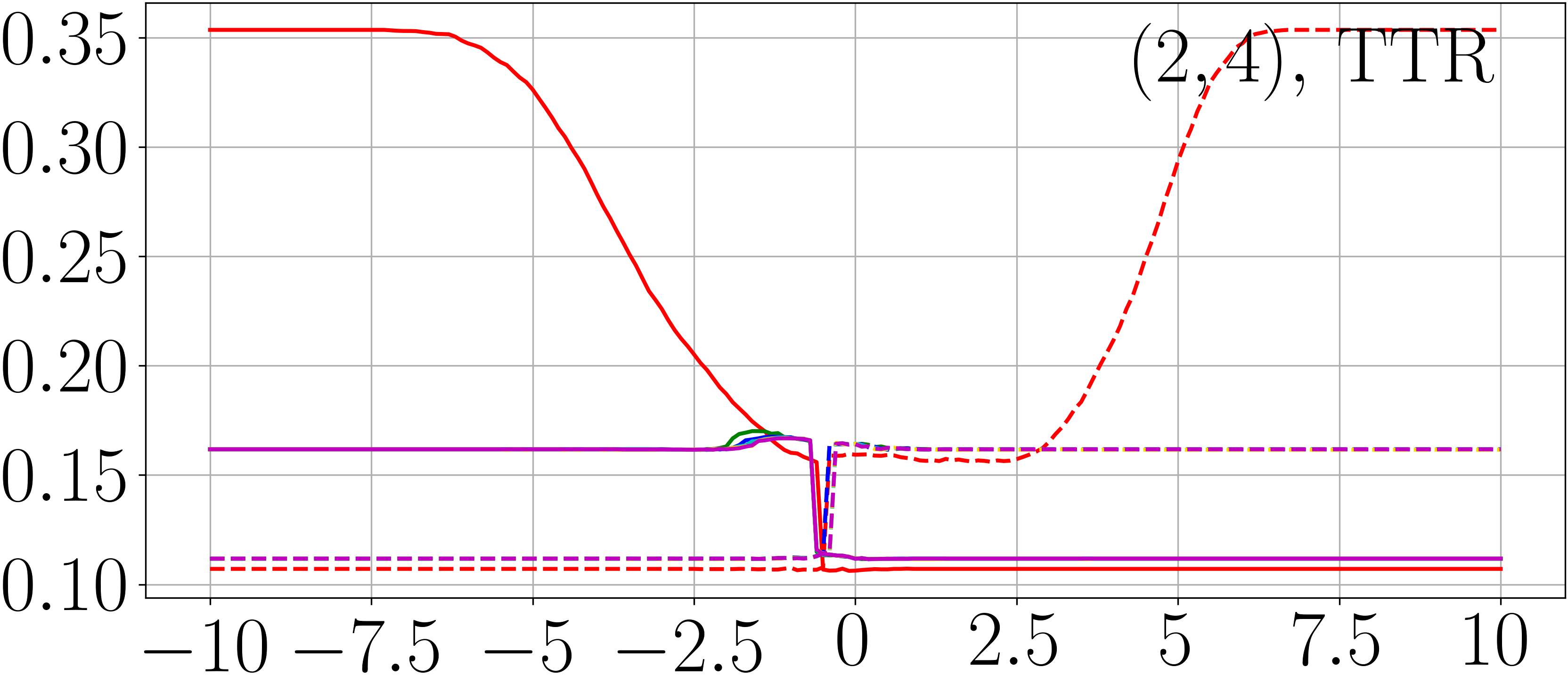}
\end{tabular}
\caption{%
Plots of SoB and TTR for LR (red), MLSLR with $\alpha=0.2,0.4,0.6,0.8$ (green, blue, cyan, yellow), and LSQLR (magenta) 
versus $x_{c,2}$, where solid and dashed lines are for $y_c=+1, -1$ (for Section \ref{sec:RSimulation}).}
\label{fig:RS}
\end{figure}

%==========%OK
Although theoretical discussion to guarantee the robustness of an estimator often adopts the notion of `breakdown point' 
($\inf\{\epsilon\mid\sup_{(\bx_c,y_c)}\|\bbeta_{\epsilon,(\bx_c,y_c)}\|=\infty\}$),
it has been known that it does not bring about meaningful results in the context of binary classification;
For example, \cite[Theorem 1]{croux2002breakdown}%
\footnote{%
\cite[Theorem 2]{croux2002breakdown} discussed an unconventional version of breakdown point
($\inf\{\epsilon\mid\inf_{(\bx_c,y_c)}\|\bbeta_{\epsilon,(\bx_c,y_c)}\|=0\}$) too, under the linear learner class $\calG$.
However, such breakdown is different from the situation in which the largest logit takes a quite large value, 
which is regarded as a trouble in studies on LS.
Also, it does not provide suggestions for cases where flexible models such as a neural network model are used, 
because it can occur depending heavily on the linearity of the learner.
This paper thus adopts only the conventional version of the breakdown point.}
states that the breakdown (in the above-mentioned sense) in the finite-sample situation
will not happen even for LR, which is in conflict with
the observed instability of LR as confirmed by \cite{pregibon1981logistic}.
Due to such difficulty in theoretical discussion, 
we study the robustness by resorting to a simulation experiment here.
%
%Note that, in the comparison of methods' robustness, the superiority of a method 
%will be on a case-by-case basis and depend largely on whether the contaminated 
%data distribution $F_{\epsilon, (\bx_c, y_c)}$ was compatible with the method or not,
%and hence it is difficult to obtain a universal claim by 
%comparing `MLSLR and LSQLR' with existing robust LRs.
%Therefore, instead of such an experiment, 
%we take experiments focusing to check robust behaviors of MLSLR and LSQLR.
The focuses of this experiment are to see 
if MLSLR and LSQLR actually make LR robust
and how the choice of smoothing level $\alpha$ affects the robustness of those methods.

%==========%OK
We consider the setting with the task with the zero-one task loss $\ell_\zo$,
the nominal distribution $\tilde{F}$ given by \eqref{eq:nominal}, and 
the point-mass contamination of $\epsilon=0.05$, $x_{c,2}=-10, -9.9, \ldots, 10$ ($x_{c,1}=1$), and $y_c=\pm1$.
We could not find a closed-form representation of $\bbeta_{\epsilon, (\bx_c, y_c)}$,
so we estimated it with a sample from $F_{\epsilon, (\bx_c, y_c)}$ of sufficiently large size 
such that the resulting variance gets negligibly small (denote it $\hat{\bbeta}_{\epsilon, (\bx_c, y_c)}$).
We calculated $\hat{\bbeta}_{\epsilon, (\bx_c, y_c)}$ with a training sample of size $n=10^4$
and evaluated SoB $\|\hat{\bbeta}_{\epsilon, (\bx_c, y_c)}-\tilde{\bbeta}\|$
and TTR $\frac{1}{m}\sum_{i=1}^m\ell_\zo(h_{\ell_\zo}(\hat{\bbeta}_{\epsilon, (\bx_c, y_c)}^\top\bx_i), y_i)$
with a sample of size $m=10^4$
(both samples follow the distribution \eqref{eq:ContDis}),
for LR, MLSLR with $\alpha=0.2, 0.4, 0.6, 0.8$, and LSQLR (by 1 trial).
We trained the model parameter in the same way as one in Section \ref{sec:ESimulation}.
Figure \ref{fig:RS} shows some of the results.

%==========%OK
Using a larger smoothing level $\alpha$ tended to improve SoB and TTR 
for many $(\bx_c, y_c)$'s that are so anomalous for \eqref{eq:nominal}
that $\frac{1}{1+e^{-y_c\tilde{\bbeta}^\top\bx_c}}$ gets quite small
(in the setting with larger $\tilde{\beta}_2$).
This result indicates that LS greatly robustified LR with larger $\alpha$.
Also, the improvement of TTR was larger than that of the experiment in Section \ref{sec:ESimulation}, 
which clarifies that LS is promising for better classification performance.

%==================================================%OK
\section{Experiment: LSLR versus MLSLR}
\label{sec:Experiment}
%==========%OK
In this section, so as to study the significance of LSLR 
using a consistent probability estimator different from LR and MLSLR,
we perform an experimental comparison between LSLR and MLSLR 
that respectively apply probability estimators based on the R-logit model and the logit model, 
while using intrinsically the same (surrogate) loss functions based on the SKL-divergence for probability estimation.
Several previous works claim that squeezing of the logits in LSLR 
(see Theorem \ref{thm:RLM-SKL}, \hyperlink{B6}{B6})
works as regularization and is an advantage of LS.
Recalling the contrasting property that the logits of MLSLR can output quite large values
(see Theorem \ref{thm:RLM-SKL}, \hyperlink{B5}{B5}),
if this claim were correct, LSLR would perform better than MLSLR.
The experiment will test whether a regularization mechanism based on logit-squeezing actually works.

%==========%OK
Once LSLR and MLSLR select a (common) model class $\calG$ (e.g., the network architecture),
optima of their surrogate risk $\min_{\bg\in\calG}\calR_\sur(\bg;\phi)$ will be different.
In order to reduce such a difference stemmed from 
the lack of representation ability of the model
and focus on their estimation performance,
we perform an experiment with a large-size learner model,
unlike those experiments in Sections \ref{sec:ESimulation} and \ref{sec:RSimulation}.

%==========%OK
Following the experiments by \cite{muller2019does} with the CIFAR-10 dataset,
we trained LR, LSLR and MLSLR with $\alpha=0.2, 0.4, 0.6, 0.8$,
and LSQLR based on ResNet-18 architecture 
(while we did not use the weight-decay \cite{krogh1992simple} in the implementation of \cite{muller2019does}) 
with the training sample of size $n=5\times10^4$,
and evaluated TSR $\frac{1}{m}\sum_{i=1}^m\phi(\bg(\bx_i),y_i)$ and 
TTR $\frac{1}{m}\sum_{i=1}^m\ell_\zo(h_{\ell_\zo}(\bg(\bx_i)),y_i)$ 
(where $h_{\ell_\zo}(\bv)=\argmax_{k\in[K]}\{v_k\}$) with $m=10^4$.
We trained each model for 150 epochs using the Nesterov's accelerated SGD similar to \cite{muller2019does},
and adopted a model at the point in time 
when each test risk achieved its minimum among those evaluated at the end of each epoch.
The results for 20 trials are summarized in Table \ref{tab:CIFAR10}
(mean and STD of TTR for LR and LSQLR were $.0784_{.0046}$ and $.0762_{.0019}$).
Note that it is meaningless to compare TSRs for different $\alpha$'s,
and we compare TSRs or TTRs of LSLR and MLSLR with the same $\alpha$
or TTRs with the different $\alpha$'s.

%==========%OK
\begin{table}[t]
\centering
\renewcommand{\tabcolsep}{5pt}
\caption{%
Mean and STD (as $\text{mean}_{\text{STD}}$) of TSR and TTR for LSLR and MLSLR (for Section \ref{sec:Experiment}).
The lower it values, the better the method is.
We showed results judged to be better by one-sided Wilcoxon rank sum test with a significance level of 0.05 in bold.}
\label{tab:CIFAR10}
\begin{tabular}{c|cccc}
\toprule
& \multicolumn{4}{c}{LSLR (upper) and MLSLR (lower)}\\
& $\alpha=0.2$ & $\alpha=0.4$ & $\alpha=0.6$ & $\alpha=0.8$ \\
\midrule
\multirow{2}{*}{TSR} &
$1.0316_{.0151}$ & $1.5413_{.0055}$ & $1.9299_{.0041}$ & $2.1942_{.0010}$ \\&
$1.0316_{.0045}$ & $\mathbf{1.5365_{.0026}}$ & $\mathbf{1.9229_{.0011}}$ & $\mathbf{2.1902_{.0004}}$ \\
\multirow{2}{*}{TTR} &
$.0781_{.0075}$ & $.0768_{.0048}$ & $.0830_{.0068}$ & $.0952_{.0059}$ \\&
$\mathbf{.0697_{.0019}}$ & $\mathbf{.0670_{.0017}}$ & $\mathbf{.0689_{.0015}}$ & $\mathbf{.0736_{.0021}}$ \\
\bottomrule
\end{tabular}
%LR: $.3966_{.0208}$, $.0784_{.0046}$
%LSQLR: $.0125_{.0003}$, $.0762_{.0019}$
\end{table}
%\vskip15pt
\begin{table}[t]
\centering
\renewcommand{\tabcolsep}{5pt}
\caption{%
Mean and STD (as $\text{mean}_{\text{STD}}$) of OPDER, OPER, and MSoR 
of the R-logit model by LSLR
evaluated for training and test sets at an epoch with the minimum TSR
(for Section \ref{sec:Experiment}).}
\label{tab:OP}
\begin{tabular}{cc|cccc}
\toprule
& & \multicolumn{4}{c}{LSLR}\\
& & $\alpha=0.2$ & $\alpha=0.4$ & $\alpha=0.6$ & $\alpha=0.8$ \\
\midrule
\multirow{3}{*}{\rotatebox{90}{training}}&
OPDER      & $.8961_{.0105}$  & $.9319_{.0148}$  & $.9646_{.0081}$  & $.9879_{.0033}$ \\&
OPER         & $.5198_{.0049}$  & $.5080_{.0023}$  & $.5024_{.0033}$  & $.4957_{.0037}$ \\&
MSoR        & $.0032_{.0007}$  & $.0043_{.0009}$  & $.0055_{.0010}$  & $.0089_{.0009}$ \\
\midrule
\multirow{3}{*}{\rotatebox{90}{test}}&
OPDER        & $.8123_{.0196}$  & $.8542_{.0158}$  & $.8947_{.0157}$  & $.9580_{.0084}$ \\&
OPER          & $.4313_{.0251}$  & $.4340_{.0373}$  & $.4356_{.0171}$  & $.4550_{.0189}$ \\&
MSoR         & $.0026_{.0006}$  & $.0035_{.0009}$  & $.0045_{.0008}$  & $.0075_{.0008}$ \\
\bottomrule
\end{tabular}
\end{table}

%==========%OK
\begin{figure}[t]
\centering
\begin{tabular}{cc}
\includegraphics[width=4cm, bb=0 0 780 420]{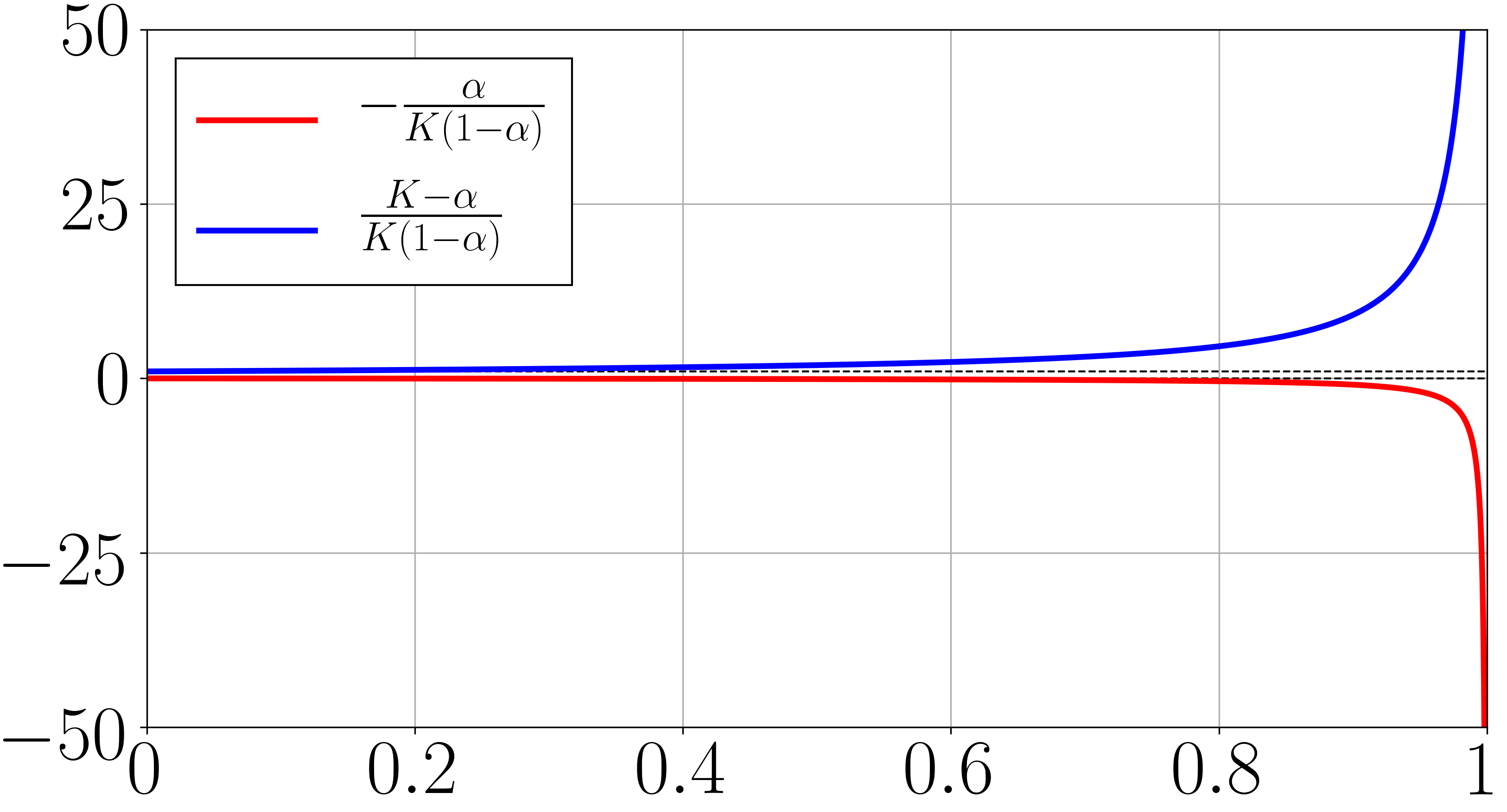}&\hskip-8pt
\includegraphics[width=4cm, bb=0 0 780 420]{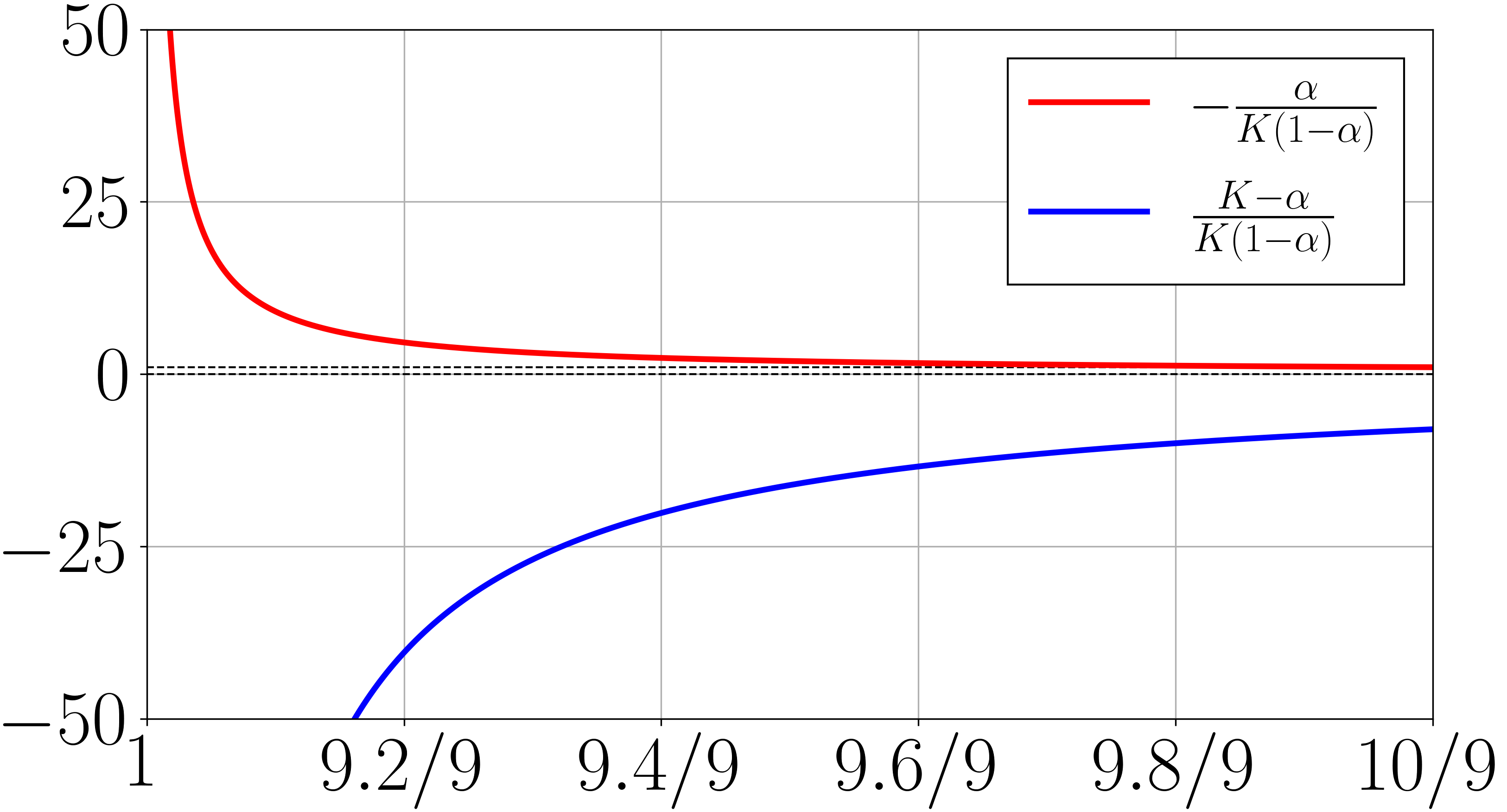}
\end{tabular}
\caption{%
Both ends $-\frac{\alpha}{K(1-\alpha)}, \frac{K-\alpha}{K(1-\alpha)}$
of the range of the element of the R-logit model $\bq_{\RLM,\alpha}$ 
versus $\alpha\in[0,1)\cup(1,\frac{K}{K-1}]$ for $K=10$.}
\label{fig:ends}
\end{figure}

%==========%OK
The R-logit model $\bq_{\RLM,\alpha}(\bg(\bx))$, 
a consistent probability estimator of LSLR, 
has an unnecessarily larger range than that of the CPD function $\bp(\bx)$, $\Delta_{K-1}$
(see Theorem \ref{thm:RLM-SKL}, \hyperlink{B1}{B1} and Figure \ref{fig:ends}).
This fact can be interpreted as learning a probability estimator 
from an unnecessarily larger hypothesis space,
and we consider that it may prevent proper learning and 
degrade prediction performance of LSLR,
in contrast to the positive statement by existing studies for the logit-squeezing.
To evaluate the degree to which the R-logit model deviates from the probability simplex, 
we calculated the following three criteria:
Table~\ref{tab:OP} shows outlier probability distribution estimate rate (OPDER)
\begin{align}
	\frac{1}{n}\sum_{i=1}^n\mathbbm{1}\{\bq_{\RLM,\alpha}(\bg(\bx_i))\not\in\Delta_{K-1}\},
\end{align}
outlier probability estimate rate (OPER)
\begin{align}
	\frac{1}{n\times K}\sum_{i=1}^n\sum_{k=1}^K\mathbbm{1}\{q_{\RLM,\alpha,k}(\bg(\bx_i))\not\in[0,1]\},
\end{align}
and mean size of residual (MSoR)
\begin{align}
	\frac{\mathrm{OPER}}{n\times K}\sum_{i=1}^n\sum_{k=1}^K\Bigl(\Bigl|q_{\RLM,\alpha,k}(\bg(\bx_i))-\frac{1}{2}\Bigr|-\frac{1}{2}\Bigr)_+,
\end{align} 
evaluated for training and test (with $m$ instead of $n$) sets at an epoch with the minimum TSR,
where $(\cdot)_+\coloneq\max\{0,\cdot\}$.

%==========%OK
Comparing TSRs or TTRs of LSLR and MLSLR with the same $\alpha$ (see Table~\ref{tab:CIFAR10}), 
it can be seen that the modification of the consistent probability estimator or squeezing of the logits 
does not help to improve the probability estimation and classification performance.
Rather, MLSLR had stable and better performance in many cases.
Table~\ref{tab:OP} indicates that the R-logit model $\bq_{\RLM,\alpha}(\bg(\bx))$ 
of the LSLR often, and greatly for larger $\alpha$, deviates from the probability simplex,
as Theorem \ref{thm:RLM-SKL}, \hyperlink{B1}{B1} and Figure \ref{fig:ends} also suggest.
This result is coherent with the fact that the TTR of LSLR 
is much worse than that of MLSLR for large $\alpha$, 
and supports our hypothesis about the trouble of LSLR.
These observations and considerations are novel findings and recommend MLSLR over LSLR
when one uses a large-size learner model.

%==========%OK
Besides, the best method with respect to the TTR was MLSLR with an intermediate smoothing level $\alpha=0.4$.
Although this result is apart from the preset purpose of the experiment in this section,
it is also notable and can be understood from the trade-off 
between efficiency and robustness discussed in Sections \ref{sec:Efficiency} and \ref{sec:Robustness}:
Even a large-size neural network model ($\argmin_{\bg\in\calG}\calR_\sur(\bg;\phi)$) cannot completely represent 
the optimal solution ($\argmin_{\bg:\bbR^d\to\bbR^K}\calR_\sur(\bg;\phi)$), and due to such deviations (model misspecification), 
robustification by LS would have contributed to improve the probability estimation and classification performance.

%==================================================%OK
\section{Conclusion and Future Prospect}
\label{sec:Conclusion}
%==========%OK
This paper has proposed the loss view, that LS adopts a loss function different from that of LR,
in contrast to the regularization view, that LS is a sort of regularization techniques, adopted in most existing studies.
This loss view will also provide theoretical generalization analysis of LSLR; 
See Appendix~\ref{sec:General}.
Also, we introduced MLSLR, for fair comparison with LR, 
that adopts the same logit model as a consistent probability estimator.
Previous studies have stated
\begin{enumerate}
\item[{\hypertarget{A4}{A4}.}] 
If a teacher network is trained with LS, 
knowledge distillation into a student network is much less effective \cite{muller2019does, yuan2020revisiting}.
\item[{\hypertarget{A5}{A5}.}] 
LS is competitive with loss-correction techniques
under label noise \cite{lukasik2020does}.
\item[{\hypertarget{A6}{A6}.}] 
LS can help to speed up the convergence of SGD 
by reducing the variance \cite{xu2020towards}.
\end{enumerate}
but they regarded the inconsistent logit model as a probability estimator
of LSLR, which does not result in a fair comparison with LR.
Thus, it may be still meaningful to re-consider these statements in the introduced alternative view, 
coupled with the fact that MLSLR provided better probability estimation and classification performance than LSLR
when they depend on a large-size neural network model.

%==========%OK
In Sections \ref{sec:Efficiency} and \ref{sec:Robustness},
we showed that MLSLR and LSQLR are less efficient but more robust than LR:
This tendency becomes more pronounced as the smoothing level is increased.
As demonstrated in Section \ref{sec:Experiment},
the selection of the smoothing level
controls the trade-off between efficiency and robustness
and is practically important for better classification performance. 
For example, \cite{li2020regularization} studied and proposed covariate-dependent adaptation of 
the smoothing level: 
it decides the smoothing level locally according to an estimated maximum conditional probability
and estimated marginal distribution of the covariate.
Also, \cite{galdran2020cost} adopted target-dependent adaptation of the smoothing level,
so called non-uniform LS.
%%
%In robust statistics, free parameters are typically determined heuristically so that 
%the corresponding ARE becomes 0.95 in a standard data distribution setting.
%
The trade-off that we have discovered may lead to a more sensible selection of the smoothing level.

%==========%OK
Since the SKL divergence, on which LS is based, is a divergence that provides a robust statistical procedure, 
it may also be associated with another class of robustifying divergence, 
such as density power divergence \cite{basu1998robust, fujisawa2008robust, cichocki2010families}.
Also, although entropy regularization techniques \cite{miyato2015distributional, pereyra2017regularizing} 
actively used in reinforcement learning would be more difficult to express 
its corresponding loss function or divergence in a closed form and theoretically analyze them, 
it might be able to interpret them and other similar logit-squeezing \cite{kannan2018adversarial, engstrom2018evaluating}
as robustification in the same way as LSLR.
These related topics are the subject of future work,
and these methods have potential to be improved like MLSLR against LSLR.

%==================================================%OK
\appendices
%==================================================%OK
\section{Proof of Theorems~\protect\ref{thm:RLM-SKL} and \protect\ref{thm:limit}}
\label{sec:SKL}
%==========%OK
First, we present a proof of Theorem~\ref{thm:RLM-SKL}.
\begin{proof}[Proof of Theorem~\ref{thm:RLM-SKL}]
The statements \hyperlink{B1}{B1}, \hyperlink{B2}{B2}, \hyperlink{B7}{B7}, 
and \hyperlink{B8}{B8} can be proved by trivial calculations, so we omit the proof.
%The inverse function of $s_\alpha$ is
%$s_\alpha^{-1}(v)=\frac{1}{1-\alpha}v-\frac{\alpha}{(1-\alpha)K}$.
%One has that $s_\alpha^{-1}(0)=-\frac{\alpha}{(1-\alpha)K}$ and 
%$s_\alpha^{-1}(1)=\frac{K-\alpha}{(1-\alpha)K}$, which shows ~.
%Also, the simple calculation proves the normalization (statement \hyperlink{B2}{B2}).

The method of Lagrange multiplier,
\begin{align}
	\begin{split}
	&\tfrac{\partial}{\partial q_k}\bigl\{D_{\SKL,\alpha}(\bp||\bq)
	-\lambda\bigl({\textstyle\sum_{k=1}^K}q_k-1\bigr)\bigr\}\\
	&=-(1-\alpha)\tfrac{(1-\alpha)p_k+\tfrac{\alpha}{K}}{(1-\alpha)q_k+\tfrac{\alpha}{K}}-\lambda=0,
	\quad\text{for }k=1,\ldots,K,
	\end{split}
\end{align}
shows that the optimal solution is determined to $\bp$ as far as $\alpha\neq1$.
This result and $D_{\SKL,\alpha}(\bp||\bp)=0$ prove \hyperlink{B3}{B3}.

On the basis of the calculus of the second derivatives 
of $D_{\SKL,\alpha}(\bp||\bq)$ in $(\bp,\bq)$,
that is, for $k,l=1,\ldots,K$ s.t.\ $k\neq l$,
\begin{align}
	\begin{split}
	&\tfrac{\partial^2D_{\SKL,\alpha}(\bp||\bq)}{\partial p_k^2}
	=\tfrac{(1-\alpha)^2}{s_\alpha(p_k)},\;
	\tfrac{\partial^2D_{\SKL,\alpha}(\bp||\bq)}{\partial q_k^2}
	=\tfrac{(1-\alpha)^2s_\alpha(p_k)}{s_\alpha(q_k)^2},\\
	&\tfrac{\partial^2D_{\SKL,\alpha}(\bp||\bq)}{\partial p_kq_k}
	=-\tfrac{(1-\alpha)^2}{s_\alpha(q_k)},\\
	&\tfrac{\partial^2D_{\SKL,\alpha}(\bp||\bq)}{\partial p_kp_l}
	=\tfrac{\partial^2D_{\SKL,\alpha}(\bp||\bq)}{\partial q_kq_l}
	=\tfrac{\partial^2D_{\SKL,\alpha}(\bp||\bq)}{\partial p_kq_l}
	=0,\\
	\end{split}
\end{align}
one has that, for $\bv=(v_1,\ldots,v_K), \bw=(w_1,\ldots,w_K)\in\bbR^K$, 
\begin{align}
	&	\begin{pmatrix}
	\bv~\bw
	\end{pmatrix}
	\nabla^2D_{\SKL,\alpha}(\bp||\bq)
	\begin{pmatrix}
	\bv~\bw
	\end{pmatrix}^\top\nonumber\\
	&=(1-\alpha)^2
	\begin{pmatrix}
	\bv~\bw
	\end{pmatrix}
	\begin{pmatrix}
	\diag\bigl\{\frac{1}{s_\alpha(p_k)}\bigr\}&\diag\bigl\{-\frac{1}{s_\alpha(q_k)}\bigr\}\nonumber\\
	\diag\bigl\{-\frac{1}{s_\alpha(q_k)}\bigr\}&\diag\bigl\{\frac{s_\alpha(p_k)}{s_\alpha(q_k)^2}\bigr\}
	\end{pmatrix}
	\begin{pmatrix}
	\bv\\\bw
	\end{pmatrix}\nonumber\\
	&=(1-\alpha)^2\sum_{k=1}^K \Bigl(\tfrac{v_k}{s_\alpha(p_k)^{1/2}}-\tfrac{s_\alpha(p_k)^{1/2}w_k}{s_\alpha(q_k)}\Bigr)^2\ge0,
\end{align}
which clarifies the statement \hyperlink{B4}{B4}.

The constraint $g_K=0$, together with the symmetry, 
implies $g_2=\cdots=g_K=0$.
Then, as $q_{\LM,K}(\bg)=\frac{1}{(K-1)+e^{g_1}}=0$,
one has $g_1=\infty$ (statement \hyperlink{B5}{B5}).
Also, since $q_{\RLM,\alpha,K}(\bg)\propto\frac{1}{(K-1)+e^{g_1}}-\frac{\alpha}{K}=0$,
one has $g_1=\ln(\frac{K}{\alpha}+1-K)$ if $\frac{K}{\alpha}+1-K>0$,
i.e., for $\alpha\in(0,1)\cup(1,\frac{K}{K-1})$ (statement \hyperlink{B6}{B6}).
\end{proof}

%==========%OK
Next, we prove Theorem~\ref{thm:limit}.
\begin{proof}[Proof of Theorem~\ref{thm:limit}]
Taylor expansion $\ln(1+v)\approx v-\frac{1}{2}v^2$ for $v\in\bbR$ with a small absolute value shows
\begin{align}
\label{eq:SKL2}
	\begin{split}
	&D_{\SKL,\alpha}(\bp||\bq)
	=\sum_{k=1}^K\{(1-\alpha) p_k+\tfrac{\alpha}{K}\}
	\ln\tfrac{(1-\alpha) p_k+\tfrac{\alpha}{K}}{(1-\alpha) q_k+\tfrac{\alpha}{K}}\\
	&\to\sum_{k=1}^K\Bigl[\bigl\{(1-\alpha) p_k+\tfrac{\alpha}{K}\bigr\}
	\bigl\{\tfrac{(1-\alpha) K}{\alpha}p_k-\tfrac{1}{2}\bigl(\tfrac{(1-\alpha) K}{\alpha}p_k\bigr)^2\bigr\}\\
	&\hphantom{\to~}
	-\bigl\{(1-\alpha) p_k+\tfrac{\alpha}{K}\bigr\}
	\bigl\{\tfrac{(1-\alpha) K}{\alpha}q_k-\tfrac{1}{2}\bigl(\tfrac{(1-\alpha) K}{\alpha}q_k\bigr)^2\bigr\}\Bigr]\\
	&=\tfrac{(1-\alpha)^2K}{2\alpha}\|\bp-\bq\|^2+O\bigl(\tfrac{(1-\alpha)^3}{\alpha^2}\bigr),
	\end{split}
\end{align}
when $\alpha\to1$.
\end{proof}

%==================================================%OK
\section{Generalized Version of Corollaries~\protect\ref{cor:ProbPred} and \protect\ref{cor:ClasCali}}
\label{sec:ProCla}
%==========%OK
This section describes the generalized version of Corollaries~\ref{cor:ProbPred} and \ref{cor:ClasCali}
for the multi-class settings ($K>2$), cost-sensitive tasks ($\ell\neq\ell_\zo$), and $\alpha\in[0,1)\cup(1,\frac{K}{K-1}]$.

%==========%OK
For an $\bbR^K$-valued learner,
the surrogate loss function $\phi$ for LR, LSLR, MLSLR, and LSQLR can be represented as
\begin{align}
	\begin{split}
	\phi_\LR(\bv,y)&\,{\textstyle=-\sum_{k=1}^Kt_k(y)\ln\bigl(e^{v_k}/\sum_{l=1}^Ke^{v_l}\bigr)},\\
	\phi_{\LS,\alpha}(\bv,y)&\,{\textstyle=-\sum_{k=1}^Ks_\alpha(t_k(y))\ln\bigl(e^{v_k}/\sum_{l=1}^Ke^{v_l}\bigr)},\\
	\phi_{\MLS,\alpha}(\bv,y)&\,{\textstyle=-\sum_{k=1}^Ks_\alpha(t_k(y))\ln s_\alpha\bigl(e^{v_k}/\sum_{l=1}^Ke^{v_l}\bigr)},\\
	\phi_\LSQ(\bv,y)&\,{\textstyle=\sum_{k=1}^K\bigl(t_k(y)-e^{v_k}/\sum_{l=1}^Ke^{v_l}\bigr)^2}.
	\end{split}
\end{align}
Then, one has the following results:
\begin{proposition}
\label{prop:ProbPred2}
Assume $\alpha\in[0,1)\cup(1,\frac{K}{K-1}]$, and let 
$\bar{\bg}\in\argmin_{\bg:\bbR^d\to\bbR^K}\calR_\sur(\bg;\phi)$.
Then, regardless of the distribution of $(\bX,Y)$,
$\bq_\LM(\bar{\bg}(\bx))=\bp(\bx)$ a.s.~for LR, MLSLR, and LSQLR,
and $\bq_{\RLM,\alpha}(\bar{\bg}(\bx))=\bp(\bx)$ a.s.~for LSLR.
\end{proposition}
\begin{proposition}
\label{prop:ClasCali2}
Assume $\alpha\in[0,1)\cup(1,\frac{K}{K-1}]$, and let 
$\bar{\bg}\in\argmin_{\bg:\bbR^d\to\bbR^K}\calR_\sur(\bg;\phi)$.
Then, regardless of the distribution of $(\bX,Y)$,
$h_{\ell}(\bv)=\argmin_{l\in[K]}\sum_{k=1}^Kq_{\LM,k}(\bv)\ell(l,k)$ satisfies
$\calR_\tsk(h_{\ell}\circ\bar{\bg};\ell)=\inf_{f:\bbR^d\to[K]}\calR_\tsk(f;\ell)$
for LR, MLSLR, and LSQLR, and 
$h_{\ell}(\bv)=\argmin_{l\in[K]}\sum_{k=1}^Kq_{\RLM,\alpha,k}(\bv)\ell(l,k)$ satisfies
$\calR_\tsk(h_{\ell}\circ\bar{\bg};\ell)=\inf_{f:\bbR^d\to[K]}\calR_\tsk(f;\ell)$
for LSLR.
\end{proposition}
Theorem \ref{thm:RLM-SKL}, \hyperlink{B3}{B3} shows Proposition \ref{prop:ProbPred2},
and Proposition \ref{prop:ClasCali2} is trivial from Proposition \ref{prop:ProbPred2},
considering the form of the task risk, $\calR_\tsk(f,\ell)=\bbE_\bX[\sum_{k=1}^K\Pr(Y=k|\bX)\ell(f(\bX),k)]$.
Also, one has to pay attention to the 
labeling function of LSLR with $\alpha\in(1,\frac{K}{K-1}]$
even for the task with the zero-one task loss $\ell_\zo$.
\begin{corollary}%[{Corollary of Proposition~\ref{prop:ClasCali2}}]
\label{cor:ClasCali3}
Under the assumption of Proposition~\ref{prop:ClasCali2},
$h_{\ell_\zo}(\bv)=\argmax\{v_k\}_k$ for LR, LSLR with with $\alpha\in[0,1)$, 
MLSLR with $\alpha\in[0,1)\cup(1,\frac{K}{K-1}]$, and LSQLR,
and $h_{\ell_\zo}(\bv)=\argmin\{v_k\}_k$ for LSLR with with $\alpha\in(1,\frac{K}{K-1}]$.
\end{corollary}

%==================================================%OK
\section{Proof of Theorems~\protect\ref{thm:consistency} and \protect\ref{thm:normality}}
\label{sec:ProofAN}
%==========%OK
\cite{fahrmeir1985consistency} shows 
Theorems \ref{thm:consistency} and \ref{thm:normality}
for the loss $\varphi_\LR$.
Also, the surrogate losses $\varphi_{\MLS,\alpha}$, $\alpha\in(0,1)$ and $\varphi_\SQ$ are bounded, 
and Theorems \ref{thm:consistency} and \protect\ref{thm:normality} for these losses
can be proved in a way similar to that for \cite[Theorems 2.3 and 2.4]{bianco1996robust},
which is based on \cite[Chapter 6]{huber2004robust};
Refer to these studies for the proof of Theorems \ref{thm:consistency} and \ref{thm:normality}.
Additionally, we note that these theorems can be extended to 
the case $\alpha\in(1,2]$ (which is $(1,\frac{K}{K-1}]$ for $K=2$):
As Theorem \ref{thm:RLM-SKL}, \hyperlink{B8}{B8}, 
$\rmA_{\MLS,\alpha}=\rmA_{\MLS,2-\alpha}$,
and $\rmB_{\MLS,\alpha}=\rmB_{\MLS,2-\alpha}$ suggest,
MLSLRs with $\alpha\in(1,2]$ and $2-\alpha\in[0,1)$ provide the same result.

%==================================================%OK
\section{Discussion on LSLR and MLSLR with $\alpha\in(1,\frac{K}{K-1}]$}
\label{sec:OutAlpha}
%==========%OK
\begin{table}[t]
\centering
\renewcommand{\tabcolsep}{3pt}
\caption{%
$\alpha\in(1,\frac{K}{K-1}]$-version of Table~\ref{tab:CIFAR10} (for Appendix \ref{sec:OutAlpha}).}
\label{tab:CIFAR10Out}
\begin{tabular}{c|ccccc}
\toprule
& \multicolumn{5}{c}{LSLR (upper) and MLSLR (lower)}\\
& $\alpha=\frac{9.2}{9}$ & $\alpha=\frac{9.4}{9}$ & $\alpha=\frac{9.6}{9}$ & $\alpha=\frac{9.8}{9}$ & $\alpha=\frac{10}{9}$\\
\midrule
\multirow{2}{*}{TSR} &
$2.3015_{.0000}$ & $2.2953_{.0001}$ & $2.2829_{.0001}$ & $2.2616_{.0002}$ & $2.2233_{.0006}$ \\&
$\mathbf{2.3006_{.0000}}$ & $\mathbf{2.2939_{.0000}}$ & $\mathbf{2.2811_{.0004}}$ & $\mathbf{2.2597_{.0003}}$ & $2.2238_{.0008}$ \\
\multirow{2}{*}{TTR} &
$.3300_{.0094}$ & $.1627_{.0059}$ & $.1152_{.0026}$ & $.0946_{.0020}$ & $.0797_{.0019}$\\&
$\mathbf{.0989_{.0029}}$ & $\mathbf{.0752_{.0018}}$ & $\mathbf{.0721_{.0084}}$ & $\mathbf{.0713_{.0026}}$ & $\mathbf{.0682_{.0014}}$\\
\bottomrule
\end{tabular}
\end{table}
%\vskip15pt
\begin{table}[t]
\centering
\renewcommand{\tabcolsep}{3pt}
\caption{%
$\alpha\in(1,\frac{K}{K-1}]$-version of Table~\ref{tab:OP} (for Appendix \ref{sec:OutAlpha}).}
\label{tab:OPOut}
\begin{tabular}{cc|ccccc}
\toprule
& & \multicolumn{5}{c}{LSLR}\\
& & $\alpha=\frac{9.2}{9}$ & $\alpha=\frac{9.4}{9}$ & $\alpha=\frac{9.6}{9}$ & $\alpha=\frac{9.8}{9}$ & $\alpha=\frac{10}{9}$\\
\midrule
\multirow{3}{*}{\rotatebox{90}{training}}&
OPDER  & $.9885_{.0016}$ & $.9978_{.0003}$ & $.9993_{.0001}$ & $.9997_{.0001}$ & $.9957_{.0009}$ \\&
OPER     & $.3278_{.0044}$ & $.4303_{.0036}$ & $.4768_{.0036}$ & $.4922_{.0015}$ & $.4392_{.0026}$ \\&
MSoR     & $.0509_{.0025}$ & $.0322_{.0020}$ & $.0244_{.0011}$ & $.0193_{.0009}$ & $.0207_{.0009}$ \\
\midrule
\multirow{3}{*}{\rotatebox{90}{test}}&
OPDER  & $.9934_{.0017}$ & $.9983_{.0005}$ & $.9996_{.0002}$ & $.9997_{.0002}$ & $.9919_{.0016}$ \\&
OPER     & $.3652_{.0043}$ & $.4348_{.0057}$ & $.4755_{.0117}$ & $.4922_{.0080}$ & $.4322_{.0092}$ \\&
MSoR     & $.0579_{.0037}$ & $.0328_{.0028}$ & $.0235_{.0015}$ & $.0193_{.0011}$ & $.0213_{.0022}$ \\
\bottomrule
\end{tabular}
\end{table}

%==========%OK
In the binary setting $K=2$, Theorem \ref{thm:RLM-SKL}, \hyperlink{B8}{B8} 
suggests that LSLRs or MLSLRs with $\alpha\in(1,2]$ and $2-\alpha\in[0,1)$ 
give the same probability estimation and classification performance.
On the other hand, in a multi-class setting $K>2$, an analogy for LSLRs or MLSLRs 
with $\alpha\in(1,\frac{K}{K-1}]$ and $K-(K-1)\alpha\in[0,1)$ may not exactly hold;
See Theorem \ref{thm:RLM-SKL}, \hyperlink{B7}{B7}.
It would be more difficult to understand the results for $\alpha\in(1,\frac{K}{K-1}]$ 
than to understand the results for $\alpha\in[0,1)$ in a multi-class setting $K>2$
in our way of interpreting the LS technique with respect to LR, 
because the difference from LR ($\alpha=0$) gets bigger
(just like that the mathematical approximation becomes less accurate for a more distant point).
Therefore, we here report only considerations based on the experimental observations.

%==========%OK
We took CIFAR-10 experiments similar to those in 
Section~\ref{sec:Experiment} for $\alpha\in(1,\frac{K}{K-1}]$.
We tried LSLR and MLSLR for $\alpha=\frac{9.2}{9},\frac{9.4}{9},\frac{9.6}{9},\frac{9.8}{9},\frac{10}{9}$
under the task with the zero-one task loss $\ell_\zo$,
where we used the labeling function $h_{\ell_\zo}$ described in Corollary \ref{cor:ClasCali3}.
The results are shown in Tables~\ref{tab:CIFAR10Out} and \ref{tab:OPOut}.

%==========%OK
As Figure~\ref{fig:ends} and Table~\ref{tab:OPOut} show,
the R-logit model $\bq_{\RLM,\alpha}(\bg(\bx))$ of the LSLR 
often and largely deviates from the probability simplex, 
and LSLR gave worse TSR and TTR than MLSLR with the same $\alpha$.
When $\alpha\in(1,\frac{K}{K-1}]$, 
the smaller $\alpha$ is, the larger the deviation tended, which negatively affected 
the probability estimation and classification performance of LSLR.
Also, all TTRs of MLSLR with $\alpha\in(1,\frac{K}{K-1}]$ 
were worse than TTR of MLSLR with $\alpha=0.4$,
the best result of MLSLR with $\alpha\in[0,1)$.
In the CIFAR-10 experiments ($K=10$), 
we could not find advantage of choosing $\alpha\in(1,\frac{K}{K-1}]$.

%==================================================%OK
\section{Discussion on Generalization Analysis}
\label{sec:General}
%==========%OK
Several previous studies have provided experimental investigations 
of the generalization performance of LSLR, 
but there has not been theoretical analysis.
This may be attributed to the fact that many studies view 
the LS technique as regularization added to LR.
The loss view gives generalization analysis of LSLR collaterally.
Although it does not present a meaningful comparison between LR and LSLR, 
we here give discussions on generalization analysis in probability estimation and classification 
tasks for LSLR from the loss view to compensate for the absence of theory,
and mention challenges in this direction.

%==========%OK
According to \cite[Theorem 4]{bartlett2006convexity},
one can obtain a generalization bound for LSLR (and LR) that uses 
a classification calibrated, Lipschitz continuous and convex loss function $\varphi$
in a simple setting of Section \ref{sec:Setting}.
This bound is governed by a covexified variational transformation of the loss $\varphi$
(called $\psi$-transform in \cite{bartlett2006convexity}),
and the tightest possible upper bound uniform over all probability distributions.
By viewing the LS technique as modification of the surrogate loss, 
a lot of theories by existing research can be directly applied 
to analyze the properties of LSLR in other various settings;
See \cite[Theorem 5]{bartlett2006convexity} under low-noise assumption 
(which assumes that $p_1(\bx)$ is unlikely to be close to $1/2$),
\cite[Chapter 10]{cucker2007learning}
for a norm-regularized version using a kernel-based learner,
and \cite{pires2013cost} for multi-class cost-sensitive tasks.
However, it also should be noted that such results grounded on 
learning theories are often loose bounding-based evaluation 
and contain quantities such like $\inf_{\bg\in\calG}\calR_\sur(\bg;\phi)$ or $\psi$-transform
that cannot be known in advance or may vary for methods using different surrogate losses, 
making it difficult to make clear comparisons between methods using different losses.

\section*{Acknowledgment}
This work was supported by Grant-in-Aid for JSPS Fellows, Number 20J23367.

%==================================================%OK
\bibliographystyle{IEEEtran}
\bibliography{bibtex}

%==================================================%OK
\vspace*{-2\baselineskip}
\begin{IEEEbiography}[{\includegraphics[width=1in,height=1.25in,clip,bb=0 0 640 800]{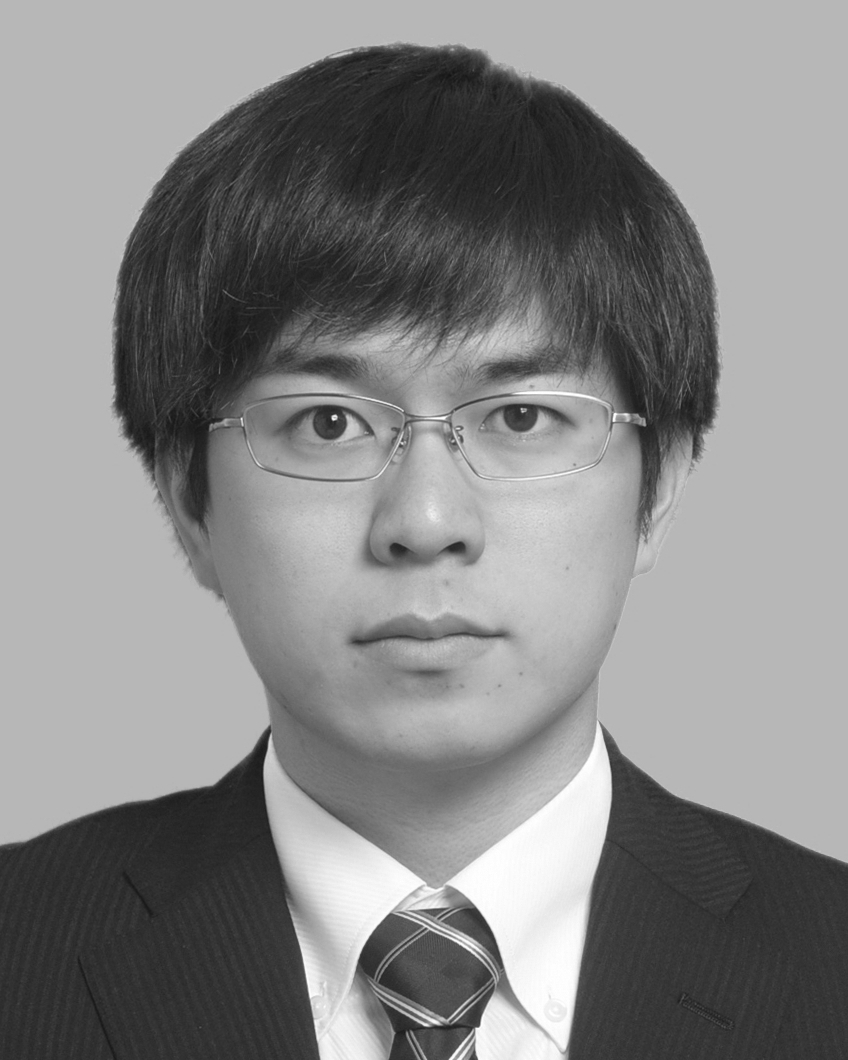}}]{Ryoya Yamasaki}
received the B.E.\ and M.Inf.\ degrees from Kyoto University, Kyoto, Japan, in 2018 and 2020, respectively.
He is currently working toward the D.Inf.\ degree of Graduate School of Informatics, Kyoto University, Kyoto, Japan.
His research interests are in areas of statistics and machine learning.
\end{IEEEbiography}
\vspace*{-2\baselineskip}
\begin{IEEEbiography}[{\includegraphics[width=1in,height=1.25in,clip,bb=0 0 800 1000]{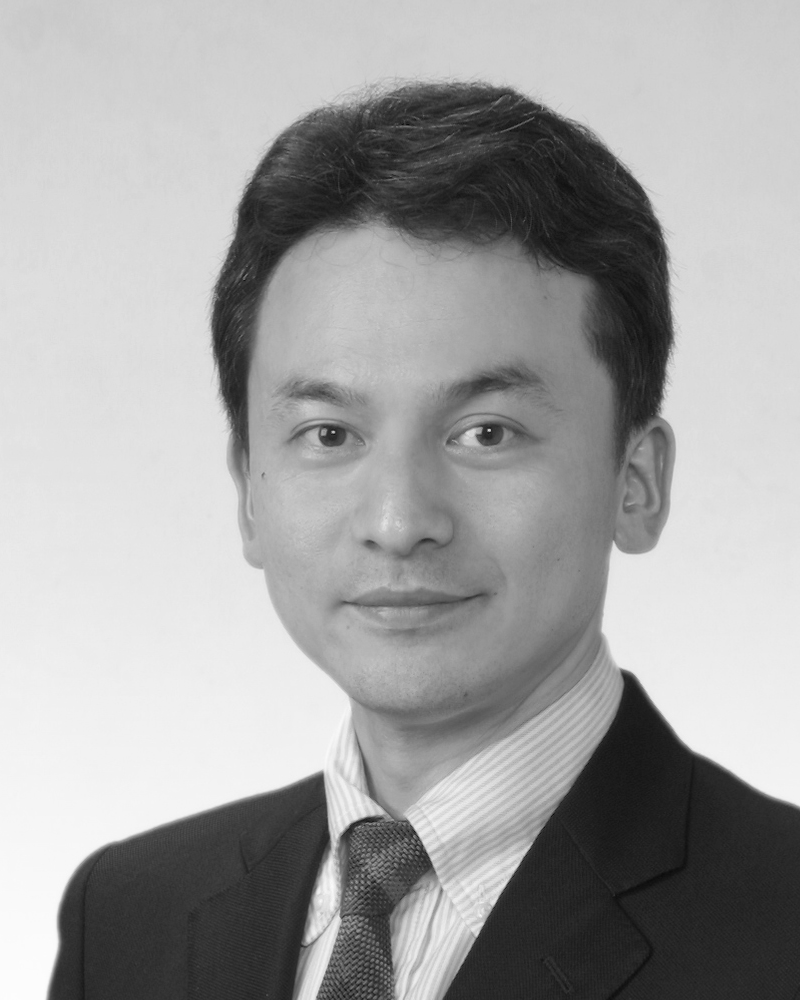}}]{Toshiyuki Tanaka} 
received the B.E., M.E., and D.E.\ degrees from the University of Tokyo, Tokyo, Japan, in 1988, 1990, and 1993, respectively.
He is currently a professor of Graduate School of Informatics, Kyoto University, Kyoto, Japan.
His research interests are in areas of information, coding, and communications theory, and statistical learning.
\end{IEEEbiography}
\end{document}